%% file: main.tex
\documentclass{article}

\usepackage[main, final]{neurips_2026}

\usepackage[utf8]{inputenc} 
\usepackage[T1]{fontenc}    
\usepackage{hyperref}       
\usepackage{url}            
\usepackage{booktabs}       
\usepackage{amsfonts}       
\usepackage{nicefrac}       
\usepackage{microtype}      
\usepackage{xcolor}         
\usepackage{amsmath}
\usepackage{amsthm}
\usepackage{ragged2e}
\usepackage{graphicx}
\usepackage{subcaption}

\newtheorem{theorem}{Theorem}
\newtheorem{lemma}{Lemma}
\newtheorem{proposition}{Proposition}

\newtheorem{definition}{Definition}

\title{A Statistical Theory of Gated Attention through the Lens of Hierarchical Mixture of Experts}

\author{%
  Viet Nguyen$^{1,\star}$ \quad
  Tuan Minh Pham$^{2,\star}$ \quad
  Thinh Cao$^{1,\star}$ \quad
  Tan Dinh$^{3}$ \\
  \textbf{Huy Nguyen$^{4}$ \quad
  Nhat Ho$^{4,\star\star}$ \quad
  Alessandro Rinaldo$^{4,\star\star}$} \\[0.5em]
  $^{1}$Hanoi University of Science and Technology \quad
  $^{2}$Purdue University \\
  $^{3}$Trivita AI \quad
  $^{4}$The University of Texas at Austin
}

\input{math_commands}

\begin{document}

\maketitle

\begin{abstract}
    Self-attention has greatly contributed to the success of the widely used Transformer architecture by enabling learning from data with long-range dependencies. In an effort to improve performance, a gated attention model that leverages a gating mechanism within the multi-head self-attention has recently been proposed as a promising alternative. Gated attention has been empirically demonstrated to increase the expressiveness of low-rank mapping in standard attention and even to eliminate the attention sink phenomenon. Despite its efficacy, a clear theoretical understanding of gated attention's benefits remains lacking in the literature. To close this gap, we rigorously show that each entry in a gated attention matrix or a multi-head self-attention matrix can be written as a hierarchical mixture of experts. By recasting learning as an expert estimation problem, we demonstrate that gated attention is more sample-efficient than multi-head self-attention. In particular, while the former needs only a polynomial number of data points to recover the model parameters, the latter requires exponentially many data points to achieve the same estimation error. Furthermore, our analysis also provides a theoretical justification for why gated attention is more effective at mitigating the attention sink phenomenon and the low-rank bottleneck when a gate is placed at the output of the scaled dot product attention or the value map rather than at other positions in the multi-head self-attention architecture. 
\end{abstract}
\let\thefootnote\relax\footnotetext{$^{\star}$Equal contribution, $^{\star\star}$Co-last authors.}
\vspace{-0.5cm}
\section{Introduction}
\label{sec:introduction}
The self-attention mechanism, introduced as a central component of the Transformer architecture by Vaswani et al. \cite{vaswani2017attention}, has fundamentally transformed sequence modeling tasks in deep learning. Unlike recurrent and convolutional neural networks that process tokens sequentially or within fixed local windows, self-attention allows every position in a sequence to directly attend to all other positions in a single parallel operation. In particular, for each token in the input sequence, it produces a context vector as a weighted sum of all tokens, allocating higher weights to those that are more relevant to the current token's context. Therefore, the self-attention is capable of capturing long-range dependencies with unprecedented flexibility and efficiency. As a consequence, Transformers have served as the backbone of virtually all state-of-the-art large language models (GPT series \citep{openai2024gpt4technicalreport}, Gemini \citep{comanici2025gemini25pushingfrontier}, LLaMA \citep{grattafiori2024llama3}, DeepSeek \citep{deepseekv3}, Qwen \citep{qwen2025}), vision language models (ViT \citep{dosovitskiy_image_2021}, Swin \citep{liu2021swin}), and multimodal models (CLIP \citep{radford2021clip}, LLaVA \citep{liu2023llava}).

Despite the remarkable success of the self-attention architecture, there is still room for improvement. Firstly, since the value and dense projections (i.e., the matrices $W_V$ and $W_O$; see Section \ref{sec:preliminaries} below for details) are two adjacent linear layers in the multi-head self-attention, they can be rewritten as one low-rank linear projection \citep{qiu2025gated}. This low-rank mapping turns out to limit the expressiveness of attention head linear combinations. 
Secondly, the attention sink phenomenon -- whereby a disproportionate amount of attention mass is allocated to a small number of initial tokens even when those tokens carry little or no semantic relevance to the current prediction -- is a widely observed behavior in autoregressive large language models \citep{xiao2024efficient,gu2024attentionsink,ramapuram2024sigmoidattention}. To mitigate these problems, Qiu et al. \cite{qiu2025gated} proposed using gated attention where a gating mechanism is applied at one of five positions in the multi-head self-attention:  following the scaled dot product attention (SDPA) outputs (G1); after the value (G2), key (G3), and query projections (G4); and after the final dense output layer (G5). Through extensive experiments and comparisons, the authors noted that adopting gating mechanisms at positions G1 and G2 appears to yield the highest performance gains. This enhancement is attributed to two key properties of gating mechanisms: non-linearity and sparsity. In particular, the non-linearity augments the expressiveness of the aforementioned low-rank linear projection, while the sparsity helps circumvent the attention sink phenomenon. Although there is now compelling  empirical evidence to corroborate the benefits of gated attention, a theoretical understanding of gated attention is still lacking in the literature. 

\textbf{Contributions.} The main goal of this work is to theoretically understand the benefits of the gated attention mechanism from a statistical perspective, complementing the extensive empirical evidence provided in \cite{qiu2025gated}. To this end, we establish a novel relation between multi-head attention/gated attention and hierarchical mixture of experts (HMoE) models \citep{Jordan1993hmoe}. Leveraging this connection, we analyze the sample complexity of gated attention by recasting it as a problem of expert specialization  \citep{oldfield2024specialize}. Here, the sample complexity refers to the number of samples required to recover the parameters of the attention block, rather than the number needed to achieve a given prediction error. Our contributions are threefold and can be summarized as follows. 

\emph{1. Gated attention meets HMoE.} In Section~\ref{sec:preliminaries}, we prove that each entry of a gated attention (resp. multi-head self-attention) matrix can be represented as a three-level HMoE with non-linear (resp. linear) experts.  

\emph{2. Sample complexity of multi-head self-attention.} Due to the simple linear structure of experts in the representation of multi-head self-attention, we show that the sample complexity of multi-head self-attention is exponential. Specifically, we prove that in order to estimate an expert within a given accuracy $\epsilon>0$ (see also Table~\ref{tab:rate_summary}), one needs exponentially many data points, on the order of $\mathcal{O}(\exp(\epsilon^{-1/\tau}))$ for some positive constant $\tau$. This result should discourage the use of vanilla self-attention.

\emph{3. Sample complexity of gated attention.} We investigate the sample complexity of the gated attention variants G1 and G2 that were experimentally found in \cite{qiu2025gated} to be most effective. Our theoretical results provide an explanation for why strongly identifiable activation functions, such as sigmoid and tanh with small bias, are more effective than the SiLU activation function. In our main result, we show that gating significantly reduces the sample complexity, from an exponential order to a polynomial order $\mathcal{O}(\epsilon^{-4})$ (see also Table~\ref{tab:rate_summary}). Hence, we claim that gated attention is more sample-efficient than multi-head self-attention.

In addition, we perform several numerical experiments to corroborate our theories in Section~\ref{sec:experiments}. Finally, in Section~\ref{sec:discussion}, we highlight some practical implications of our findings, and discuss a few limitations of our work as well as some potential future directions.

\textbf{Notation.} For any $n\in\mathbb{N}$, we set $[n] = \{1,2,\ldots,n\}$. For a set $S$, $|S|$ denotes its cardinality. For a vector $v = (v_1,\ldots,v_d) \in \mathbb{R}^d$, $\|v\|$ is its Euclidean ($\ell_2$) norm. For a multi-index $\alpha = (\alpha_1,\ldots,\alpha_d) \in \mathbb{N}^d$, we write $|\alpha| = \sum_{i=1}^d \alpha_i$, $\alpha! = \alpha_1! \cdots \alpha_d!$, and for $v \in \mathbb{R}^d$, $v^{\alpha} = v_1^{\alpha_1} \cdots v_d^{\alpha_d}$. For positive sequences $(a_n)$ and $(b_n)$, we write $a_n = \mathcal{O}(b_n)$ or $a_n \lesssim b_n$ if $a_n \le C b_n$ for all $n$, with a universal constant $C>0$. For random positive sequences $(t_n)$ and $(s_n)$, $t_n = \mathcal{O}_P(s_n)$ means $t_n/s_n$ is stochastically bounded: for any $\epsilon>0$ there exists $M>0$ such that $\mathbb{P}(t_n/s_n > M) < \epsilon$ for all sufficiently large $n$.
\vspace{-0.4cm}
\section{Related Work}
\label{sec:related_works}

\textbf{Self-Attention and Transformers.} The self-attention mechanism, introduced by Vaswani et al. \cite{vaswani2017attention}, has become the de facto backbone of modern deep learning, enabling the modeling of long-range dependencies in foundation models. Despite its success, standard dot-product attention exhibits several limitations. First, the \emph{attention sink} phenomenon causes autoregressive large language models (LLMs) to disproportionately attend to initial tokens regardless of their semantic relevance \citep{xiao2024efficient, gu2024attentionsink}. Second, the low-rank bottleneck of multi-head attention limits its expressivity \citep{bhojanapalli2021lowrank}. Finally, the quadratic time and memory complexity of standard attention poses challenges for long-context applications, motivating linear-time alternatives \citep{katharopoulos2020transformers} and kernel-based approximations \citep{choromanski2021rethinking}. Together, these limitations highlight the structural and computational challenges of standard attention, particularly in efficient streaming applications.

\textbf{Mixture of Experts (MoE) and Hierarchical MoE.}
Our theoretical framework builds on the literature of Mixture of Experts (MoE), pioneered by Jacobs et al. \cite{Jacobs1991adaptive} and Jordan et al. \cite{Jordan1993hmoe}. MoE models enhance model capacity through conditional expert activation, while HMoE extends this approach to multi-level architectures for modeling complex data structures \citep{jiang1999hmoe}. Recent developments include hierarchical representation learning for improved generalization \citep{li2025hmoe} and hierarchical mixtures of low-rank experts for parameter-efficient adaptation of large language models \citep{liao2025hmora}. From a statistical perspective, research has evolved from likelihood-based density estimation \citep{jiang_approximation_1999} toward understanding parameter estimation and expert specialization. Chen et al. \cite{chen2022theory} investigated the conditions underlying expert specialization in deep networks, whereas Ho et al. \cite{ho2022gaussian} and Nguyen et al. \cite{nguyen2023demystifying} derived convergence rates for Gaussian MoE and softmax-gated Gaussian MoE, respectively. Subsequent studies extended these analyzes to sparse \cite{nguyen2024statistical}, sigmoid \cite{nguyen2024sigmoid}, dense-to-sparse \cite{nguyen2026convergence}, and hierarchical MoE \cite{nguyen2024hmoe}, revealing how different gating mechanisms affect expert parameter estimation rates.

\textbf{Gated Attention Mechanisms.} To address the limitations of standard attention, incorporating gating mechanisms into attention blocks has emerged as a promising approach. Ramapuram et al. \cite{ramapuram2024sigmoidattention} and Yan et al. \cite{yan2025sigmoid} investigated the application of the sigmoid function to self-attention, provided comprehensive analyses, and demonstrated that this modification can improve performance. Subsequently, Qiu et al. \cite{qiu2025gated} introduced gated attention, empirically demonstrating that non-linear gating enhances the expressiveness of low-rank mappings and mitigates attention sinks through induced sparsity. However, existing studies on gated attention remain largely empirical, lacking a rigorous explanation for the benefits of non-linear gating. We address this gap by establishing a connection between gated attention and HMoE, leveraging statistical tools from the MoE literature to characterize its improved sample efficiency. We note that although recent works \citep{akbarian2024quadratic, yan2025sigmoid} have drawn connections between single-head attention and standard MoE architectures, their analyses do not fully capture multi-head attention, particularly the linear interaction between the value projection and output projection matrices, a key limitation addressed by gated attention \cite{qiu2025gated}.

\section{Preliminaries}
\label{sec:preliminaries}
\begin{table}
\centering
\caption{Summary of convergence rates for attention mechanisms under different gating configurations. The parameter estimation rate refers to the convergence of the mixing measure under the Voronoi losses $\mathcal{L}_1$ (Theorem~\ref{theorem:minimax_lower_bound_voronoi}) and $\mathcal{L}_2$ (Theorems~\ref{theorem:nonlinear_at_V} and~\ref{theorem:nonlinear_at_SDPA}) defined in Sections~\ref{sec:MHA} and~\ref{sec:gated_attention}, respectively. The sample complexity describes the number of data points required for expert estimation with error $\epsilon$.}
\label{tab:rate_summary}
\resizebox{\columnwidth}{!}{
\begin{tabular}{|l|c|c|c|}
\hline
\textbf{Mechanisms} 
& \textbf{Expert estimation rates} 
& \textbf{Sample complexity}
& \textbf{Theorems}\\
\hline
Multi-head self-attention 
& Slower than any polynomial order 
& $\mathcal{O}(\exp(\epsilon^{-1/\tau}))$
& Theorem~\ref{theorem:minimax_lower_bound_voronoi}\\
\hline
Gated attention (Setting I) 
& $\mathcal{O}_P([\log(n)/n]^{1/2})$ to $\mathcal{O}_P([\log(n)/n]^{1/4})$
& $\mathcal{O}(\epsilon^{-4})$ 
& Theorem~\ref{theorem:nonlinear_at_V}\\
\hline
Gated attention (Setting II) 
& $\mathcal{O}_P([\log(n)/n]^{1/2})$ to $\mathcal{O}_P([\log(n)/n]^{1/4})$
& $\mathcal{O}(\epsilon^{-4})$ 
& Theorem~\ref{theorem:nonlinear_at_SDPA}\\
\hline
\end{tabular}
}
\end{table}

In this section, we
present gated attention as an extension of multi-head attention, and then introduce hierarchical MoE. Later on, we will relate these two models and present a unified analysis. 

\textbf{Multi-head Self-attention (MHA).} Denote the input sequence by $\mathbb{X} := [x_1, \ldots, x_N]^\top \in \mathbb{R}^{N \times d}$,
where $N$ is the sequence length and $d$ is the embedding dimension. A single-head/Scaled dot-product attention (SDPA) is computed by 
\[
\mathrm{SDPA}(\mathbb{X};W_Q,W_K,W_V)
=
\mathrm{softmax}\!\left(\frac{\mathbb{X}W_QW_K^\top\mathbb{X}^\top}{\sqrt{d_v}}\right)\mathbb{X}W_V  \in \mathbb{R}^{N \times d_v},
\]
with learnable projection matrices $W_Q, W_K, W_V \in \mathbb{R}^{d \times d_v}$. Here and throughout, the dimensions $N$ and $d_v$ are assumed as given.
Multi-head self-attention (MHA) extends SDPA by computing multiple attention heads in parallel, each of which processes the input sequence through a different learned projection subspace. Formally, the output of an MHA layer is calculated as 
$$\mathrm{MHA}(\mathbb{X}) :
=
\mathrm{Concat}\big(\mathrm{head}_1,\ldots,\mathrm{head}_H\big) W_O = \sum_{h=1}^H \mathrm{softmax}\left(\frac{\mathbb{X}W_{Q,h}W_{K,h}^\top\mathbb{X}^\top}{\sqrt{d_v}}\right)\mathbb{X}W_{V,h}W_{O,h},
$$
where $\mathrm{head}_h := \mathrm{SDPA}(\mathbb{X};W_{Q,h},W_{K,h},W_{V,h})$ and $W_O := [W_{O,1}^\top,\ldots, W_{O,H}^\top]^\top \in \mathbb{R}^{Hd_v \times d}$. 

\textbf{Gated Attention.} In conventional MHA, the composition of the two adjacent projections $W_{V,h}$ and $W_{O,h}$ results in an implicit low-rank linear transformation. To enhance the expressiveness of this low-rank structure, Qiu et al. \cite{qiu2025gated} empirically demonstrated that introducing non-linearity either following the SDPA output or directly after the value output enhances the expressive power of the model. When a non-linear activation $\varphi$ is applied to the value output, the output of the gated attention layer can be expressed as
$$
\mathrm{GatedAttention}_{\mathrm{Value}}(\mathbb{X}) = \sum_{h=1}^H \mathrm{softmax}\left(\frac{\mathbb{X}W_{Q,h}W_{K,h}^\top\mathbb{X}^\top}{\sqrt{d_v}}\right)\varphi(\mathbb{X}W_{V,h})W_{O,h}.
$$
Alternatively, if the non-linearity is applied to the SDPA output, the gated attention layer can be written as
$$\mathrm{GatedAttention}_{\mathrm{SDPA}}(\mathbb{X}) = 
\sum_{h=1}^H \varphi\left(\mathrm{softmax}\left(\frac{\mathbb{X}W_{Q,h}W_{K,h}^\top\mathbb{X}^\top}{\sqrt{d_v}}\right)\mathbb{X}W_{V,h}\right)W_{O,h}.
$$

\textbf{Hierarchical mixture of experts (HMoE).}
The HMoE architecture \citep{Jordan1993hmoe} extends and adds structure to the standard MoE model \citep{Jacobs1991adaptive} by organizing local experts into a tree-structured hierarchy rather than a flat architecture. This structure employs a multi-level probabilistic routing mechanism to partition complex input spaces. For instance, let us consider a two-level HMoE model as follows
$$
y = \sum_{u=1}^{M} \phi_u(x) \sum_{v=1}^{K} \psi_{v|u}(x) \cdot \mathcal{E}_{v|u}(x).
$$
The first level is governed by a root gating network $\phi(x) := [\phi_1(x), \phi_2(x),\ldots,\phi_M(x)]^\top$, whereas the second level contains conditional gating networks $\psi_{v|u}(x)$ embedded within each branch. The model output is then formed as a hierarchical combination of the leaf experts $\mathcal{E}_{v|u}(x)$.

\textbf{Gated Attention meets Hierarchical Mixture of Experts.}
Recent studies \citep{akbarian2024quadratic, yan2025sigmoid} have established a connection between single-head attention and standard MoE architectures. However, these analyzes do not fully reflect realistic multi-head attention, as they fail to account for the linear product between the value projection and the output matrix, which is precisely the key issue addressed by gated attention \citep{qiu2025gated}. Therefore, in this work, we first establish a new connection between the outputs of multi-head attention/gated attention and HMoE models. Let $x = \mathrm{Vec}(\mathbb{X}) = (x_1^\top,\ldots,x_N^\top)^\top \in \mathbb{R}^{Nd}$ be the vectorization of $\mathbb{X}$ and set  $P_h:= \frac{W_{Q,h}(W_{K,h})^\top}{\sqrt{d_v}}$ and $J_i := \boldsymbol{e}_i^\top \otimes I_d\in \mathbb{R}^{d \times Nd}$, where $\otimes$ stands for Kronecker product.
Then, $J_i$ extracts the transpose of the $i^{th}$ row of matrix $\mathbb{X}$:  $J_i
x = x_i$.  Let us denote $M_{h,ij} = J_i^\top P_hJ_j$ and $A_{h,j} = J_j^\top W_{V,h}$. As a result, the $(i,j)$-th entry of the softmax output can be expressed as $[\mathbb{X} P_h\mathbb{X}^{\top}]_{i,j} = x_i^{\top} P_h x_j =  x^{\top}J_i^\top P_hJ_jx= x^{\top}M_{h,ij}x.$
Hence, the transpose of the $i^{th}$ row of the multi-head self-attention matrix is given by 
\begin{equation*}
     [\mathrm{MHA}(\mathbb{X})]^{\top}_{i} = \sum_{h=1}^{H}\sum_{j=1}^N\dfrac{\exp(x^{\top}M_{h,ij}x)}{\sum_{l=1}^N\exp(x^{\top}M_{h,il}x)}\cdot W_{O,h}^{\top}  (A_{h,j}^{\top})x.
\end{equation*}
Set $W_{O,h}=(\omega_{h,ii'})_{i\in[d_v],i'\in[d]}$ and let $a_{h,j,k}\in\mathbb{R}^{Nd}$ be the $k$-th column vector of the matrix $A_{h,j}$ Then one can verify that the $(i,i')$-th entry of the MHA matrix can be written in a HMoE form:
\begin{equation*}
     [\mathrm{MHA}(\mathbb{X})]_{i,i'} = \sum_{h=1}^{H}\sum_{k=1}^{d_v}\omega_{h,ki'}\sum_{j=1}^N\dfrac{\exp(x^{\top}M_{h,ij}x)}{\sum_{l=1}^N\exp(x^{\top}M_{h,il}x)}\cdot a_{h,j,k}^{\top}x.
\end{equation*}
Two other variants of gated attention are considered later in equations~\eqref{eq:gated_attention_model} and \eqref{eq:gated_SPDA_model}.

\section{Sample Complexity of Multi-head Attention}\label{sec:MHA}
In this section, we analyze the statistical sample complexity of multi-head self-attention in the expert specialization problem \citep{oldfield2024specialize} by leveraging its HMoE representation. In particular, we analyze how fast an expert learns a specific region of the data. To begin, let us formally present the problem setup. 

\textbf{Problem setup.} Assume that $\{(X_i,Y_i)\}_{i=1}^{n}\subset\mathbb{R}^{\bar{d}}\times\mathbb{R}$ are i.i.d samples from 
the regression model
\begin{align}
    Y_i=f_{\Gs}(X_i)+\varepsilon_i, \quad i=1,2,\ldots,n, 
    \label{eq:regression_model}
\end{align}
where the regression function $x \mapsto f_{\Gs}(x)$ is the $(H^*,K^*,N^*)$-HMoE model given by
\begin{align}
    \label{eq:mha_moe}
f_{G^*}(x) = \sum_{h=1}^{H^*}\sum_{k=1}^{K^*}\omega^*_{h,k}\sum_{i=1}^{N^*}\frac{\exp(x^{\top}M^*_{h,i}x)}{\sum_{j=1}^{N^*}\exp(x^{\top}M^*_{h,j}x)}
    \cdot (a^*_{h,i,k})^{\top}x.
\end{align}
Above, $\varepsilon_1,\varepsilon_2,\ldots,\varepsilon_n$ are independent Gaussian noise variables such that $\bbE[{\varepsilon_{i}}|X_i] = 0$ and $\var[\varepsilon_{i}|X_i] = \sigma^2$, for all $1 \le i \le n$. The covariates $X_i$'s are supported on a subset $\mathcal{X}$ of $\mathbb{R}^{\overline{d}}$.
Meanwhile, the ground-truth \emph{mixing measure} $\Gs:=\sum_{h=1}^{H^*}\sum_{k=1}^{K^*}\omega^*_{h,k}\sum_{i=1}^{N^*}\delta_{(M^*_{h,i},a^*_{h,i,k})}$ inherits the hierarchical structure of the HMoE model, where $(\omega_{h,k}^*, M_{h,i}^*,a_{h,i,k}^*)_{h \in [H^*], k \in [K^*],i \in [N^*]}$ are true yet unknown parameters in the parameter space  $\Theta \subset \mathbb{R}\times\mathbb{R}^{\bar{d}\times \bar{d}}\times\mathbb{R}^{\bar{d}}$. 

\textbf{Least squares estimation.} 
As the convergence analysis would become needlessly complicated if $H^*$ and $N^*$ are not given, we assume that their values are known for ease of presentation, while the value of $K^*$ still remains unknown. Under these assumptions, we over-specify the ground-truth model~\eqref{eq:mha_moe} by taking into account a least squares estimator within a class of $(H^*,K,N^*)$-HMoE models, where $K>K^*$ is given, as follows:
$\widehat{G}_n\in\argmin_{G\in\mathcal{G}_{H^*,K,N^*}(\Theta)}\sum_{i=1}^{n}\Big(Y_i-f_{G}(X_i)\Big)^2$,
where $\mathcal{G}_{H^*,K,N^*}(\Theta) = \{G=\sum_{h=1}^{H^*}\sum_{k=1}^{K}\omega_{h,k}\sum_{i=1}^{N^*}\delta_{(M_{h,i},a_{h,i,k})}:  (\omega_{h,k},M_{h,i},a_{h,i,k})\in\Theta\}$ stands for the set of all feasible mixing measures. 

\textbf{Assumptions.}
In our analysis, we make the following assumptions throughout, unless explicitly stated otherwise.
(\textbf{A.1})
The parameter space $\Theta$ is compact and the input space $\mathcal{X} \subset \mathbb{R}^{\overline{d}}$ is bounded. This guarantees the convergence of least squares estimation. 
(\textbf{A.2})
All mixture weights $\omega^*_{h,k}$ are nonnegative, and at least one among them is strictly positive. This assumption ensures that the MoE model is well-defined.
(\textbf{A.3})
The matrices $M_{h,i}$ are symmetric and
$M_{h,N^*} = \mathbf{0}_{\bar d \times \bar d}$, for all $h\in[H^*], i \in [N^*]$. This normalization is required for model identifiability because softmax weights are invariant to translations.
(\textbf{A.4})
For each $h\in[H^*]$, there exists 
$i \in [N^*-1]$ such that $M^*_{h,i} \neq \mathbf{0}_{\bar d \times \bar d}$; this ensures that the softmax weights depend on the input value.
(\textbf{A.5})
All the expert parameters
$a^*_{h,i,k}$, for all $h\in[H^*],i\in[N^*],k\in[K^*]$ are nonzero and pairwise distinct. 

We first show that any least-squared estimator of the regression function $f_{\widehat{G}_n}$ converges to the ground-truth regression function $f_{\Gs}$ at the parametric rate on the sample size in Proposition~\ref{prop:regression_rate_mha}.

\begin{proposition}
    \label{prop:regression_rate_mha}
    For the least squares estimator $\widehat{G}_n$, it holds that 
    \begin{align}
        \label{eq:mha_rate}
        \|f_{\widehat{G}_n}-f_{\Gs}\|_{L^2(\mu)}=\mathcal{O}_{P}([\log(n)/n]^{1/2}).
    \end{align}
\end{proposition}

The proof of Proposition~\ref{prop:regression_rate_mha} is provided in Appendix~\ref{appendix:regression_rate_mha}. 
This result implies that the regression function can be estimated at a nearly parametric rate. 
Next, to derive estimation rates for the parameters of the HMoE model of equation~\eqref{eq:mha_moe} from the above convergence result, we decompose the discrepancy $f_{\widehat{G}_n}-f_{G^*}$ into a sum of linearly independent components. 
This is achieved by applying a Taylor expansion to the product of the second-level gating weights and the expert function $u(x;M,a)
:=
\exp(x^\top M x)\cdot (a^\top x)$.
Unfortunately, there is an intrinsic interaction between the expert parameter $a$ and the function $u(x;M,a)$, which is captured by the partial differential equation (PDE) as follows 
\begin{equation}\label{eq:interaction}
a^\top \cdot 
\frac{\partial u(x;M,a)}{\partial a}
=
u(x;M,a).
\end{equation}
The above PDE indicates that the function $u$ and its partial derivatives from the Taylor expansion are linearly dependent. As a result, even if the regression error tends to zero, the coefficients of those terms, which encode parameter mismatches, need not vanish, so a small regression discrepancy does not necessarily imply parameter convergence.
To account for this coupling among parameters, we build a Voronoi loss function \citep{manole22refined}, and use it to elucidate how these interactions affect the parameter estimation rate.


\textbf{Voronoi loss.} For each mixing measure $G\in\mathcal{G}_{H^*,K,N^*}(\Theta)$, we consider the Voronoi cells
$\{V_{h,k}\equiv V_{h,k}(G):h \in [H^*], k \in[K^*]\}$ defined as the set of all pairs $(h',k')\in [H^*]\times [K]$ such that 
$$
\|f_{G}^{h',k'} - f_{G^*}^{h,k}\|\leq \|f_{G}^{h',k'} - f_{G^*}^{h_1,k_1}\|, \forall (h_1,k_1) \ne (h,k) \in [H^*]\times [K^*],
$$
where we define 
$
f_G^{h,k}(x) := \sum_{i=1}^{N^*}\frac{\exp(x^\top M_{h,i}x)}{\sum_{j=1}^N\exp(x^\top M_{h,j}x)} \cdot (a_{h,i,k}^\top x). 
$
Next, for each pair $(h', k') \in V_{h,k}$, we let $\kappa^{h',k'}_G$ be the permutation of the set $[N^*]$ such that $$\|\theta_{h',\kappa^{h',k'}_G(i), k'} - \theta_{h,i,k}\|\le \|\theta_{h',\kappa^{h',k'}_G(i), k'} - \theta_{h,i',k}\|,$$ for any $i' \ne i$, where $\theta_{h,i,k} := (M_{h,i}, a_{h,i,k})$. 
Without loss of generality, we assume that $\kappa_G^{h',k'}(i)=i$ for all $(h',k')\in[H^*]\times[K^*]$.
Then, the Voronoi loss of interest is given by
\begin{align*}
    \mathcal{L}_{1,r}(G,\Gs):=\sum_{h=1}^{H^*}\sum_{k =1}^{K^*}\left|\sum_{(h',k') \in V_{h,k}}\omega_{h',k'} - \omega_{h,k}^*\right|+\sum_{h=1}^{H^*}\sum_{k=1}^{K^*}\sum_{(h',k') \in V_{h,k}} \omega_{h',k'}\cdot R_{h',h,k',k}(r),
\end{align*}
where we denote
\(
R_{h',h,k',k}(r)
=
\sum_{i=1}^{N^*}\|\Delta M_{h' h,i}\|^r
+
\|\Delta a_{h' h,i,k'k}\|^r
\), where  
$\Delta M_{h' h,i}:=M_{h',i}-M^*_{h,i}$ and $\Delta a_{h' h,i,k' k}:=a_{h',i,k'}-a^*_{h,i,k}$. Based on the above Voronoi loss, we obtain the following minimax lower bound for estimating the true mixing
measure $G^\ast$.

\begin{theorem}
\label{theorem:minimax_lower_bound_voronoi}
For any $r\geq 1$, it holds that 
\begin{align}
\inf_{G_n \in \mathcal{G}_{H^*,K,N^*}(\Theta)}
\sup_{G \in \mathcal{G}_{H^*,K,N^*}(\Theta)\setminus \mathcal{G}_{H^*,K^*-1,N^*}(\Theta)}
\mathbb{E}_{f_G}
\big[
\mathcal L_{1,r}(G_n, G)
\big]
\gtrsim
n^{-1/2}.
\label{eq:minimax_lower_bound_voronoi}
\end{align}
\end{theorem}
We remark that above, $\mathbb{E}_{f_G}$ denotes the expectation with respect to the joint distribution of the response and features obeying the model \eqref{eq:regression_model} with $G$ in place of $G^*$, and the infimum is taken over all estimators $G_n$ with values in $\mathcal{G}_{H^*,K,N^*}(\Theta)$.




The proof of Theorem~\ref{theorem:minimax_lower_bound_voronoi} is provided in Appendix~\ref{appendix:minimax_lower_bound_voronoi}. Combining the minimax lower bound with the Voronoi loss $\mathcal L_{1,r}$ shows that the estimation rates for parameters $\omega^*_{h,k}$, $M^*_{h,i}$, and $a^*_{h,i,k}$ are slower than any polynomial rate $\gO_P(n^{-1/(2r)})$, for any $r \ge 1$. This behavior suggests that parameter estimation may achieve a logarithmic-type convergence rate $\mathcal O_P(1/\log^\lambda(n))$, for some constant $\lambda>0$. 

The resulting slow parameter convergence also adversely impacts the estimation of the expert functions. Since the input space $\mathcal{X}$ is bounded, the expert map $x \mapsto a^\top x$ is Lipschitz continuous over $\mathcal{X}$. Hence, there exists a constant $L>0$ such that 
\begin{align}
\sup_{x \in \mathcal{X}} \bigl| (\widehat{a}_{h,i,k}^n)^\top x - (a^*_{h,i,k})^\top x \bigr|
&\le L \, \|\widehat{a}_{h,i,k}^n - a^*_{h,i,k}\|. 
\end{align} 
Consequently, expert estimation inherits the same slow rates for estimating parameters. As a consequence, we need exponentially many data points of the order $\gO(\exp(\varepsilon^{-1/\lambda}))$ to approximate these experts with a given error $\varepsilon$. As will be shown in the following section~\ref{sec:gated_attention}, this slow rate in standard MHA is effectively addressed by the gated attention mechanism, which applies a non-linear activation function to the output of the value projections or after the SDPA output. 

\section{Sample Complexity of Gated Attention}\label{sec:gated_attention}
In this section, we analyze the sample complexity of gated attention models in two settings where a nonlinear activation function is applied after the value projection (Setting I) and after the SDPA output (Setting II), respectively.

\subsection{Setting I}

We begin with Setting I, in which a nonlinear activation is applied after the value projection output. From the HMoE perspective in Section~\ref{sec:preliminaries}, this modification is equivalent to transforming a linear expert function $a^\top x$ in equation~\eqref{eq:mha_moe} to a non-linear one $\varphi(a^\top x)$. 

\textbf{Problem setup.}
We assume an i.i.d. sample $\{(X_i, Y_i)\}_{i=1}^n \subset \mathbb{R}^{\overline{d}}\times \mathbb{R}$ generated from the model
\begin{align}
    Y_i = g_{\Gs}(X_i) + \varepsilon_i, \quad i=1,2,\ldots,n,
    \label{eq:regression_model_gated}
\end{align}
where the regression function $x \mapsto g_{G^*}(x)$ is now 
\begin{align}\label{eq:gated_attention_model}
 g_{G^*}(x) := \sum_{h=1}^{H^*}\sum_{k=1}^{K^*}\omega^*_{h,k}\sum_{i=1}^{N^*}
    \frac{\exp(x^{\top}M^*_{h,i}x)}
    {\sum_{j=1}^{N^*}\exp(x^{\top}M^*_{h,j}x)}\cdot \varphi\left((a^*_{h,i,k})^{\top}x\right). 
\end{align}
Following the same reasoning as in Section~\ref{sec:MHA}, we tackle parameter estimation by applying a Taylor expansion to the function $\overline{u}(x;M,a):= \exp(x^\top Mx)\cdot \varphi(a^\top x)$, which allows us to decompose the discrepancy $g_G(x) - g_{G^*}(x)$ into a collection of terms. To rule out unwanted interactions, we assume a \emph{type-1 strong identifiability} condition on the activation function $\varphi(\cdot)$, which ensures that $\overline{u}(x;M,a)$ and its first- and second-order derivatives are linearly independent.  
\begin{definition}\label{def:type1}
    (Type-1 Strong Identifiability). A function $\varphi: \mathbb{R} \to \mathbb{R}$ is said to be \emph{type-1 strong identifiable} if it is injective, twice differentiable, and Lipschitz continuous, and if the collection of functions of $x$
\begin{align*}
\bigcup_{i=1}^N \bigcup_{k=0}^2\Big\{x^\alpha \exp(x^\top M_i x)\varphi^{(k)}(a_i^\top x): \alpha \in \mathcal{I}_k\Big\},
\end{align*}
where $\mathcal{I}_0 = \{\alpha \in \mathbb{N}^{\overline{d}}: |\alpha| \in \{0,2,4\}\}$, $\mathcal{I}_1 = \{\alpha \in \mathbb{N}^{\overline{d}}: |\alpha| \in \{1,3\}\}$, and $\mathcal{I}_2 = \{\alpha \in \mathbb{N}^{\overline{d}}: |\alpha| =2\}$,
is linearly independent for any pairwise distinct nonzero parameters $a_1,a_2,\ldots,a_N$. 
\end{definition}
The above condition helps address the detrimental PDE-typed interaction in equation~\eqref{eq:interaction}. In other words, it ensures the linear independence of terms in the Taylor expansion of the function $\overline{u}(x;M,a)$ in the decomposition of $g_G(x) - g_{G^*}(x)$, which is an important step in our proof techniques.

\textbf{Examples.} It can be verified that the function $\varphi(z) \in \{\mathrm{sigmoid}(z+b), \mathrm{tanh}(z+b), \mathrm{softplus}(z+b)\},$ where 
$b\ne 0$ is a bias, satisfies the type-1 strong identifiability condition. In contrast, the identity function $\varphi(z)=z$ fails to meet this condition as shown in the PDE in equation~\eqref{eq:interaction}. 

Next, we introduce a Voronoi-based loss to characterize the parameter estimation convergence rate in Setting I, as stated in Theorem~\ref{theorem:nonlinear_at_V}. For notational convenience, we reuse the Voronoi cells $V_{h,k}$ from Section~\ref{sec:MHA}, defined as the collection of index pairs $(h',k') \in [H^*] \times [K]$ such that 
$$
\|\overline{f}_{G}^{h',k'} - \overline{f}_{G^*}^{h,k}\|\leq \|\overline{f}_{G}^{h',k'} - \overline{f}_{G^*}^{h_1,k_1}\|, \forall (h_1,k_1) \ne (h,k) \in [H^*]\times [K^*],
$$
where 
$
\overline{f}_G^{h,k}(x) := \sum_{i=1}^{N^*}\frac{\exp(x^\top M_{h,i}x)}{\sum_{j=1}^N\exp(x^\top M_{h,j}x)} \cdot \varphi(a_{h,i,k}^\top x). 
$
Hence, the Voronoi loss is defined as
\begin{align*}
    &\mathcal{L}_{2}(G,\Gs)
    :=
    \sum_{h=1}^{H^*}\sum_{k=1}^{K^*}
    \left|
    \sum_{(h',k') \in V_{h,k}}\omega_{h',k'}
    -
    \omega_{h,k}
    \right|\\
    &\quad +
    \sum_{h=1}^{H^*}\sum_{k=1}^{K^*}
    \sum_{\substack{(h',k') \in V_{h,k} \\ |V_{h,k}| = 1}}
    \omega_{h',k'}\cdot R_{h',h,k',k}(1) +
    \sum_{h=1}^{H^*}\sum_{k=1}^{K^*}
    \sum_{\substack{(h',k') \in V_{h,k} \\ |V_{h,k}| > 1}}
    \omega_{h',k'}\cdot R_{h',h,k',k}(2).
\end{align*}
\begin{theorem}\label{theorem:nonlinear_at_V}
    Under Setting I of gated attention, suppose that the activation function $\varphi$ is type-1 strongly identifiable. Then, the following bound holds for any mixing measure $G \in \gG_{H^*,K,N^*}(\Theta)$: 
    $$
    \|g_G - g_{G^*}\|_{L^2(\mu)} \gtrsim \gL_2(G, G^*).
    $$
    This bound implies that $\gL_2(\widehat{G}_n, G^*) = \gO_P([\log(n)/n]^{1/2})$. 
\end{theorem}
The proof of Theorem~\ref{theorem:nonlinear_at_V} is provided in
Appendix~\ref{appendix:nonlinear_at_V}. 
This result reveals that parameters $(\omega_{h,k}^*, M_{h,i}^*, a_{h,i,k}^*)$ associated with a single fitted component $(|V_{h,k}| = 1)$ achieve the standard parametric estimation rate $\gO_P([\log(n)/n]^{1/2})$, ignoring logarithmic components. On the other hand, parameters that are approximated by multiple components $(|V_{h,k}| > 1)$ converge at a slower rate of order $\gO_P([\log(n)/n]^{1/4})$. Since the function $x \mapsto \varphi(a^\top x)$ is Lipschitz continuous, there exists a constant $L_1 > 0$ such that  
\begin{align}
\sup_{x \in \mathcal{X}} \bigl| \varphi\left((\widehat{a}_{h,i,k}^n)^\top x\right) - \varphi\left((a^*_{h,i,k})^\top x \right)\bigr| \le L_1 \|\widehat{a}_{h,i,k}^n - a^*_{h,i,k}\|.
\end{align}
Consequently, the expert estimation rate also lies between $\gO_P([\log(n)/n]^{1/4})$ and $\gO_P([\log(n)/n]^{1/2})$, depending on the number of fitted experts $|V_{h,k}|$. Equivalently, achieving an approximation error $\varepsilon$ requires at most a polynomial number of samples on the order of $\gO(\varepsilon^{-4})$ to approximate these experts with a given error $\varepsilon$, which is substantially smaller than in the original MHA model in Section~\ref{sec:MHA}.  
\subsection{Setting II}

We now turn to Setting~II of Gated Attention in which the non-linear activation is applied after the SDPA output.

\textbf{Problem setup.}
We assume that 
$\{(X_i,Y_i)\}_{i=1}^{n}\subset\mathbb{R}^{\bar d}\times\mathbb{R}$
is an i.i.d. sample of size $n$ generated according to the following model
\begin{align}
    Y_i = g_{\Gs}(X_i) + \varepsilon_i, \quad i=1,2,\ldots,n,\label{eq:regression_model_gated_II}
\end{align}
where the regression function $g_{G^*}(\cdot)$ takes the form:
\begin{align}\label{eq:gated_SPDA_model}
    g_{G^*}(x) := \sum_{h=1}^{H^*}\sum_{k=1}^{K^*}\omega^*_{h,k}\cdot \varphi\Bigg(\sum_{i=1}^{N^*}
    \frac{\exp(x^{\top}M^*_{h,i}x)}
    {\sum_{j=1}^{N^*}\exp(x^{\top}M^*_{h,j}x)}
    \cdot (a^*_{h,i,k})^{\top}x\Bigg),   
\end{align}
The key difference between the regression functions in Setting~I and Setting~II lies in where the non-linear activation $\varphi(\cdot)$ is applied. In particular, this activation introduces non-linearity into the product of the second-level gating weights and the expert functions, rather than acting only on the expert functions as in Setting~I. In this setting, the model can be viewed as an MoE model whose experts are themselves mixture models. Therefore, we impose an additional assumption to ensure that the HMoE model is well-defined and identifiable. 

\textbf{Assumption (A.6).} For each $h \in [H^*]$, the expert parameters $M_{h,i}^*$ and $a_{h,i,k}^*$ are pairwise distinct across $i \in [N^*]$ and $k \in [K^*]$. Furthermore, the discrepancies between $M_{h,i}^* - M_{h,j}^*$, for $1 \le i < j \le N^*$ and $1 \le h \le H^*$, are pairwise distinct. 

As in Setting~I, our main challenge is to establish the required $L^2$ lower bound 
$\|g_G - g_{G^*}\|_{L^2(\mu)} \gtrsim \gL_2(G, G^*)$, 
which is achieved by decomposing $g_{G_n}(x) - g_{G^*}(x)$ into linearly independent terms via a Taylor expansion of the function $\varphi(f_{G}^{h,k}(x))$. Therefore, we again need to establish a \emph{type-2 strong identifiability} condition on $\varphi$ to ensure the linear independence property. 

\begin{definition}
    (Type-2 Strong Identifiability). A function $\varphi: \mathbb{R} \to \mathbb{R}$ is said to be \emph{type-2 strong identifiable} if it is injective, twice differentiable, and Lipschitz continuous, and if the collection of functions of $x$
    \begin{align*}
&\Big\{\varphi\left(f_{G^*}^{h,k}(x)\right)\Big\}\bigcup \Big\{x^{\alpha}u(x;M_{h,i}^*, a_{h,i,k}^*)\varphi'\left(f_{G^*}^{h,k}(x)\right): \alpha \in \mathcal{J} \Big\}_{(h,i,k) \in [H^*]\times [N^*]\times [K^*]} \\
&\bigcup \Big\{x^\alpha u(x;M_{h,i_1}^*, a_{h,i_1,k}^*) u(x;M_{h,i_2}^*, a_{h,i_2,k}^*)\varphi''\left(f_{G^*}^{h,k}(x)\right): \alpha \in \mathcal{J}\Big\}_{(h,i,k) \in [H^*]\times [N^*]\times [K^*]} 
\end{align*}
is linearly independent, where $f_{G^*}^{h,k}(x) = \sum_{i=1}^{N^*}\frac{\exp(x^\top M_{h,i}^* x)}{\sum_{j=1}^{N^*}\exp(x^\top M_{h,j}^* x)} \cdot (a_{h,i,k}^*)^\top x$. 
\end{definition}
\textbf{Examples.} Analogous to Setting I, the map $\varphi(z) \in \{\mathrm{sigmoid}(z+b), \mathrm{tanh}(z+b), \mathrm{softplus}(z+b)\}$ with a nonzero bias 
$b\ne 0$ satisfies the type-2 strong identifiability condition. By contrast, the identity activation $\varphi(z)=z$ does not meet this requirement and therefore cannot mitigate the interaction captured by the PDE in \eqref{eq:interaction}. 

Now, we present the convergence rate of parameter estimation under Setting II in Theorem~\ref{theorem:nonlinear_at_SDPA}. 
\begin{theorem}\label{theorem:nonlinear_at_SDPA}
    Under Setting II of gated attention, suppose that the activation function $\varphi$ is type-2 strongly identifiable. Then, the following bound holds for any mixing measure $G \in \gG_{H^*,K,N^*}(\Theta)$:
    $$
    \|g_G - g_{G^*}\|_{L^2(\mu)} \gtrsim \gL_2(G, G^*).
    $$
    This bound implies that $\gL_2(\widehat{G}_n, G^*) = \gO_P([\log(n)/n]^{\frac{1}{2}})$.
\end{theorem}
The proof of Theorem~\ref{theorem:nonlinear_at_SDPA} is provided in Appendix~\ref{appendix:nonlinear_at_SDPA}. 
Compared to the results of Setting I in Theorem~\ref{theorem:nonlinear_at_V}, the convergence rates of parameter estimation and expert estimation in Setting II remain unchanged, which range from $\gO_P([\log(n)/n]^{1/4})$ to $\gO_P([\log(n)/n]^{1/2})$. Consequently, to approximate experts within a given error $\varepsilon$, it requires $\mathcal O(\varepsilon^{-4})$ number of samples. 

\subsection{Practical implications}

Our theoretical results offer two notable insights for the design of attention mechanisms.

\emph{(I.1) Gated attention is more sample-efficient than Multi-head self-attention.} Our key theoretical insight is that gated attention substantially enhances the sample efficiency compared to multi-head self-attention. As summarized in Table~\ref{tab:rate_summary}, while multi-head self-attention requires exponentially many samples to recover experts, gated attention needs only a polynomial number of samples. This provides a clear statistical justification for adopting gated attention to obtain better performance as suggested in \citep{qiu2025gated}.

\emph{(I.2) Placing a non-linear activation at appropriate locations fundamentally improves sample efficiency.} Within gated attention architectures, not all placements of non-linearity are equally effective. We show that placing the non-linearity after the SDPA output or immediately after the value projection substantially improves sample efficiency. In particular, we rigorously demonstrate that the exponential-order sample complexity of vanilla MHA stems from the PDE-type interaction in equation~\eqref{eq:interaction}, induced by the linear expert form $a^{\top}x$. Applying a non-linear activation after the SDPA output or the values makes the experts non-linear, e.g., $\varphi(a^{\top}x)$. In contrast, placing the activation after the queries or keys keeps the experts linear, according to the connection between gated attention and HMoE in Section~\ref{sec:preliminaries}. Hence, the PDE-type interaction in equation~\eqref{eq:interaction} still remains, and the model's sample efficiency is not improved.

Furthermore, our theoretical results shed light on why SDPA Gated Attention can mitigate the attention sink phenomenon.

\emph{(I.3) Explanation for why Gated Attention can mitigate the attention sink phenomenon.} As demonstrated in \citep{qiu2025gated}, Gated Attention introduces sparse, input-dependent gating scores to suppress irrelevant value projections, thereby mitigating attention sinks and improving long-context extrapolation. From the HMoE perspective, this mechanism corresponds to assigning higher weights to informative experts while suppressing uninformative ones. Such selective weighting requires accurate estimation of gating parameters to distinguish between these experts. Since our analysis shows that Gated Attention has lower sample complexity than standard MHA, its gating parameters can be estimated more accurately with a given number of samples. Consequently, more reliable gating scores can facilitate the suppression of irrelevant experts, leading to sparser gating patterns.

\vspace{-0.3cm}
\section{Numerical Experiments}
\label{sec:experiments}
In this section, we conduct numerical experiments to empirically validate our theoretical findings that gated attention mechanisms are more sample-efficient compared to multi-head self-attention. 

\textbf{Overview of experiments:} We design three experiments. \textit{Experiment 1} uses a base configuration with $d = 2$, $H^* = N^* = K^* = 2$, and Sigmoid activation. \textit{Experiment 2} repeats the base configuration with tanh activation to test robustness to the choice of nonlinearity. \textit{Experiment 3} scales to $d = 3$ and $H^* = N^* = K^* = 3$ (with $K \in \{4, 5\}$ fitted channels) to evaluate larger configurations. We report Experiment 1 in the main text and defer Experiments 2 and 3 to Appendix~\ref{appendix:experiment}.

\textbf{Data generation:} 
For each sample size $n$, we generate i.i.d. samples $\{(X_i, Y_i)\}_{i=1}^n$ by drawing $X_i$ from the uniform distribution over $[-1,1]^d$ and sampling $Y_i$ from the true regression function specified in each theorem configuration: $f_{\Gs}(X)$ for multi-head self-attention (Model~\eqref{eq:mha_moe}), $g_{G^*}(X)$ for gated attention Setting I (Model~\eqref{eq:gated_attention_model}), and $g_{G^*}(X)$ for gated attention Setting II (Model~\eqref{eq:gated_SPDA_model}), with additive Gaussian noise $\varepsilon_i \sim \mathcal{N}(0, \nu^2)$ and $\nu = 0.1$.

\textbf{Experimental setup:} Ground-truth parameters are listed in Appendix~\ref{appendix:experiment}, Table~\ref{tab:true_params} (Experiments 1--2) and Table~\ref{tab:true_params_larger} (Experiment 3).

\textbf{Training procedure:}
For each $n \in [10^4,10^5]$, we conduct 10 runs. We consider $K \in \{3, 4\}$ fitted channels in Experiments 1--2 and $K \in \{4, 5\}$ in Experiment 3 to evaluate different degrees of over-specification. Parameters are initialized near the true values with perturbations of order $O(n^{-0.083})$, and we fit via gradient descent with backtracking line search for up to 1000 epochs and an initial learning rate $\eta=0.05$. We report the mean and standard deviation of the Voronoi losses in Figure~\ref{fig:combined} for Experiment 1, Figure~\ref{fig:combinedGELU} for Experiment 2, and Figure~\ref{fig:combinedLarger} for Experiment 3.

\begin{figure*}
    \centering
    \begin{subfigure}{0.32\textwidth}
        \centering
        \includegraphics[width=\textwidth]{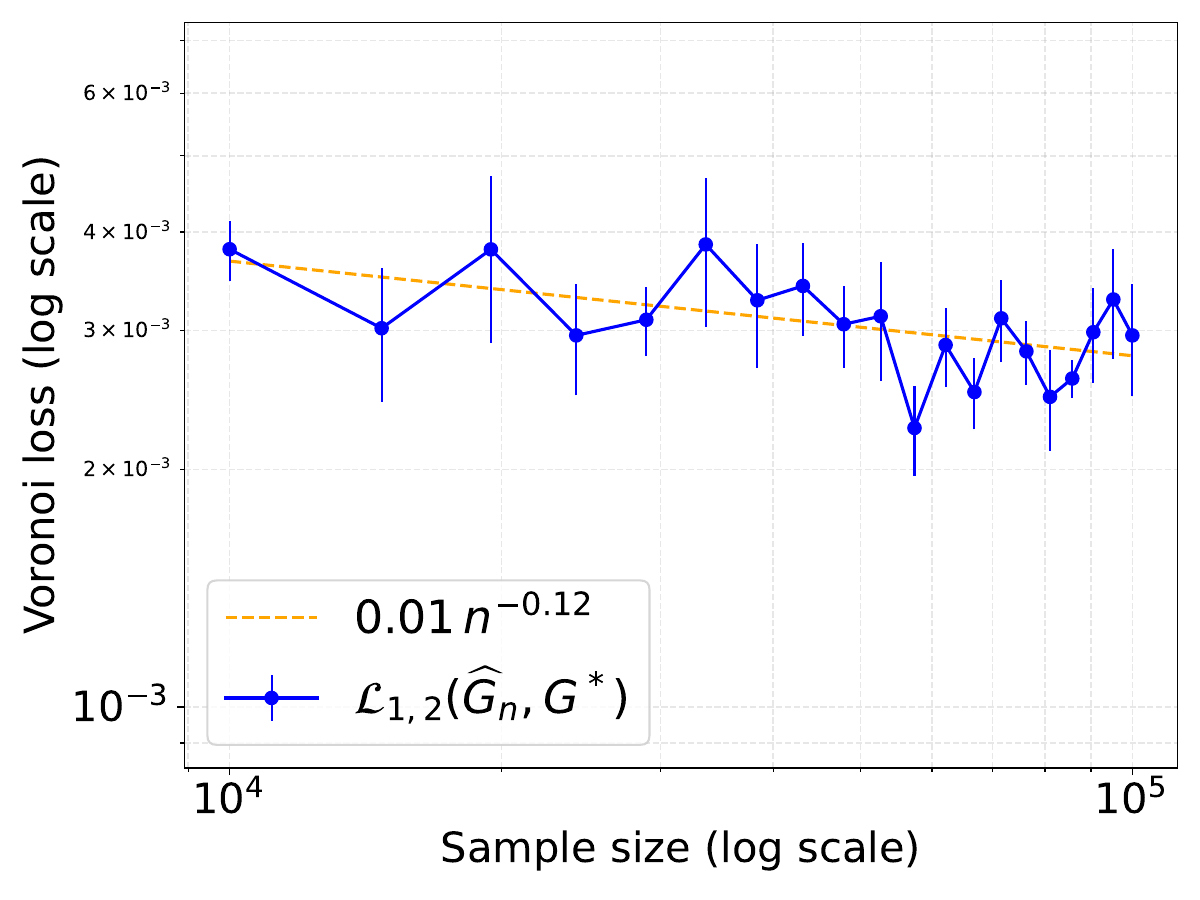}
        \caption{ MHA ($K=3$)}  
        \label{fig:sub1}
    \end{subfigure}
    \hfill
    \begin{subfigure}{0.32\textwidth}
        \centering
        \includegraphics[width=\textwidth]{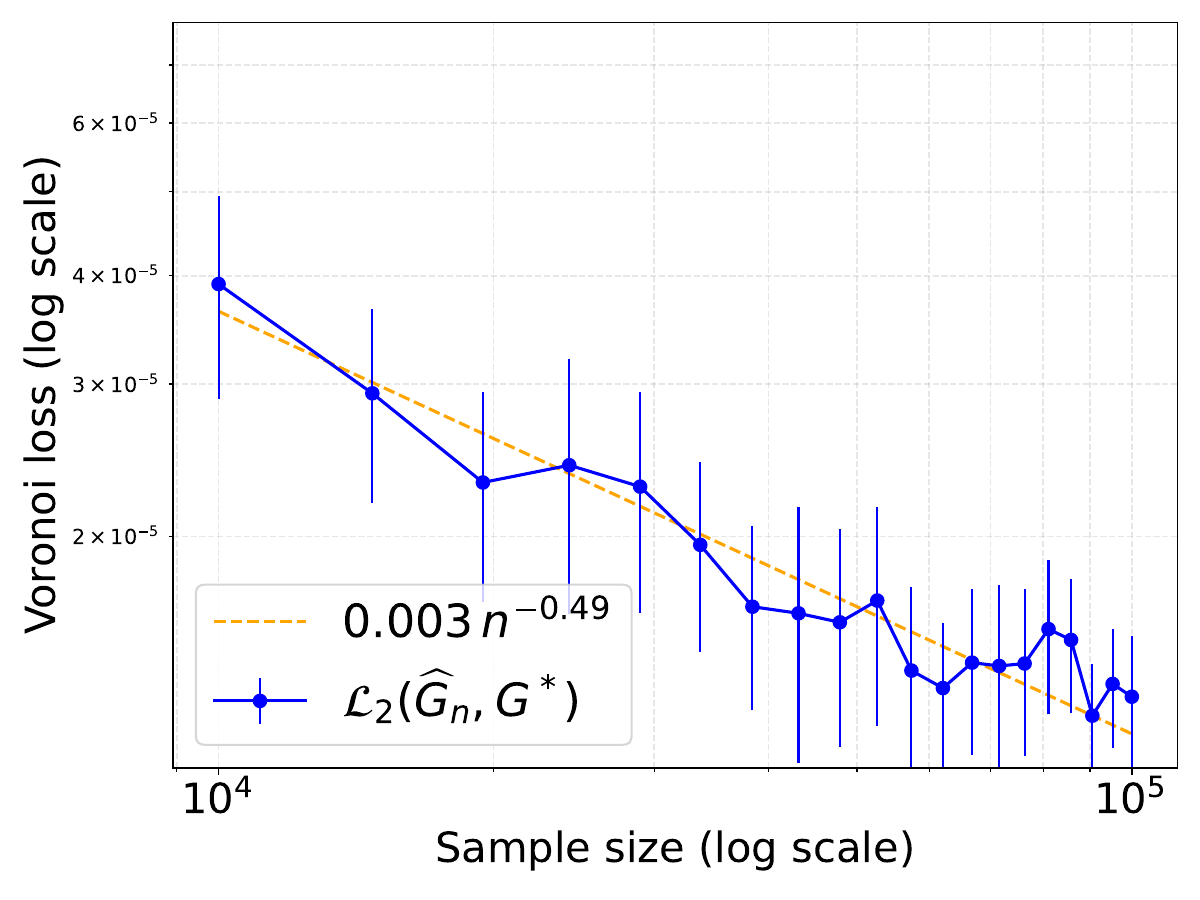}
        \caption{Value Gated Attention ($K=3$)}
        \label{fig:sub2}
    \end{subfigure}
    \hfill
    \begin{subfigure}{0.32\textwidth}
        \centering
        \includegraphics[width=\textwidth]{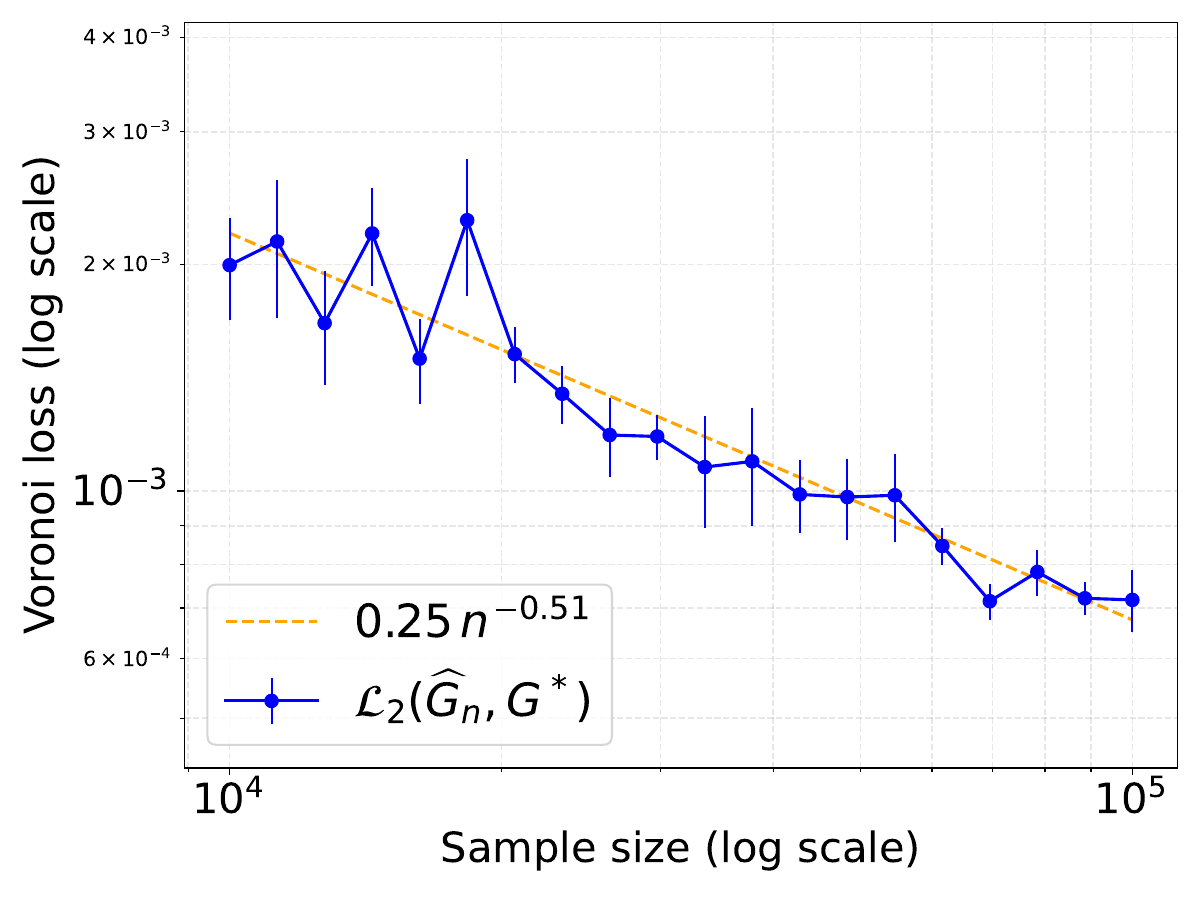}
        \caption{SDPA Gated Attention ($K=3$)}
        \label{fig:sub3}
    \end{subfigure}
    \begin{subfigure}{0.32\textwidth}
        \centering
        \includegraphics[width=\textwidth]{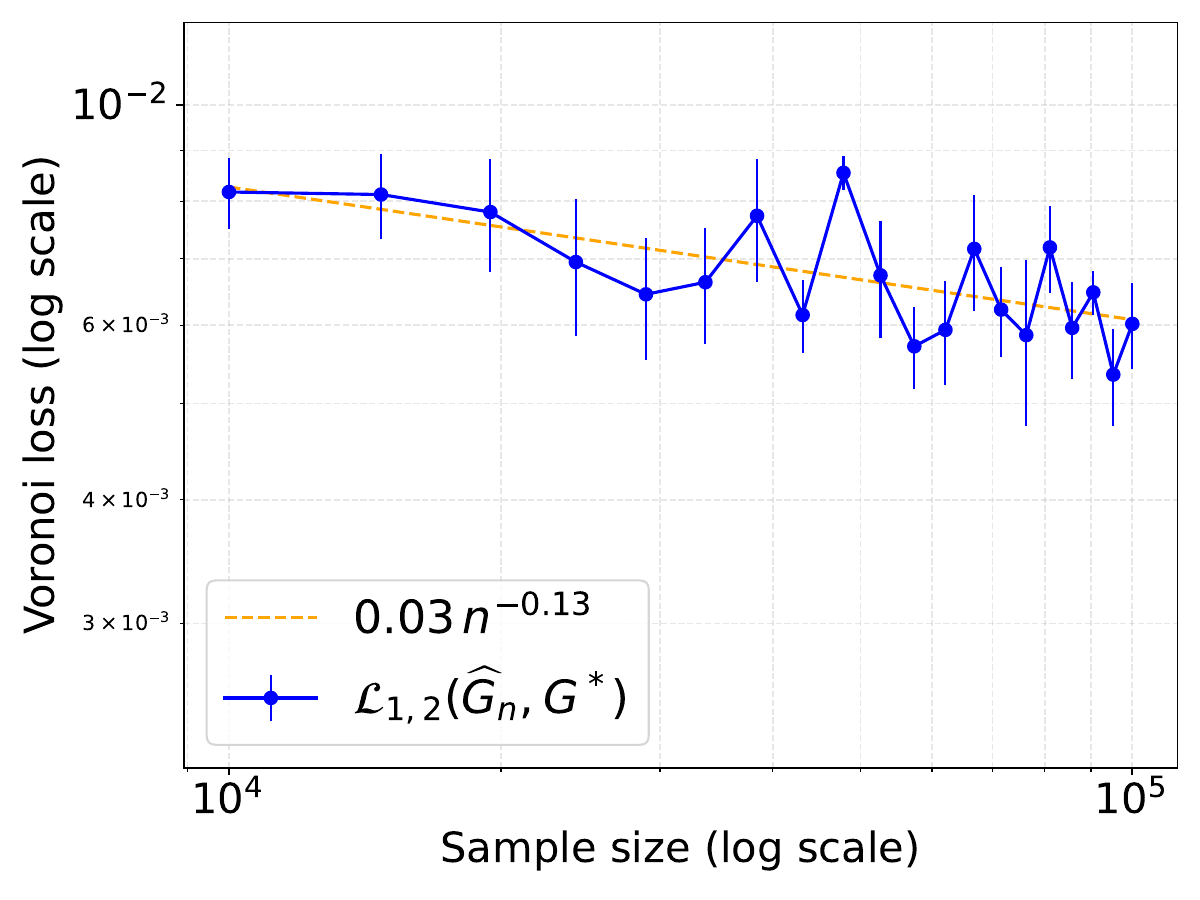}
        \caption{MHA ($K=4$)}
        \label{fig:sub4}
    \end{subfigure}
    \hfill
    \begin{subfigure}{0.32\textwidth}
        \centering
        \includegraphics[width=\textwidth]{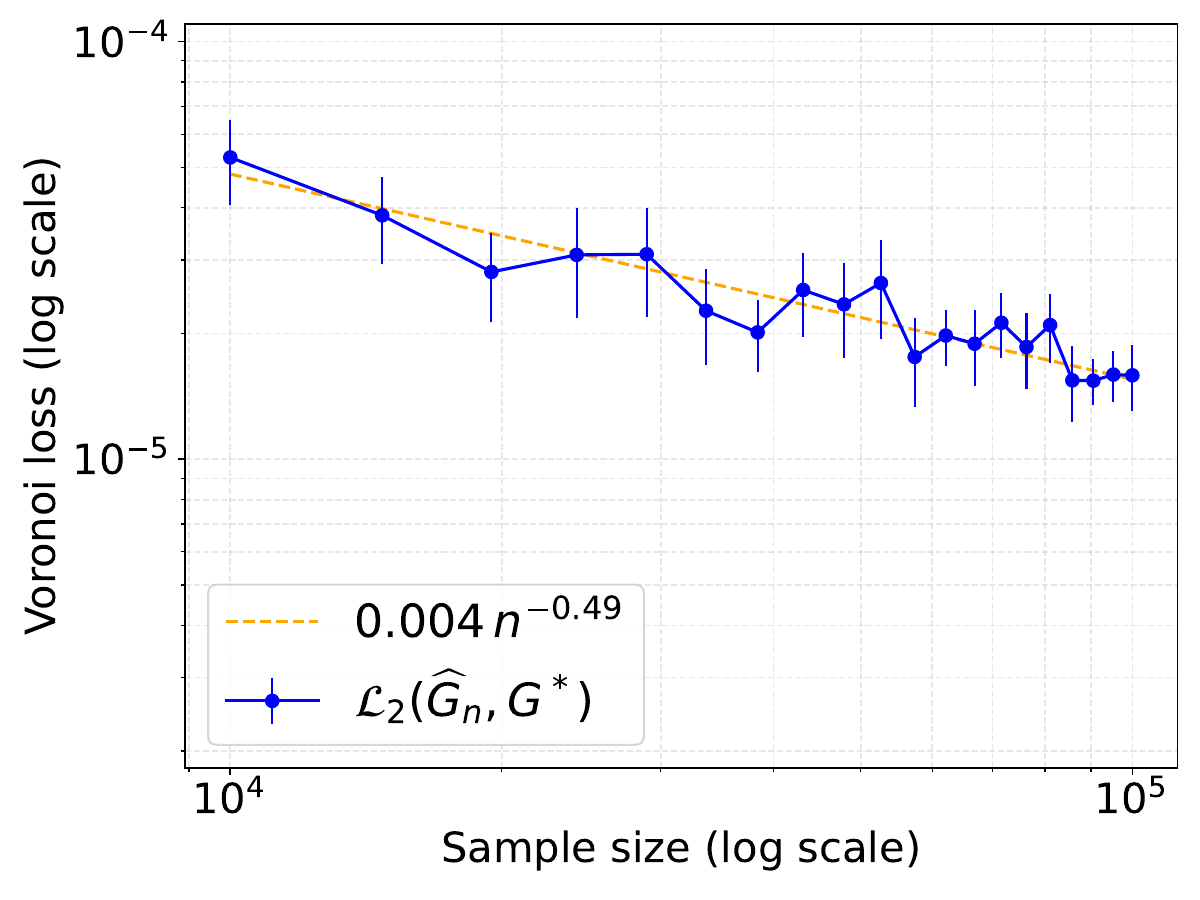}
        \caption{Value Gated Attention ($K=4$)}
        \label{fig:sub5}
    \end{subfigure}
    \hfill
    \begin{subfigure}{0.32\textwidth}
        \centering
        \includegraphics[width=\textwidth]{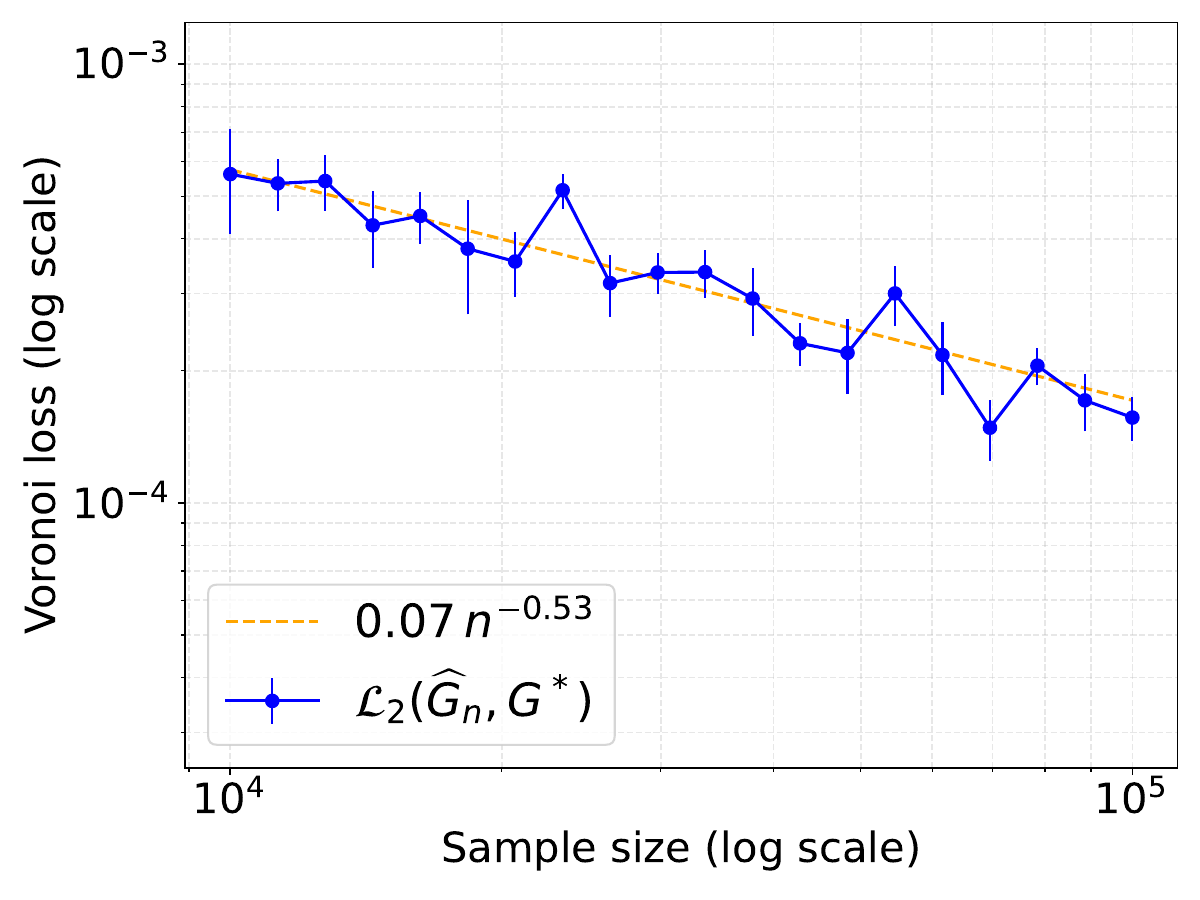}
        \caption{SDPA Gated Attention ($K=4$)}
        \label{fig:sub6}
    \end{subfigure}
\captionsetup{justification=justified,singlelinecheck=false}
    \caption{Log-log plots of empirical Voronoi losses versus sample size $n$ for multi-head self-attention and gated attention mechanisms with Sigmoid activation function (Experiment 1).}
    \label{fig:combined}
\end{figure*}

\textbf{Results of Experiment 1:} Figure~\ref{fig:combined} shows that gated attention with Sigmoid activation achieves fast convergence under both settings. In Setting~I, the Voronoi loss decays at empirical rates of about $O(n^{-0.49})$ in Figures~\ref{fig:sub2} and~\ref{fig:sub5}. In Setting~II, the rates are similarly fast, around $O(n^{-0.51})$ and $O(n^{-0.53})$ in Figures~\ref{fig:sub3} and~\ref{fig:sub6}. By contrast, vanilla multi-head self-attention converges much more slowly, with rates $O(n^{-0.12})$ and $O(n^{-0.13})$ in Figures~\ref{fig:sub1} and~\ref{fig:sub4}. These results empirically confirm the sample-efficiency advantage of gated attention.


\vspace{-0.4cm}
\section{Discussion}
\vspace{-0.1cm}
\label{sec:discussion}
This paper provides a theoretical justification for the advantage of gated attention over vanilla MHA. First, we establish a novel link between gated attention/MHA and hierarchical mixture of experts. From the perspective of expert specialization in HMoE, we show that MHA suffers from exponential-order sample complexity due to a PDE-type interaction. In contrast, by introducing non-linearity either after the SDPA output or after the value projection, gated attention improves the sample complexity to polynomial order. These results indicate that gated attention is more sample-efficient than multi-head self-attention. More importantly, our theory reveals that the PDE-type interaction in multi-head self-attention is caused by the linear form of experts. Thus, placing the non-linearity after the SDPA output or the values resolves this issue by making the experts non-linear, whereas applying it at other positions leaves the experts linear and the problem unresolved.

There are some potential research directions that arise from our work. First, our theoretical analysis is conducted in a well-specified setting where the data are generated from a regression framework whose regression function is an HMoE model. Future work can extend our results to cases where the ground-truth function is not necessarily an HMoE model. Second, this paper compares gated attention with multi-head self-attention through the statistical problem of estimating experts in their representations. We believe that comparing these two models from another perspective, such as universal approximation capability as in \citep{ramapuram2024sigmoidattention}, would further strengthen the claim that gated attention can outperform multi-head self-attention.

\vspace{-0.4cm}
\section*{Impact Statement}
\vspace{-0.1cm}
This paper aims to contribute to the advancement of machine learning. Due to its primarily theoretical focus, we do not anticipate any direct societal impacts arising from this work.


\bibliography{references}
\bibliographystyle{abbrv}

\newpage
\onecolumn
\appendix

\centering
\textbf{\Large{Supplementary Material for \\ 
``A Statistical Theory of Gated Attention Mechanism through the Lens of Hierarchical Mixture of Experts''}}

\justifying
\setlength{\parindent}{0pt}
In this supplementary material, we first expand on related works in Section~\ref{sec:related_works}, providing additional context and connections to existing literature. We then present rigorous proofs for results regarding the convergence rates of parameter estimation under the hierarchical mixture of experts representations of attention mechanisms in Section~\ref{appendix:main_results}. Specifically, we provide proofs for Theorem~\ref{theorem:minimax_lower_bound_voronoi} on multi-head self-attention in Appendix~\ref{appendix:minimax_lower_bound_voronoi}, Theorem~\ref{theorem:nonlinear_at_V} on gated attention (Setting I) in Appendix~\ref{appendix:nonlinear_at_V}, and Theorem~\ref{theorem:nonlinear_at_SDPA} on gated attention (Setting II) in Appendix~\ref{appendix:nonlinear_at_SDPA}. Next, we provide proofs for auxiliary results, including Proposition~\ref{prop:regression_rate_mha} in Appendix~\ref{appendix:auxiliary_results}. Finally,  we provide the true parameter configurations used in our numerical experiments in Appendix~\ref{appendix:experiment}.

\section{Proofs of Main Results}
\label{appendix:main_results}

\subsection{Proof of Theorem~\ref{theorem:minimax_lower_bound_voronoi}\label{appendix:minimax_lower_bound_voronoi}}

\paragraph{Proof overview.}
We begin by establishing that
\begin{align}
\lim_{\varepsilon \to 0}
\inf_{G \in \mathcal{G}_k(\Theta):\, \mathcal L_{1, r}(G, G^\ast) \le \varepsilon}
\frac{\| f_G - f_{G^\ast} \|_{L^2(\mu)}}{\mathcal L_{1, r}(G, G^\ast)} = 0
\label{eq:vanishing_ratio}
\end{align}

for any $r \geq 1$. With this result, we then obtain the minimax lower bound stated in Theorem~\ref{theorem:minimax_lower_bound_voronoi}:
\begin{align}
\inf_{G_n \in \mathcal{G}_{H^*,K,N^*}(\Theta)}
\sup_{G \in \mathcal{G}_{H^*,K,N^*}(\Theta)\setminus \mathcal{G}_{H^*,K^*-1,N^*}(\Theta)}
\mathbb{E}_{f_G}
\big[
\mathcal L_{1,r}(G_n, G)
\big]
\;\gtrsim\;
n^{-1/2}.
\label{eq:minimax_lower_bound}
\end{align}

\paragraph{Proof of equation \eqref{eq:vanishing_ratio}.}
To show this, it is enough to exhibit a sequence of mixing measures $(G_n)$ with
\begin{align}
\mathcal L_{1, r}(G_n, G^\ast) \to 0
\quad \text{and} \quad
\frac{\| f_{G_n} - f_{G^\ast} \|_{L_2(\mu)}}{\mathcal L_{1, r}(G_n, G^\ast)} \to 0, \text{ as $n \to \infty$.}
\label{eq:vanishing_ratio_sequence}
\end{align}
Consider the sequence $(G_n)$ given by
\begin{align}
G_n := \sum_{h=1}^{H^*}\sum_{k=1}^{K}\omega^{(n)}_{h,k}
\sum_{i=1}^{N^*}\delta_{(M^{(n)}_{h,i},a^{(n)}_{h,i,k})},
\end{align}
where $K = K^*+1$ and for all $h\in[H^*]$ and $i\in[N^*]$, the parameters are specified as follows:
\begin{align}
\omega^{(n)}_{h,1} = 
\omega^{(n)}_{h,2} = \frac{1}{2}\omega^{*}_{h,1} + \frac{1}{2n^{r+1}},
\qquad
\omega^{(n)}_{h,k} = \omega^{*}_{h,k - 1}, \quad  3 \leq k \leq K,
\end{align}
\begin{align}
M^{(n)}_{h,i} = M^\ast_{h,i}, \quad 1 \le i \le N^\ast,
\end{align}
and for the expert parameters,
\begin{align}
a^{(n)}_{h,i,1} = a^\ast_{h,i,1} + \frac{1}{n}e_i,
\qquad
a^{(n)}_{h,i,2} = a^\ast_{h,i,1} - \frac{1}{n}e_i,
\qquad
a^{(n)}_{h,i,k} = a^\ast_{h,i,k-1}, \quad  3 \leq k \leq K,
\end{align}
with $e_i = J_i^\top u$, where $u \in \mathbb{R}^d$ is any unit vector satisfying $\|u\| = 1$. 

Under the above construction, the Voronoi cells $\{V_{h,k}(G_n)\}$ can be chosen as follows.
For each $h\in[H^*]$, we have
\[
V_{h,1}(G_n)=\{(h,1),(h,2)\},
\qquad
V_{h,k}(G_n)=\{(h,k+1)\},\quad 2\le k\le K^*.
\]
Indeed, the two components $(h,1)$ and $(h,2)$ are both constructed to be closest to
$f_{G^*}^{h,1}$, while for each $k\ge 2$ the component $(h,k+1)$ is an exact copy of
$f_{G^*}^{h,k}$ up to the index shift.

Next, for $(h',k')\in V_{h,1}(G_n)$, we have
\[
\Delta M_{h',i,\;k'}=0,
\qquad
\Delta a_{h',i,\;k'}
=
a^{(n)}_{h,i,k'}-a^*_{h,i,1}
=
\frac{(-1)^{k'-1}}{n}e_i,
\qquad k'\in\{1,2\},
\]
while for $2\le k\le K^*$ and $(h',k')=(h,k+1)\in V_{h,k}(G_n)$, we have
\[
\Delta M_{h',i,\;k'}=0,
\qquad
\Delta a_{h',i,\;k'}
=
a^{(n)}_{h,i,k+1}-a^*_{h,i,k}
=0,
\]
so that the corresponding $(k\ge 2)$ terms in $\mathcal L_{1,r}(G_n,G^*)$ vanish.

Hence, by the definition of $\mathcal L_{1,r}$, we obtain
\begin{align}
\mathcal L_{1,r}(G_n,G^*)
&=
\sum_{h=1}^{H^*}\sum_{k =1}^{K^*}
\left|\sum_{(h',k') \in V_{h,k}(G_n)}\omega^{(n)}_{h',k'} - \omega^*_{h,k}\right|
\nonumber\\
&\quad+
\sum_{h=1}^{H^*}\sum_{k=1}^{K^*}\sum_{(h',k') \in V_{h,k}(G_n)} \omega^{(n)}_{h',k'}
\sum_{i=1}^{N^*}\Big(\|\Delta M_{h',i,\;k'}\|^r+\|\Delta a_{h',i,\;k'}\|^r\Big)
\nonumber\\
&=
\frac{H^*}{n^{r+1}}
+
\sum_{h=1}^{H^*}\Big(\omega^{(n)}_{h,1}+\omega^{(n)}_{h,2}\Big)
\sum_{i=1}^{N^*}\big\|\frac{1}{n}e_i\big\|^{\,r}
\nonumber\\
&=
\frac{H^*}{n^{r+1}}
+
\sum_{h=1}^{H^*}(\omega^*_{h,1} + \frac{1}{n^{r+1}})\,N^*\,\frac{1}{n^r}
=
O\!\left(\frac{1}{{n^r}}\right).
\label{eq:voronoi_rate_sequence}
\end{align}

The above formulation ensures that
$\mathcal L_{1, r}(G_n, G^\ast) \to 0$ as $n \to \infty$. Therefore, it remains to establish \eqref{eq:vanishing_ratio_sequence}.
For this purpose, 
for each $(h,k)\in[H^\ast]\times[K]$, define
\begin{align}
u^{h,k}_G(x)
:=
\sum_{i=1}^{N^\ast}
\frac{\exp\!\big(x^\top M_{h,i}x\big)}
{\sum_{t=1}^{N^\ast}\exp\!\big(x^\top M_{h,t}x\big)}
\,(a_{h,i,k})^\top x.
\end{align}
Now, we have $
f_G(x)
=
\sum_{h=1}^{H^\ast}\sum_{k=1}^{K}\omega_{h,k}\,u^{h,k}_G(x).$
We can partition
\begin{align}
f_{G_n}(x)-f_{G^\ast}(x)
=
\sum_{h=1}^{H^\ast}
\Bigg[
\sum_{k'=1}^{K}\omega^{(n)}_{h,k'}\,u^{h,k'}_{G_n}(x)
-
\sum_{k=1}^{K^*}\omega^\ast_{h,k}\,u^{h,k}_{G^\ast}(x)
\Bigg].
\end{align}

Next, for each fixed $h\in[H^\ast]$, we decompose
\begin{align}
& 
\Bigg[
\sum_{k'=1}^{K}\omega^{(n)}_{h,k'}\,u^{h,k'}_{G_n}(x)
-
\sum_{k=1}^{K}\omega^\ast_{h,k}\,u^{h,k}_{G^\ast}(x)
\Bigg]
\nonumber\\
&=
\Bigg[
\omega^{(n)}_{h,1}u^{h,1}_{G_n}(x)
+
\omega^{(n)}_{h,2}u^{h,2}_{G_n}(x)
-
\big(\omega^{(n)}_{h,1}+\omega^{(n)}_{h,2}\big)u^{h,1}_{G^\ast}(x)
\Bigg]
\nonumber\\
&\quad+
\Big(\omega^{(n)}_{h,1}+\omega^{(n)}_{h,2}-\omega^\ast_{h,1}\Big)\,u^{h,1}_{G^\ast}(x)
\nonumber\\
&\quad+
\sum_{k=2}^{K^*}
\Big[
\omega^{(n)}_{h,k+1}u^{h,k+1}_{G_n}(x)
-
\omega^\ast_{h,k}u^{h,k}_{G^\ast}(x)
\Big]
\nonumber\\
&=: A_{n,h}(x) + B_{n,h}(x) + C_{n,h}(x).
\end{align}

We now evaluate the three terms. By the construction of $G_n$, the head-shift part
satisfies
$
\omega^{(n)}_{h,k+1} = \omega^\ast_{h,k},\ u^{h,k+1}_{G_n}(x)=u^{h,k}_{G^\ast}(x),\  2\le k\le K,
$
and hence 
$C_{n,h}(x)=0.$
Moreover, using $\omega^{(n)}_{h,1}=\omega^{(n)}_{h,2}$ together with the symmetric
perturbation
$
a^{(n)}_{h,i,1}=a^\ast_{h,i,1}+\frac{1}{n}e_i,\
a^{(n)}_{h,i,2}=a^\ast_{h,i,1}-\frac{1}{n}e_i,
$
we obtain exact cancellation in the pivot term:
\begin{align}
A_{n,h}(x)
=
\omega^{(n)}_{h,1}
\Big(u^{h,1}_{G_n}(x)+u^{h,2}_{G_n}(x)-2u^{h,1}_{G^\ast}(x)\Big)
=0.
\end{align}
Finally, since
$
\omega^{(n)}_{h,1}+\omega^{(n)}_{h,2}-\omega^\ast_{h,1}=\frac{1}{n^{r+1}},
$
the remaining term satisfies
\begin{align}
B_{n,h}(x)
=
\frac{1}{n^{r+1}}
u^{h,1}_{G^\ast}(x)
=
O\!\left(\frac{1}{n^{r+1}}\right).
\end{align}

Combining the above, we conclude that,
\[
f_{G_n}(x)-f_{G^\ast}(x)
=
\sum_{h=1}^{H^*}(A_{n,h}(x)+B_{n,h}(x)+C_{n,h}(x))
=
\sum_{h=1}^{H^*}B_{n,h}(x),
\]
where $A_{n,h}(x)=0$, $C_{n,h}(x)=0$, and $B_{n,h}(x)=O(\frac{1}{n^{r+1}})$ for almost every $x$. Since $\mathcal L_{1,r}(G_n,G^\ast)=O(\frac{1}{n^r})$, it follows that

\begin{align}
\frac{f_{G_n}(x)-f_{G^\ast}(x)}{\mathcal L_{1,r}(G_n,G^\ast)}
=
\frac{\sum_{h=1}^{H^*}B_{n,h}(x)}
{\mathcal L_{1,r}(G_n,G^\ast)}
\to 0
\quad \text{for almost every } x.
\end{align}

As a consequence,
\[
\frac{\| f_{G_n} - f_{G^\ast} \|_{L_2(\mu)}}{\mathcal L_{1, r}(G_n, G^\ast)} \to 0
\quad \text{as } n \to \infty.
\]
This completes the proof of \eqref{eq:vanishing_ratio}.

\paragraph{Proof of equation \eqref{eq:minimax_lower_bound}.}
Given that the noise variables $\varepsilon_i\mid x_i$ are Gaussian,
\[
Y_i \mid x_i \sim \mathcal{N^\ast}(f_{G^\ast}(x_i), \sigma^2)
\quad \text{for all } i \in [n].
\]
From \eqref{eq:vanishing_ratio}, for sufficiently small $\varepsilon>0$ and some constant $c>0$ (to be chosen later), we can find $G\ast' \in \mathcal{G}_k(\Theta)$ such that
\[
\mathcal L_{1, r}(G_\ast', G^\ast) = 2 \varepsilon
\quad \text{and} \quad
\| f_{G_\ast'} - f_{G^\ast} \|_{L_2(\mu)} \le c \cdot \varepsilon.
\]

Applying Le Cam's lemma~\cite{yu97lecam} and using the weak triangle inequality for $\mathcal L_{1, r}$, we obtain
\begin{align}
&\inf_{G_n \in \mathcal{G}_{H^*,K,N^*}(\Theta)}
\sup_{G \in \mathcal{G}_{H^*,K,N^*}(\Theta)\setminus \mathcal{G}_{H^*,K^*-1,N^*}(\Theta)}
\mathbb{E}_{f_G}
\big[ \mathcal L_{1, r}(G_n, G) \big]
\nonumber\\
&\gtrsim
\frac{\mathcal L_{1, r}(G_\ast', G^\ast)}{8}
\exp\!\Big(
- n \, \mathbb{E}_{x \sim \mu}
\big[
\mathrm{KL}(
\mathcal{N^\ast}(f_{G_\ast'}(x), \sigma^2),
\mathcal{N^\ast}(f_{G^\ast}(x), \sigma^2)
)
\big]
\Big)
\nonumber\\
&\gtrsim
\varepsilon \cdot
\exp\!\big(
- n \| f_{G_\ast'} - f_{G^\ast} \|_{L_2(\mu)}^2
\big)
\gtrsim
\varepsilon \cdot \exp(- c n \varepsilon^2).
\label{eq:lecam_bound}
\end{align}

Here, the second inequality uses
$
\mathrm{KL}\!\left(
\mathcal{N^\ast}(f_{G_\ast'}(x), \sigma^2),
\mathcal{N^\ast}(f_{G^\ast}(x), \sigma^2)
\right)
=
\frac{\big(f_{G_\ast'}(x)-f_{G^\ast}(x)\big)^2}{2\sigma^2}.
$

Choosing $\varepsilon = n^{-1/2}$ gives
\[
\varepsilon \cdot \exp(- c n \varepsilon^2)
= n^{-1/2} \exp(-c).
\]
Thus, the minimax lower bound \eqref{eq:minimax_lower_bound} follows, completing the proof.

\subsection{Proof of Theorem~\ref{theorem:nonlinear_at_V}}
\label{appendix:nonlinear_at_V}

In this proof, we first show that
\begin{equation}\label{eq:main_goal_in_theorem4}
    \inf_{G \in \gG_{H^*,K,N^*}(\Theta)}\|g_G - g_{G^*}\|_{L^2(\mu)}/ \gL_2(G,G^*) > 0
\end{equation}
We separate the proof of \eqref{eq:main_goal_in_theorem4} into a global estimate and a local estimate. The local portion requires showing that \eqref{eq:main_goal_in_theorem4} holds whenever $\gL_2(G,\Gs)$ is small enough as follows
\begin{equation}\label{theorem4:local}
    \lim_{\varepsilon \to 0} \inf_{G \in \gG_{H^*,K,N^*}(\Theta): \gL_2(G,G^*) \le \varepsilon} \|g_G - g_{G^*}\|_{L^2(\mu)}/ \gL_2(G,G^*) > 0. 
\end{equation}
In the local regime, the problem is addressed using Taylor expansion as the main analytical tool. In contrast, the global part of the proof examines the behavior when $\gL_2(G,\Gs)$ is sufficiently large.
\begin{equation}\label{theorem4:global}
    \inf_{G \in \gG_{H^*,K,N^*}(\Theta): \gL_2(G,G^*) > \varepsilon} \|g_G - g_{G^*}\|_{L^2(\mu)}/ \gL_2(G,G^*) > 0.
\end{equation}
\textbf{Proof of the local part (Equation~\eqref{theorem4:local}).}

Suppose that this local inequality fails. Then there exists a sequence of mixing measures $(G_n) \in \gG_{H^*K^*N}(\Theta)$ such that $\gL_2(G_n,G^*) \to 0$ and $\|g_{G_n} - g_{G^*}\|_{L^2(\mu)}/ \gL_2(G_n,G^*) \to 0$ when $n \to \infty$.\\
\textbf{Step 1.} First of all, we denote
$$
\overline{f}_{G_n}^{h',k'}(x) = \sum_{i=1}^{N^*}\frac{\exp(x^{\top}M^n_{h',i}x)}{\sum_{j=1}^{N^*}\exp(x^{\top}M^n_{h',j}x)}\cdot \varphi((a^n_{h',i,k'})^{\top}x);
$$
$$
\overline{f}^{h,k}_{G^*}(x) = \sum_{i=1}^{N^*}\frac{\exp(x^{\top}M^*_{h,i}x)}{\sum_{j=1}^{N^*}\exp(x^{\top}M^*_{h,j}x)}\cdot \varphi
((a^*_{h,i,k})^{\top}x).
$$
Then, we have $g_{G_n}(x)=\sum_{h'=1}^{H^*}\sum_{k'=1}^{K}\omega_{h',k'}^n\cdot \overline{f}_{G_n}^{h',k'}(x)$,
$g_{G^*}(x)=\sum_{h=1}^{H^*}\sum_{k=1}^{K^*}\omega_{h,k}^*\cdot \overline{f}_{G^*}^{h,k}(x)$. 
Consequently, the Voronoi loss function between $G_n$ and $G^*$ becomes
\begin{align*}
    \gL_{2n}:= \gL_2(G_n,\Gs)&=\sum_{h=1}^{H^*}\sum_{k =1}^{K^*}\left|\sum_{(h',k') \in V_{h,k}}\omega_{h',k'}^{n} - \omega_{h,k}\right|\\
    &+\sum_{h=1}^{H^*}\sum_{k=1}^{K^*}\sum_{\substack{(h',k') \in V_{h,k}, \\ |V_{h,k}| = 1}} \omega_{h',k'}^{n}\Big(\sum_{i=1}^{N^*}\|\Delta M_{h' h,i}\|+\|\Delta a_{h' h,i,k' k}\|\Big) \\ 
    &+\sum_{h=1}^{H^*}\sum_{k=1}^{K^*}\sum_{\substack{(h',k') \in V_{h,k}, \\ |V_{h,k}| > 1}} \omega_{h',k'}^{n}\Big(\sum_{i=1}^{N^*}\|\Delta M_{h' h,i}\|^2+\|\Delta a_{h' h,i,k' k}\|^2\Big),
\end{align*}
where $\Delta M_{h' h, i} := M_{h',i}^n - M_{h,i}^*$ and $\Delta a_{h' h, i , k' k} := a_{h',i,k'}^n - a_{h,i,k}^*$ for any $i \in [N^*]$.
We now decompose the following difference:
\begin{align}
    \gD_n(x) &:= g_{G_n}(x) - g_{G^*}(x) \nonumber\\
    & = \sum_{h=1}^{H^*}\sum_{k=1}^{K^*}\sum_{(h',k') \in V_{h,k}}\omega_{h',k'}^n\overline{f}^{h',k'}_{G_n}(x) - \sum_{h=1}^{H^*}\sum_{k=1}^{K^*}\omega_{h,k}^*\overline{f}^{h,k}_{G^*}(x) \nonumber\\
    &= \sum_{h=1}^{H^*}\sum_{k=1}^{K^*} \sum_{(h',k') \in V_{h,k}}\omega_{h',k'}^n\left[\overline{f}^{h',k'}_{G_n}(x) - \overline{f}^{h,k}_{G^*}(x)\right] \nonumber\\
    &\qquad + \sum_{h=1}^{H^*}\sum_{k=1}^{K^*}\left(\sum_{(h',k') \in V_{h,k}}\omega_{h',k'}^n - \omega_{h,k}^*\right) \overline{f}^{h,k}_{G^*}(x)\nonumber\\
    &= \sum_{h=1}^{H^*}\sum_{k=1}^{K^*} \sum_{(h',k') \in V_{h,k}}\omega_{h',k'}^n\gD_{n}^{h' h, k' k} + \sum_{h=1}^{H^*}\sum_{k=1}^{K^*}\left(\sum_{(h',k') \in V_{h,k}}\omega_{h',k'}^n - \omega_{h,k}^*\right) \overline{f}^{h,k}_{G^*}(x), \label{appendixSDPA:first-level decompose I}
\end{align}
where $\gD^{h' h, k' k}_n(x) := \overline{f}_{G_n}^{h',k'}(x)-\overline{f}_{G^*}^{h,k}(x)$. We define 
\begin{align*}
    E^{h'}_n(x) &= \sum_{j=1}^{N^*}\exp(x^{\top}M^n_{h',j}x) \\
    E^h_*(x) &= \sum_{j=1}^{N^*}\exp(x^{\top}M^*_{h,j}x) \\
    \overline{u}(x;M,a) &= \exp(x^\top Mx)\cdot \varphi(a^\top x) \\
    \overline{v}_n^{h',k'}(x;M) &= \exp(x^\top Mx)\cdot \overline{f}^{h',k'}_{G_n}(x),
\end{align*}
for all $h,h' \in [H^*]$ and $k' \in [K]$. 
Each term $\gD_n^{h' h, k' k}(x)$ can be decomposed as follows:
\begin{align*}
    \gD_n^{h' h, k' k}(x) =\,& \sum_{i=1}^{N^*}\dfrac{1}{E_n^{h'}(x)}\overline{u}(x;M_{h',i}^n, a^n_{h',i,k'}) - \sum_{i=1}^{N^*}\dfrac{1}{E_*^h(x)}\overline{u}(x;M_{h,i}^*, a^*_{h,i,k}) \\ 
    =\,& \dfrac{1}{E^h_*(x)}\sum_{i=1}^{N^*}\overline{u}(x;M^n_{h',i}, a^n_{h',i,k'}) - \dfrac{1}{E^h_*(x)}\sum_{i=1}^{N^*}\overline{u}(x;M^*_{h,i}, a^*_{h,i,k}) \\
    & - \left[\dfrac{1}{E^h_*(x)}\sum_{i=1}^{N^*}\overline{v}_n^{h', k'}(x;M^n_{h',i}) - \dfrac{1}{E^h_*(x)}\sum_{i=1}^{N^*}\overline{v}_n^{h',k'}(x;M^*_{h,i})\right] \\
    :=\,& A_n^{h' h, k' k}(x) - B_n^{h' h, k' k}(x)
\end{align*}
We now expand each term $A_n^{h' h, k' k}(x)$ and $ B_n^{h' h, k' k}(x)$ via a second-order Taylor expansion. Specifically,
\begin{align*}
     &A_n^{h' h, k' k}(x)
   \\=\,& \dfrac{1}{E^h_*(x)}\sum_{i=1}^{N^*}\overline{u}(x;M^n_{h',i}, a^n_{h',i,k'}) - \dfrac{1}{E^h_*(x)}\sum_{i=1}^{N^*}\overline{u}(x;M^*_{h,i}, a^*_{h,i,k}) \\
     =\,& \dfrac{1}{E^h_*(x)}\sum_{i=1}^{N^*}\left(\overline{u}(x;M^n_{h',i}, a^n_{h',i,k'}) - \overline{u}(x;M^*_{h,i}, a^*_{h,i,k})\right) \\
    =\,& \dfrac{1}{E^h_*(x)}\left (\sum_{i=1}^{N^*}\sum_{1 \le |t|\le 2} \dfrac{1}{t!} \left(\Delta M_{h' h,i}\right)^{t_1}\left(\Delta a_{h' h,i,k' k}\right)^{t_2}
    \dfrac{\partial^{|t_1|+|t_2|} \overline{u}}{\partial M^{t_1} \partial a^{t_2}}(x;M^*_{h,i}, a^*_{h,i,k})\right)  + \gR_{n,1}^{h' h, k' k}(x),
\end{align*}

where $t = (t_1, t_2) \in \mathbb{N}^{\overline{d} \times \overline{d}}\times \mathbb{N}^{\overline{d}}$ and  $\gR_{n,1}^{h' h, k' k}(x)$ are Taylor remainders such that $\gR_{n,1}^{h' h, k' k}(x) / \gL_{2n} \to 0$ when $n \to \infty$. 
Next, we denote
\begin{align*}
    A_n(x)&:=
    \sum_{h=1}^{H^*}\sum_{k=1}^{K^*} \sum_{(h',k') \in V_{h,k}}\omega_{h',k'}^nA_n^{h' h, k' k}(x)
\end{align*}
\begin{align*}
    B_n(x)&:=
    \sum_{h=1}^{H^*}\sum_{k=1}^{K^*} \sum_{(h',k') \in V_{h,k}}\omega_{h',k'}^nB_n^{h' h, k' k}(x)
\end{align*}
Recall that $\gL_{2n}$ is given by summing first-order terms $|\Delta M_{h' h, i}|$ and $|\Delta a_{h' h, i, k' k}|$ over singleton Voronoi cells $(|V_{h,k}|=1)$, and second-order terms $|\Delta M_{h' h,i}|^2$ and $|\Delta a_{h' h,i,k' k }|^2$ over cells with $|V_{h,k}|>1$.
Therefore, $A_n$ can be approximated by retaining only the dominant contributions:
\begin{align*}
    A_n(x)
    =\,& \sum_{h=1}^{H^*}\sum_{k=1}^{K^*} \dfrac{1}{E_*^h(x)}\sum_{\substack{(h',k') \in V_{h,k} \\ |V_{h,k}| = 1}} \omega_{h',k'}^{n}\sum_{i=1}^{N^*}\sum_{ |t| = 1} \dfrac{1}{t!} \\ &\qquad \qquad \qquad  \times \left(\Delta M_{h' h,i}\right)^{t_1}\left(\Delta a_{h' h,i,k' k}\right)^{t_2}
    \dfrac{\partial^{|t_1|+|t_2|} \overline{u}}{\partial M^{t_1} \partial a^{t_2}}(x;M^*_{h,i}, a^*_{h,i,k})\\
&+    \sum_{\substack{(h',k') \in V_{h,k} \\ |V_{h,k}| > 1}} \omega_{h',k'}^n\sum_{i=1}^{N^*}\sum_{1 \le |t|\le 2} \dfrac{1}{t!} \\
&\qquad \qquad \qquad  \times \left(\Delta M_{h' h,i}\right)^{t_1}\left(\Delta a_{h' h,i,k' k}\right)^{t_2}
    \dfrac{\partial^{|t_1|+|t_2|} \overline{u}}{\partial M^{t_1} \partial a^{t_2}}(x;M^*_{h,i}, a^*_{h,i,k}) + \gR_{n,2}(x)
\end{align*}

Similarly, decomposing $B_n(x)$ using the same reasoning yields
\begin{align*}
    B_n(x)
&= 
\sum_{h=1}^{H^*}\sum_{k=1}^{K^*}\dfrac{1}{E_*^h(x)} \sum_{\substack{(h',k') \in V_{h,k} \\ |V_{h,k}| = 1}}\omega_{h',k'}^n\sum_{i=1}^{N^*}\sum_{ |r| = 1} \dfrac{1}{r!} \left(\Delta M_{h' h,i}\right)^{r}
    \dfrac{\partial^{|r|} \overline{v}_n^{h',k'}}{\partial M^{r}}(x;M^*_{h,i})\\
    & +   \sum_{\substack{(h',k') \in V_{h,k} \\ |V_{h,k}| > 1}}\omega_{h',k'}^n\sum_{i=1}^{N^*} \sum_{1 \le |r|\le 2} \dfrac{1}{r!} \left(\Delta M_{h' h,i}\right)^{r}
    \dfrac{\partial^{|r|} \overline{v}_n^{h',k'}}{\partial M^{r}}(x;M^*_{h,i}) +\gR_{n,3}(x)
\end{align*}
where $r \in \mathbb{N}^{\overline{d} \times \overline{d}}$ and $\gR_{n,3}(x)$ is Taylor remainder such that $\gR_{n,3}(x)/\gL_{2n} \to 0$ as $n \to \infty$.

Now, we define $\gJ_t := \{(t_1,t_2) \in \mathbb{N}^{\overline{d}\times \overline{d}}\times \mathbb{N}^{\overline{d}}: |t_1| + |t_2| = |t|\}$. Combining the preceding results, the function $\gD_n^{h' h,k' k}(x)$ can be represented as 
\begin{align*}
    &\gD_n(x) \\
    &= \sum_{h=1}^{H^*}\sum_{k=1}^{K^*}\dfrac{1}{E_*^h(x)}\sum_{i=1}^{N^*}\sum_{|t|=1}^{1 + \mathbf{1}_{\{|V_{h,k}| > 1\}}}\sum_{(t_1,t_2) \in \gJ_t}\left (\frac{1}{t!} \sum_{(h',k') \in V_{h,k}}\omega_{h',k'}^n \left(\Delta M_{h' h,i}\right)^{t_1}\left(\Delta a_{h' h,i,k' k}\right)^{t_2}\right ) \\
    &\qquad \times \dfrac{\partial^{|t_1|+|t_2|} \overline{u}}{\partial M^{t_1} \partial a^{t_2}}(x;M^*_{h,i}, a^*_{h,i,k})\\
    &+\sum_{h=1}^{H^*}\sum_{k=1}^{K^*}\dfrac{1}{E_*^h(x)}\sum_{i=1}^{N^*}\sum_{|r|=1}^{1 + \mathbf{1}_{\{|V_{h,k}| > 1\}}}\frac{1}{r!}\sum_{(h',k') \in V_{h,k}}  \omega_{h',k'}^n\left(\Delta M_{h' h,i}\right)^{r}
    \dfrac{\partial^{|r|} \overline{v}_n^{h',k'}}{\partial M^{r}}(x;M^*_{h,i})\\
    &+ \sum_{h=1}^{H^*}\sum_{k=1}^{K^*}\left(\sum_{(h',k') \in V_{h,k}}\omega_{h',k'}^n - \omega_{h,k}^*\right) \overline{f}^{h,k}_{G^*}(x)+ \gR_{n,4}(x), 
\end{align*}

where $\gR_{n,4}(x)$ is Taylor remainder such that $\gR_{n,4}(x)/\gL_{2n} \to 0$ as $n \to \infty$ and we denote
\begin{align*}
    \overline{Y}_{t_1,t_2}^{h, k}(i) &= \frac{1}{t!} \sum_{(h',k') \in V_{h,k}}\omega_{h',k'}^n \left(\Delta M_{h' h,i}\right)^{t_1}\left(\Delta a_{h' h,i,k' k}\right)^{t_2}, \\
    \overline{Z}_{r}^{h' h,k'}(i) &= \frac{1}{r!}\omega_{h',k'}^n\left(\Delta M_{h' h,i}\right)^{r}
\end{align*}

\textbf{Step 2 (Non-vanishing coefficients). } \\ 
In this step, we will show that at least one coefficient in $\{\overline{Y}_{t_1,t_2}^{h, k}(i)/ \gL_{2n};  \left(\sum_{(h',k') \in V_{h,k}}\omega_{h',k'}^n - \omega_{h,k}^*\right)/ \gL_{2n}\}$ does not go to $0$ as $n \to \infty$. Suppose, for contradiction, that every such coefficient go to $0$. Summing over the coefficients of $\overline{f}_{\Gs}^{h,k}(x)$ for all $1 \le h \le H^*, 1 \le k \le K^*$, we obtain 
\begin{equation}\label{appendix3:term11}
    \dfrac{1}{\gL_{2n}
    }\sum_{h=1}^{H^*}\sum_{k=1}^{K^*} |\sum_{(h',k') \in V_{h,k}}\omega_{h',k'}^n - \omega_{h,k}^*| \to 0.
\end{equation}
For index $(h,k) \in [H^*]\times [K^*]$ such that $|V_{h,k}| = 1$, we take $(t_1,t_2) = (e_{uv},\mathbf{0}_{\overline{d}})$ and $ (t_1,t_2) = (\mathbf{0}_{\overline{d}\times \overline{d}},e_u)$, where $e_{uv}$ denotes the canonical basis matrix in $\mathbb{R}^{\overline{d}\times \overline{d}}$ with a $1$ in the $(u,v)$-th entry and $0$ elsewhere. Summing the limits $\overline{Y}_{t_1,t_2}^{h, k}(i)/\gL_2 \to 0$ for all $i \in [N^*]$, we obtain 
\begin{equation}\label{appendix3:term12}
    \dfrac{1}{\gL_2}\sum_{h=1}^{H^*}\sum_{k=1}^{K^*}\sum_{\substack{(h',k') \in V_{h,k}, \\ |V_{h,k}| = 1}} \omega_{h',k'}^{n}\Big(\sum_{i=1}^{N^*}\|\Delta M_{h' h,i}\|+\|\Delta a_{h' h,i,k' k}\|\Big)\to 0
\end{equation}
For index $(h,k) \in [H^*]\times [K^*]$ such that  $|V_{k | h}| > 1$, set $(t_1,t_2) = (2e_{uv},\mathbf{0}_{\overline{d}})$ and $ (t_1,t_2) = (\mathbf{0}_{\overline{d}\times \overline{d}},2e_u)$. Taking the summation with respect to the limits $\overline{Y}_{t_1,t_2}^{h, k}(i)/\gL_2 \to 0$ for all $i \in [N^*]$, we have     
\begin{equation}\label{appendix3:term13}
    \dfrac{1}{\gL_2}\sum_{h=1}^{H^*}\sum_{k=1}^{K^*}\sum_{\substack{(h',k') \in V_{h,k}, \\ |V_{h,k}| > 1}} \omega_{h',k'}^{n}\Big(\sum_{i=1}^{N^*}\|\Delta M_{h' h,i}\|^2+\|\Delta a_{h' h,i,k' k}\|^2\Big)\to 0
\end{equation}
Combining the results in Equations \eqref{appendix3:term11}, \eqref{appendix3:term12}, \eqref{appendix3:term13} yields $ 1 = \dfrac{\gL_2}{\gL_2} \to 0$, which is a contradiction. Consequently, at least one coefficient of the linearly independent functions $\gD_n(x)/\gL_2$ does not vanish as $n \to \infty$. 

\textbf{Step 3 (Application of the Fatou's lemma). } 
Let $\overline{m}_n$ be the maximum of the absolute values among the coefficients of the linearly independent functions in  $\gD_n(x)/\gL_{2n}$. Since at least one of these coefficients does not vanish, we have $1/\overline{m}_n \not \to \infty$ as $n \to \infty$. Applying Fatou's lemma, we obtain 
$$
0 = \lim_{n\to\infty}\dfrac{\|g_{G_n} - g_{G^*}\|_{L^2(\mu)}}{\overline{m}_n\gL_{2n}} \ge \int\liminf_{n \to \infty}\dfrac{|g_{G_n} - g_{G^*}|}{\overline{m}_n\gL_{2n}} d\mu(x) \ge 0.
$$
As a result, we achieve that 
$$
\liminf_{n \to \infty}\dfrac{|g_{G_n} - g_{G^*}|}{\overline{m}_n\gL_{2n}} = 0.
$$
When $n \to \infty$, we denote
\begin{align*}
    \dfrac{\sum_{(h',k') \in V_{h,k}}\omega_{h',k'}^n  - \omega_{h,k}^*}{\overline{m}_n\gL_{2n}} &\to \overline{\lambda}^{h,k}_{\omega};\\
\dfrac{\overline{Y}_{t_1,t_2}^{h, k}(i)}{\overline{m}_n\gL_{2n}} &\to \overline{\lambda}^{h, k}_{y,t_1,t_2,i};
\\
\dfrac{\overline{Z}_r^{h' h,k'}(i)}{\overline{m}_n\gL_{2n}} &\to \overline{\lambda}^{h' h,k'}_{z,r,i}; 
\\ \sum_{(h',k') \in V_{h,k}}\overline{\lambda}^{h' h,k'}_{z,r,i} &= \overline{\lambda}^{ h}_{z,r,i}
\end{align*} 

Moreover, since $\sum_{(h',k') \in V_{h,k}}\omega_{h',k'}^n \to \omega_{h,k}^*$ and $\overline{f}_{G_n}^{h' ,k'}(x) \to \overline{f}_{G^*}^{h,k}(x)$ as $n\to \infty$ for all $(h',k') \in V_{h,k}$, we have
\begin{align*}
    \overline{v}_n^{h',k'}(x;M) &\to \exp(x^\top Mx) \cdot \overline{f}_{G^*}^{h,k}(x):= \overline{v}^{h,k}_*(x;M)
\end{align*}
 Since $\overline{f}_{G^*}^{h,k}(x) = \dfrac{1}{E_*^h(x)}\sum_{i=1}^{N^*}\overline{u}(x;M_{h,i}^*, a_{h,i,k}^*)$, the limit $\displaystyle\liminf_{n \to \infty}\dfrac{|g_{G_n} - g_{G^*}|}{\overline{m}_n\gL_2} = 0$ can be expressed as 
\begin{align*}
&\dfrac{1}{E_*^h(x)}\sum_{h=1}^{H^*}\sum_{k=1}^{K^*}\sum_{i=1}^{N^*}\sum_{|t|=1}^{1 + \mathbf{1}_{\{|V_{h,k}| > 1\}}}\sum_{(t_1,t_2) \in \gJ_t}\overline{\lambda}^{h, k}_{y,t_1,t_2,i}\dfrac{\partial^{|t_1|+|t_2|} \overline{u}}{\partial M^{t_1} \partial a^{t_2}}(x;M^*_{h,i}, a^*_{h,i,k})\\
    &+\dfrac{1}{E_*^h(x)}\sum_{h=1}^{H^*}\sum_{k=1}^{K^*}\sum_{i=1}^{N^*}\sum_{|r|=1}^{1 + \mathbf{1}_{\{|V_{h,k}| > 1\}}}\overline{\lambda}^{ h}_{z,r,i}
    \dfrac{\partial^{|r|} \overline{v}_*^{h,k}}{\partial M^{r}}(x;M^*_{h,i})\\
    &+ \dfrac{1}{E_*^h(x)}\sum_{h=1}^{H^*}\sum_{k=1}^{K^*}\overline{\lambda}^{h,k}_{\omega} \sum_{i=1}^{N^*}\overline{u}(x;M_{h,i}^*, a_{h,i,k}^*)\\
    &=0
\end{align*}
for almost every $x$. It is important to note that $\overline{\lambda}_{y,t_1,t_2,N^*}^{h,k} = \overline{\lambda}_{z,r,N^*}^{h,k} = 0$, since $M_{h,N^*}^n = M_{h,N^*}^* = 0$ (due to assumption A.3). Hence, the left-hand side of the above equation can be expressed as a linear combination of the following functions: 
\begin{align*}
    \left\{\dfrac{\partial^{|t_1|+|t_2|} \overline{u}}{\partial M^{t_1} \partial a^{t_2}}(x;M^*_{h,i}, a^*_{h,i,k}), \dfrac{\partial^{|r|} \overline{v}_*^{h,k}}{\partial M^{r}}(x;M^*_{h,i})  \right\}_{0 \le |t_1| + |t_2| \le 2, 1 \le |r| \le 2, (h,k,i) \in [H^*]\times [K^*]\times [N^*-1]}
\end{align*}
Since the function $\varphi(\cdot)$ is type-1 strongly identifiable, the set above
is linearly independent for almost every $x$. 
Hence, all coefficients $\overline{\lambda}_{y,t_1,t_2,i}^{h,k}$, $\overline{\lambda}_{z,r,i}^h$, and $\overline{\lambda}_{\omega}^{h,k}$ must be zero. This yields a contradiction, establishing the desired result.

\textbf{Proof of the global part (Equation~\eqref{theorem4:global}).}

Assume, by contradiction, that the equation~\eqref{theorem4:global} is not true. Then, there exists a sequence $G_n' \in \gG_{H^*,K,N^*}(\Theta)$ such that $\gL_2(G_n', G^*) > \varepsilon$ and $\|g_{G_n'} - g_{G^*}\|_{L^2(\mu)} \to 0$, as $n \to \infty$. Recall that $\Theta$ is a compact set, therefore, we can replace the sequence $G_n'$ by one of its subsequences that converges to a mixing measure $G' \in \gG_{H^*,K,N^*}(\Theta)$. Since $\gL_2(G_n', G^*) > \varepsilon$, we deduce that $\gL_2(G', G^*) > \varepsilon$. 

Next, by invoking Fatou's lemma, we have that 
\begin{equation}
    0 = \lim_{n\to \infty}\|g_{G'_n} - g_{G^*}\|_{L^2(\mu)}^2 \ge \int \liminf_{n\to \infty} \left|g_{G'_n}(x) - g_{G^*}(x)\right|^2\mathrm{d}\mu(x)
\end{equation}
Consequently, we obtain that $g_{G'}(x) = g_{G^*}(x)$ for almost every $x$. Then, we have 
\begin{align*}
    &\sum_{h=1}^{H^*}\sum_{k=1}^{K}\omega'_{h,k}\sum_{i=1}^{N^*}
    \frac{\exp(x^{\top}M'_{h,i}x)}
    {\sum_{j=1}^{N^*}\exp(x^{\top}M'_{h,j}x)} \cdot \varphi\left((a'_{h,i,k})^{\top}x\right)\\
    &= \sum_{h=1}^{H^*}\sum_{k=1}^{K^*}\omega^*_{h,k}\sum_{i=1}^{N^*}
    \frac{\exp(x^{\top}M^*_{h,i}x)}
    {\sum_{j=1}^{N^*}\exp(x^{\top}M^*_{h,j}x)} \cdot  \varphi\left((a^*_{h,i,k})^{\top}x\right)
\end{align*}
Since the function $\varphi(\cdot)$ is type-$1$ strong identifiable, the set of functions $\varphi(a^\top x)$ is linearly independent with distinct parameters $a$. Therefore, we deduce that for each $i \in [N^*]$ and $(h,k) \in [H^*]\times [K^*]$, there exists a set $V_{i,h,k} = \Big\{(i_1,h_1,k_1) \in [N^*]\times [H^*] \times [K]: \varphi((a'_{h_1,i_1,k_1})^\top x) = \varphi((a_{h,i,k}^*)^\top x)\Big\}$. Without loss of generality, we assume that $\varphi((a'_{h_1,i,k_1})^\top x) = \varphi((a^*_{h,i,k})^\top x)$ for all $i \in [N^*]$ and $(h_1,k_1) \in V_{h,k}$ for all $(h,k)\in [H^*] \times [K^*]$. Consequently, we have $a'_{h_1,i,k_1} = a_{h,i,k}^*$ (since the function $\varphi(\cdot)$ is injective) and 
$$
\sum_{(h_1,k_1) \in V_{h,k}} \omega'_{h_1,k_1}\dfrac{\exp(x^\top M_{h_1,i}'x)}{\sum_{j=1}^{N^*}\exp(x^\top M_{h_1,j}'x)} = \omega^*_{h,k}\dfrac{\exp(x^\top M_{h,i}^*x)}{\sum_{j=1}^{N^*}\exp(x^\top M_{h,j}^*x)},
$$
for all $(i,h,k) \in [N^*]\times [H^*] \times [K^*]$. 
By summing the above equation over $i \in [N^*]$, we deduce that 
\begin{equation}\label{appendix:omega1}
 \sum_{(h_1,k_1) \in V_{h,k}} \omega'_{h_1,k_1} = \omega_{h,k}^*,\forall \, (h,k) \in [H^*] \times [K^*].   
\end{equation}
Consequently, we have 
$
\dfrac{\exp(x^\top M_{h_1,i}'x)}{\sum_{j=1}^{N^*}\exp(x^\top M_{h_1,j}'x)} =\dfrac{\exp(x^\top M_{h,i}^*x)}{\sum_{j=1}^{N^*}\exp(x^\top M_{h,j}^*x)}, \forall \, (h_1,k_1) \in V_{h,k}.
$
Since the set of matrices $\{M_{h,i}\}_{i=1}^{N^*}$ is invariant to translations, we deduce that 
\begin{equation}\label{appendix:m1}
    M'_{h_1,i} = M_{h,i}^*, \forall \, i \in [N^*], (h_1,k_1) \in V_{h,k}.
\end{equation}
Combining \eqref{appendix:omega1} with \eqref{appendix:m1}, we obtain that $\gL_2(G', G^*) = 0$, which yields a contradiction. This completes the proof. 
\subsection{Proof of Theorem~\ref{theorem:nonlinear_at_SDPA}}\label{appendix:nonlinear_at_SDPA}
Our main goal is to demonstrate that
\begin{equation}\label{eq:main_goal_in_theorem3}
    \inf_{G \in \gG_{H^*,K,N^*}(\Theta)}\|g_G - g_{G^*}\|_{L^2(\mu)}/ \gL_2(G,G^*) > 0
\end{equation}
By following the same line of reasoning as in Theorem~\ref{theorem:nonlinear_at_V}, we decompose the inequality \eqref{eq:main_goal_in_theorem3} into a global and a local component. The local part is to establish equation~\ref{eq:main_goal_in_theorem3} when $\gL_2(G,\Gs)$ is small enough as follows
\begin{equation}\label{theorem3:local}
    \lim_{\varepsilon \to 0} \inf_{G \in \gG_{H^*,K,N^*}(\Theta): \gL_2(G,G^*) \le \varepsilon} \|g_G - g_{G^*}\|_{L^2(\mu)}/ \gL_2(G,G^*) > 0. 
\end{equation}
By contrast, the global part of the proof focuses on the behavior of this property in the regime where $\gL_2(G,\Gs)$ becomes sufficiently large.
\begin{equation}\label{theorem3:global}
    \inf_{G \in \gG_{H^*,K,N^*}(\Theta): \gL_2(G,G^*) > \varepsilon} \|g_G - g_{G^*}\|_{L^2(\mu)}/ \gL_2(G,G^*) > 0.
\end{equation}
\textbf{Proof of the local part (Equation~\eqref{theorem3:local}).}

Assume that this local inequality does not hold. Consequently, we have a sequence of mixing measures $(G_n) \in \gG_{H^*,K,N^*}(\Theta)$ such that $\gL_2(G_n,G^*) \to 0$ and $\|g_{G_n} - g_{G^*}\|_{L^2(\mu)}/ \gL_2(G_n,G^*) \to 0$ when $n \to \infty$.

\textbf{Step 1 (Decomposing the discrepancy between regression functions).} First of all, we recall that 
$$
f_{G_n}^{h',k'}(x) = \sum_{i=1}^{N^*}\frac{\exp(x^{\top}M^n_{h',i}x)}{\sum_{j=1}^{N^*}\exp(x^{\top}M^n_{h',j}x)}\cdot (a^n_{h',i,k'})^{\top}x;
$$
$$
f^{h,k}_{G^*}(x) = \sum_{i=1}^{N^*}\frac{\exp(x^{\top}M^*_{h,i}x)}{\sum_{j=1}^{N^*}\exp(x^{\top}M^*_{h,j}x)}\cdot (a^*_{h,i,k})^{\top}x.
$$
Now, we have
\begin{align*}
 g_{G_n}(x) &= \sum_{h'=1}^{H^*}\sum_{k'=1}^{K}\omega_{h',k'}^n\varphi\left(f^{h',k'}_{G_n}(x)\right)   \\
 g_{G^*}(x) &= \sum_{h=1}^{H^*}\sum_{k=1}^{K^*}\omega_{h,k}^*\varphi\left(f^{h,k}_{G^*}(x)\right)
\end{align*}
As a result, the Voronoi loss function between $G_n$ and $G^*$ becomes
\begin{align*}
    \gL_{2n}:= \gL_2(G_n,\Gs)&=\sum_{h=1}^{H^*}\sum_{k =1}^{K^*}\left|\sum_{(h',k') \in V_{h,k}}\omega_{h',k'}^{n} - \omega_{h,k}\right|\\
    &+\sum_{h=1}^{H^*}\sum_{k=1}^{K^*}\sum_{\substack{(h',k') \in V_{h,k}, \\ |V_{h,k}| = 1}} \omega_{h',k'}^{n}\Big(\sum_{i=1}^{N^*}\|\Delta M_{h' h,i}\|+\|\Delta a_{h' h,i,k' k}\|\Big) \\ 
    &+\sum_{h=1}^{H^*}\sum_{k=1}^{K^*}\sum_{\substack{(h',k') \in V_{h,k}, \\ |V_{h,k}| > 1}} \omega_{h',k'}^{n}\Big(\sum_{i=1}^{N^*}\|\Delta M_{h' h,i}\|^2+\|\Delta a_{h' h,i,k' k}\|^2\Big),
\end{align*}
where $\Delta M_{h' h, i} = M_{h',i}^n - M_{h,i}^*$ and $\Delta a_{h' h, i , k' k} = a_{h',i,k'}^n - a_{h,i,k}^*$ for any $i \in [N^*]$.
Next, we decompose the following difference as follows
\begin{align}
    \gD_n(x) &:= g_{G_n}(x) - g_{G^*}(x) \nonumber\\
    & = \sum_{h=1}^{H^*}\sum_{k=1}^{K^*}\sum_{(h',k') \in V_{h,k}}\omega_{h',k'}^n\varphi\left(f^{h',k'}_{G_n}(x)\right) - \sum_{h=1}^{H^*}\sum_{k=1}^{K^*}\omega_{h,k}^*\varphi\left(f^{h,k}_{G^*}(x)\right) \nonumber\\
    &= \sum_{h=1}^{H^*}\sum_{k=1}^{K^*} \sum_{(h',k') \in V_{h,k}}\omega_{h',k'}^n\left[\varphi\left(f^{h',k'}_{G_n}(x)\right) - \varphi\left(f^{h,k}_{G^*}(x)\right)\right] \nonumber\\
    &\qquad + \sum_{h=1}^{H^*}\sum_{k=1}^{K^*}\left(\sum_{(h',k') \in V_{h,k}}\omega_{h',k'}^n - \omega_{h,k}^*\right) \varphi\left(f^{h,k}_{G^*}(x)\right)\nonumber\\
    &= \sum_{h=1}^{H^*}\sum_{k=1}^{K^*} \sum_{(h',k') \in V_{h,k}}\omega_{h',k'}^n\gD_{n}^{h' h, k' k} + \sum_{h=1}^{H^*}\sum_{k=1}^{K^*}\left(\sum_{(h',k') \in V_{h,k}}\omega_{h',k'}^n - \omega_{h,k}^*\right) \varphi\left(f^{h,k}_{G^*}(x)\right), \label{appendixSDPA:first-level decompose}
\end{align}
where $\gD^{h' h, k' k}_n(x) := \varphi\left(f^{h',k'}_{G_n}(x)\right) - \varphi\left(f^{h,k}_{G^*}(x)\right)$. Next, we denote
\begin{align*}
    E^{h'}_n(x) &= \sum_{j=1}^{N^*}\exp(x^{\top}M^n_{h',j}x); \\
    E^h_*(x) &= \sum_{j=1}^{N^*}\exp(x^{\top}M^*_{h,j}x); \\
    u(x;M,a) &= \exp(x^\top Mx)\cdot a^\top x; \\
    v_n^{h',k'}(x;M) &= \exp(x^\top Mx)\cdot f^{h',k'}_{G_n}(x),
\end{align*}
for all $h,h' \in [H^*]$ and $k' \in [K]$. We now decompose each term $\gD_n^{h' h, k' k}(x)$ as follows
\begin{align*}
    \gD_n^{h' h, k' k}(x) =\,& \varphi\left(\sum_{i=1}^{N^*}\dfrac{1}{E_n^{h'}(x)}u(x;M_{h',i}^n, a^n_{h',i,k'})\right) - \varphi\left(\sum_{i=1}^{N^*}\dfrac{1}{E_*^h(x)}u(x;M_{h,i}^*, a^*_{h,i,k})\right) \\ 
    =\,& \varphi\left(\dfrac{1}{E^h_*(x)}\sum_{i=1}^{N^*}u(x;M^n_{h',i}, a^n_{h',i,k'})\right) - \varphi\left(\dfrac{1}{E^h_*(x)}\sum_{i=1}^{N^*}u(x;M^*_{h,i}, a^*_{h,i,k})\right) \\
    & - \left[\varphi\left(\dfrac{1}{E^h_*(x)}\sum_{i=1}^{N^*}v_n^{h', k'}(x;M^n_{h',i})\right) - \varphi\left(\dfrac{1}{E^h_*(x)}\sum_{i=1}^{N^*}v_n^{h',k'}(x;M^*_{h,i})\right)\right] \\
    :=\,& A_n^{h' h, k' k}(x) - B_n^{h' h, k' k}(x)
\end{align*}
We now decompose each term $A_n^{h' h, k' k}(x)$ and $ B_n^{h' h, k' k}(x)$ using Taylor expansion. In particular, by means of second-order Taylor expansion, we have
\begin{align*}
    &A_n^{h' h, k' k}(x) \\ 
    =\,& \varphi\left(\dfrac{1}{E^h_*(x)}\sum_{i=1}^{N^*}u(x;M^n_{h',i}, a^n_{h',i,k'})\right) - \varphi\left(\dfrac{1}{E^h_*(x)}\sum_{i=1}^{N^*}u(x;M^*_{h,i}, a^*_{h,i,k})\right) \\
    =\,& \varphi'\left(\dfrac{1}{E^h_*(x)}\sum_{i=1}^{N^*}u(x;M^*_{h,i}, a^*_{h,i,k})\right)\times  
    \Bigg[\dfrac{1}{E^h_*(x)}\sum_{i=1}^{N^*}\sum_{1 \le |t|\le 2} \dfrac{1}{t!}  \\
    & \qquad \times \left(\Delta M_{h' h,i}\right)^{t_1}\left(\Delta a_{h' h,i,k' k}\right)^{t_2}
    \dfrac{\partial^{|t_1|+|t_2|} u}{\partial M^{t_1} \partial a^{t_2}}(x;M^*_{h,i}, a^*_{h,i,k}) \Bigg] \\
    & +\dfrac{1}{2}\varphi''\left(\dfrac{1}{E^h_*(x)}\sum_{i=1}^{N^*}u(x;M^*_{h,i}, a^*_{h,i,k})\right)\Bigg[\dfrac{1}{E^h_*(x)}\sum_{i=1}^{N^*}\sum_{ |t| = 1} \dfrac{1}{t!} \\
    & \qquad \times  \left(\Delta M_{h' h,i}\right)^{t_1}\left(\Delta a_{h' h,i,k' k}\right)^{t_2}
    \dfrac{\partial^{|t_1|+|t_2|} u}{\partial M^{t_1} \partial a^{t_2}}(x;M^*_{h,i}, a^*_{h,i,k})\Bigg]^2 + \gR_{n,1}^{h' h, k' k}(x),
\end{align*}
where $t = (t_1, t_2) \in \mathbb{N}^{\overline{d} \times \overline{d}}\times \mathbb{N}^{\overline{d}}$ and  $\gR_{n,1}^{h' h, k' k}(x)$ are Taylor remainders such that $\gR_{n,1}^{h' h, k' k}(x) / \gL_{2n} \to 0$ when $n \to \infty$. Recall that the Voronoi loss function $\gL_{2n}$ is defined as the sum of first-order terms $\|\Delta M_{h' h, i}\|$ and $ \|\Delta a_{h' h, i, k' k}\|$ for singleton Voronoi cells $(|V_{h,k}|=1)$, and second-order terms $\|\Delta M_{h' h,i}\|^2$ and $ \|\Delta a_{h' h,i,k' k }\|^2$ for Voronoi cells with $|V_{h,k}|>1$. Consequently, we can provide a reduced representation of $A_n^{h'h,k'k}$ by focusing on the dominant terms as follows
\begin{align*}
    &A_n^{h' h, k' k}(x) \\
    =& \sum_{|V_{h,k}| = 1}\sum_{i=1}^{N^*}\sum_{ |t| = 1} \dfrac{1}{t!} \left(\Delta M_{h' h,i}\right)^{t_1}\left(\Delta a_{h' h,i,k' k}\right)^{t_2}
    \dfrac{\partial^{|t_1|+|t_2|} u}{\partial M^{t_1} \partial a^{t_2}}(x;M^*_{h,i}, a^*_{h,i,k})\varphi_{1n}^{h,k}(x)\\
&+   \sum_{|V_{h,k}| > 1}\sum_{i=1}^{N^*}\sum_{1 \le |t|\le 2} \dfrac{1}{t!} \left(\Delta M_{h' h,i}\right)^{t_1}\left(\Delta a_{h' h,i,k' k}\right)^{t_2}
    \dfrac{\partial^{|t_1|+|t_2|} u}{\partial M^{t_1} \partial a^{t_2}}(x;M^*_{h,i}, a^*_{h,i,k}) \varphi_{1n}^{h,k}(x)\\
&+ \dfrac{1}{2}\sum_{|V_{h,k}| > 1}\Bigg[\sum_{i=1}^{N^*}\sum_{ |t| = 1} \dfrac{1}{t!} \left(\Delta M_{h' h,i}\right)^{t_1}\left(\Delta a_{h' h,i,k' k}\right)^{t_2}
    \dfrac{\partial^{|t_1|+|t_2|} u}{\partial M^{t_1} \partial a^{t_2}}(x;M^*_{h,i}, a^*_{h,i,k})\Bigg]^2\varphi_{2n}^{h,k}(x)\\
    &+ \gR_{n,2}^{h' h, k' k}(x),
\end{align*}
where $\gR_{n,2}^{h' h, k' k}(x)$ is the remainder such that $\gR^{h' h, k' k}_{n,2}(x) / \gL_{2n} \to 0$ when $n \to \infty$ and 
\begin{align*}
    \varphi_{1n}^{h,k}(x) &:= \dfrac{1}{E^h_*(x)}\varphi'\left(\dfrac{1}{E^h_*(x)}\sum_{i=1}^{N^*}u(x;M^*_{h,i}, a^*_{h,i,k})\right); \\
    \varphi_{2n}^{h,k}(x) &:= \dfrac{1}{E^h_*(x)^2} \varphi''\left(\dfrac{1}{E^h_*(x)}\sum_{i=1}^{N^*}u(x;M^*_{h,i}, a^*_{h,i,k})\right)
\end{align*}
Next, we represent $A_n^{h' h, k' k}(x)$ as follows 

\begin{align*}
&A_n^{h' h, k' k}(x)\\
=& \sum_{|V_{h,k}| = 1}\sum_{i=1}^{N^*}\sum_{ |t| = 1} \dfrac{1}{t!} \left(\Delta M_{h' h,i}\right)^{t_1}\left(\Delta a_{h' h,i,k' k}\right)^{t_2}
    \dfrac{\partial^{|t_1|+|t_2|} u}{\partial M^{t_1} \partial a^{t_2}}(x;M^*_{h,i}, a^*_{h,i,k})\varphi_{1n}^{h,k}(x)\\
&+   \sum_{|V_{h,k}| > 1}\sum_{i=1}^{N^*}\sum_{1 \le |t|\le 2} \dfrac{1}{t!} \left(\Delta M_{h' h,i}\right)^{t_1}\left(\Delta a_{h' h,i,k' k}\right)^{t_2}
    \dfrac{\partial^{|t_1|+|t_2|} u}{\partial M^{t_1} \partial a^{t_2}}(x;M^*_{h,i}, a^*_{h,i,k}) \varphi_{1n}^{h,k}(x)\\ 
& + \dfrac{1}{2}\sum_{|V_{h,k}| > 1}\sum_{i_1=1}^{N^*}\sum_{i_2=1}^{N^*}\sum_{\substack{|t_1| = 1;\\|t_2| =1 }}  \left(\Delta M_{h' h,i_1}\right)^{t_{1,1}}\left(\Delta M_{h' h,i_2}\right)^{t_{2,1}}\left(\Delta a_{h' h,i_1,k' k}\right)^{t_{1,2}}\left(\Delta a_{h' h,i_2,k' k}\right)^{t_{2,2}}\\
& \qquad \times
    \dfrac{\partial^{|t_{1,1}| + |t_{1,2}|} u}{\partial M^{t_{1,1}} \partial a^{t_{1,2}}}(x;M^*_{h,i_1}, a^*_{h,i_1,k})\cdot \dfrac{\partial^{|t_{2,1}| + |t_{2,2}|} u}{\partial M^{t_{2,1}} \partial a^{t_{2,2}}}(x;M^*_{h,i_2}, a^*_{h,i_2,k})\varphi_{2n}^{h,k}(x) + \gR_{n,2}^{h' h,k' k}(x).
\end{align*}
Similarly, by applying the same arguments for decomposing $A_n^{h' h, k' k}(x)$, we can express $B_n^{h' h, k' k}(x)$ as 
\begin{align*}
    &B_n^{h' h, k' k}(x) \\
=\,& 
\sum_{|V_{h,k}| = 1}\sum_{i=1}^{N^*}\sum_{ |r| = 1} \dfrac{1}{r!} \left(\Delta M_{h' h,i}\right)^{r}
    \dfrac{\partial^{|r|} v_n^{h',k'}}{\partial M^{r}}(x;M^*_{h,i})\varphi_{3n}^{h',k'}(x) \\
    &+   \sum_{|V_{h,k}| > 1}\sum_{i=1}^{N^*} \sum_{1 \le |r|\le 2} \dfrac{1}{r!} \left(\Delta M_{h' h,i}\right)^{r}
    \dfrac{\partial^{|r|} v_n^{h',k'}}{\partial M^{r}}(x;M^*_{h,i}) \varphi_{3n}^{h',k'}(x)\\
    &+ \dfrac{1}{2}\sum_{|V_{h,k}| > 1}\sum_{i_1=1}^{N^*}\sum_{i_2=1}^{N^*}\sum_{\substack{|r_1| = 1,|r_2| =1 }}  \left(\Delta M_{h' h,i_1}\right)^{r_1}\left(\Delta M_{h' h,i_2}\right)^{r_2}\\
    &\qquad \qquad \times 
    \dfrac{\partial^{|r_1|} v_n^{h',k'}}{\partial M^{r_1}}(x;M^*_{h,i_1}) \dfrac{\partial^{|r_2|} v_n^{h',k'}}{\partial M^{r_2} }(x;M^*_{h,i_2})\varphi_{4n}^{h',k'}(x)+\gR_{n,3}^{h' h,k' k}(x) ,
\end{align*}
where $r,r_1,r_2 \in \mathbb{N}^{\overline{d}\times \overline{d}}$ and $\gR_{n,3}^{h' h, k' k}(x)$ is Taylor remainder such that $\gR_{n,3}^{h' h, k' k}(x)/\gL_{2n} \to 0$ as $n \to \infty$ and 
\begin{align*}
    \varphi_{3n}^{h',k'}(x) &:= \dfrac{1}{E^h_*(x)}\varphi'\left(\dfrac{1}{E^h_*(x)}\sum_{i=1}^{N^*}v_n^{h',k'}(x;M^*_{h,i})\right) \\
    \varphi_{4n}^{h',k'}(x) &:= \dfrac{1}{E^h_*(x)^2} \varphi''\left(\dfrac{1}{E^h_*(x)}\sum_{i=1}^{N^*}v_n^{h',k'}(x;M^*_{h,i})\right)
\end{align*}
Next, let us denote $\gJ_t := \{(t_1,t_2) \in \mathbb{N}^{\overline{d}\times \overline{d}}\times \mathbb{N}^{\overline{d}}: |t_1| + |t_2| = |t|\}$. It is worth noting that $\dfrac{\partial^{2} u}{\partial a^2} \equiv 0$. Furthermore, we also have $a^\top \cdot \dfrac{\partial u}{\partial a} = u$.
Then, putting all the above results together, the function $\gD_n^{h' h,k' k}(x)$ can be represented as 
\begin{align}
    &\gD_n^{h' h, k' k}(x) \nonumber\\
    &= \sum_{i=1}^{N^*}\sum_{|t|=0}^{1 + \mathbf{1}_{\{|V_{h,k}| > 1\}}}\overline{Y}_{t}^{h' h, k' k}(i)\dfrac{\partial^{|t|} u}{\partial M^{t}}(x;M_{h,i}^*, a^*_{h,i,k})\varphi_{1n}^{h,k}(x)\nonumber\\
    &\qquad - \sum_{i=1}^{N^*}\sum_{|r|=1}^{1 + \mathbf{1}_{\{|V_{h,k}| > 1\}}}\overline{Z}_{r}^{h' h}(i)\dfrac{\partial^{|r|} v_n^{h',k'}}{\partial M^{r}}(x;M_{h,i}^*)\varphi_{3n}^{h',k'}(x) \nonumber\\ 
    &\qquad +\sum_{\substack{|V_{h,k}| > 1, \\ i_1, i_2 \in [N^*]}}\;\sum_{0 \le |t_1|,|t_2| \le 1}\overline{P}^{h' h, k' k}_{
    t_1,t_2}(i_1,i_2)\nonumber\\
    &\qquad \qquad \times \dfrac{\partial^{|t_{1}|} u}{\partial M^{t_{1}}}(x;M^*_{h,i_1}, a^*_{h,i_1,k}) \dfrac{\partial^{|t_{2}|} u}{\partial M^{t_{2}}}(x;M^*_{h,i_2}, a^*_{h,i_2,k})\varphi_{2n}^{h,k}(x) \nonumber\\ 
    &\qquad - \sum_{\substack{|V_{h,k}| > 1, \\ i_1 \in [N^*], i_2 \in [N^*]}}\;\sum_{|r_1|=1, |r_2| = 1 }\overline{Q}_{r_1,r_2}^{h' h}(i_1,i_2)\nonumber\\
    &\qquad \qquad \times \dfrac{\partial^{|r_1|} v_n^{h',k'}}{\partial M^{r_1}}(x;M_{h,i_1}^*)\dfrac{\partial ^{|r_2|}v_n^{h',k'}}{\partial M^{r_2}}(x;M_{h,i_2}^*) 
    \varphi_{4n}^{h',k'}(x) + \gR_{n,4}^{h' h, k' k}(x), \label{appendixSDPA:second-level decompose}
\end{align}
where $\gR_{n,4}^{h' h, k' k}(x) := \gR_{n,2}^{h' h, k' k}(x) - \gR_{n,3}^{h' h, k' k}(x)$ and we denote
\begin{align*}
    \overline{Y}_{t}^{h' h, k' k}(i) &= 
    \begin{cases}
       \sum_{|t_2|=1}\left(\Delta a_{h' h,i,k' k}\right)^{t_2}, &\text{ if } |t| = 0  \\
       \left(\Delta M_{h'h,i}\right)^t + \sum_{|t_2|=1} (a_{h,i,k}^*)^{t_2} \cdot \left(\Delta a_{h' h,i,k' k}\right)^{t_2}\cdot \mathbf{1}_{\{|V_{h,k}| > 1\}}, &\text{ if } |t| = 1 \\
       \left(\Delta M_{h'h,i}\right)^t, &\text{ if } |t| = 2 \\
    \end{cases}
     \\
    \overline{Z}_{r}^{h' h}(i) &= \dfrac{1}{r!}\left(\Delta M_{h' h,i}\right)^{r} \\
    \overline{P}_{t_1,t_2}^{h' h, k' k}(i_1,i_2) &= \begin{cases}
        1/2\left[\sum_{|s|=1} (a_{h,i_1,k}^*)^{s} \cdot \left(\Delta a_{h' h,i_1,k' k}\right)^{s}\right] \cdot \left[\sum_{|s|=1} (a_{h,i_2,k}^*)^{s} \cdot \left(\Delta a_{h' h,i_2,k' k}\right)^{s}\right] \\
        \hspace{7.5cm}\text{ if } (|t_1|,|t_2|) = (0,0) \\
        1/2\left[\sum_{|s|=1} (a_{h,i_1,k}^*)^{s} \cdot \left(\Delta a_{h' h,i_1,k' k}\right)^{s}\right] \cdot \left(\Delta M_{h' h,i_2}\right)^{t_{2}}, \text{ if } (|t_1|,|t_2|) = (0,1)\\
        1/2\left(\Delta M_{h' h,i_1}\right)^{t_{1}}\cdot \left[\sum_{|s|=1} (a_{h,i_2,k}^*)^{s} \cdot \left(\Delta a_{h' h,i_2,k' k}\right)^{s}\right] , \text{ if } (|t_1|,|t_2|) = (1,0)\\
        1/2 \left(\Delta M_{h' h,i_1}\right)^{t_{1}} \left(\Delta M_{h' h,i_2}\right)^{t_{2}}, \text{ if } (|t_1|,|t_2|) = (1,1)
    \end{cases}\\
    &\left(\Delta M_{h' h,i_1}\right)^{t_{1,1}}\left(\Delta M_{h' h,i_2}\right)^{t_{2,1}}\left(\Delta a_{h' h,i_1,k' k}\right)^{t_{1,2}}\left(\Delta a_{h' h,i_2,k' k}\right)^{t_{2,2}}\\
    \overline{Q}_{r_1,r_2}^{h' h}(i_1,i_2) &= \dfrac{1}{2}\left(\Delta M_{h' h,i_1}\right)^{r_1}\left(\Delta M_{h' h,i_2}\right)^{r_2}
\end{align*}

where $e_u$ denote the $u$-th  canonical basis vector in $\mathbb{R}^{\overline{d}}$ with a $1$ in the $u$-th component and $0$ elsewhere.

\textbf{Step 2 (Non-vanishing coefficients). } 
  Therefore, the equations \eqref{appendixSDPA:first-level decompose} and \eqref{appendixSDPA:second-level decompose} show that the ratio $\gD_n(x)/ \gL_{2n}(G_n,G^*)$ can be decomposed as a linear combination of the following functions 
\begin{align*}
    &\Bigg\{\varphi\left(f_{\Gs}^{h,k}(x)\right),\dfrac{\partial^{|t|} u}{\partial M^{t}}(x;M_{h,i}^*, a^*_{h,i,k})\varphi_{1n}^{h,k}(x), \dfrac{\partial^{|r|} v_n^{h',k'}}{\partial M^{r}}(x;M_{h,i}^*)\varphi_{3n}^{h',k'}(x) \\ 
    & \dfrac{\partial^{|t_{1}|} u}{\partial M^{t_{1}}}(x;M^*_{h,i_1}, a^*_{h,i_1,k})\cdot \dfrac{\partial^{|t_{2}|} u}{\partial M^{t_{2}} }(x;M^*_{h,i_2}, a^*_{h,i_2,k})\varphi_{2n}^{h,k}(x), \\ 
    & \dfrac{\partial^{|r_1|} v_n^{h',k'}}{\partial M^{r_1}}(x;M_{h,i_1}^*)\dfrac{\partial^{|r_2|} v_n^{h',k'}}{\partial M^{r_2}}(x;M_{h,i_2}^*) 
    \varphi_{4n}^{h',k'}(x)\Bigg\},
\end{align*}
for any $h \in [H^*], k \in [K^*],(h',k') \in V_{h,k},0 \le |t| \le 2$, $0 \le |t_1|,|t_2| \le 1,1 \le  |r| \le 2,|r_1|=1,$ and $|r_2| = 1$.  
In this step, we will prove that at least one coefficient of these functions does not go to $0$ as $n \to \infty$. Assume by contrary that all these coefficients of these linear independent functions go to $0$. Taking the summation with respect to the coefficient of $\varphi\left(f_{\Gs}^{h,k}(x)\right)$ for all $1 \le h \le H^*, 1 \le k \le K^*$, we have 
\begin{equation}\label{appendix3:term1}
    \dfrac{1}{\gL_{2n}}\sum_{h=1}^{H^*}\sum_{k=1}^{K^*} |\sum_{(h',k') \in V_{h,k}}\omega_{h',k'}^n - \omega_{h,k}^*| \to 0.
\end{equation}
For index $(h,k) \in [H^*]\times [K^*]$ such that $|V_{h,k}| = 1$, taking $t =\mathbf{0}_{\overline{d}}$ and $t = e_{uv}$, in which $e_{uv}$ denotes the canonical basis matrix in $\mathbb{R}^{\overline{d}\times \overline{d}}$ with a $1$ in the $(u,v)$-th entry and $0$ elsewhere. Taking the summation with respect to the limits $\overline{Y}_{t}^{h' h, k' k}(i)/\gL_2 \to 0$ for all $i \in [N^*]$, we have 
\begin{equation}\label{appendix3:term2}
    \dfrac{1}{\gL_{2n}}\sum_{h=1}^{H^*}\sum_{k=1}^{K^*}\sum_{\substack{(h',k') \in V_{h,k}, \\ |V_{h,k}| = 1}} \omega_{h',k'}^{n}\Big(\sum_{i=1}^{N^*}\|\Delta M_{h' h,i}\|+\|\Delta a_{h' h,i,k' k}\|\Big)\to 0
\end{equation}
For index $(h,k) \in [H^*]\times [K^*]$ such that  $|V_{k | h}| > 1$, by setting $t_{1} = t_{2} = e_{uv}$ and taking the summation with respect to the limits $\overline{P}_{t_1,t_2}^{h' h, k' k}(i_1,i_1)/\gL_2 \to 0$ for all $i_1 \in [N^*]$, we have     
\begin{equation}\label{appendix3:term3}
    \dfrac{1}{\gL_{2n}}\sum_{h=1}^{H^*}\sum_{k=1}^{K^*}\sum_{\substack{(h',k') \in V_{h,k}, \\ |V_{h,k}| > 1}} \omega_{h',k'}^{n}\sum_{i=1}^{N^*}\|\Delta M_{h' h,i}\|^2 \to 0
\end{equation}
Similarly, by setting $t_1=t_2 = \mathbf{0}_{\overline{d}\times \overline{d}}$ and taking the summation with respect to the limits $\overline{P}_{t_1,t_2}^{h' h, k' k}(i_1,i_1)/\gL_2 \to 0$ for all $i_1 \in [N^*]$, we have 
\begin{equation}\label{appendix3:term4}
    \dfrac{1}{\gL_{2n}}\sum_{h=1}^{H^*}\sum_{k=1}^{K^*}\sum_{\substack{(h',k') \in V_{h,k}, \\ |V_{h,k}| > 1}} \omega_{h',k'}^{n}\sum_{i=1}^{N^*}\|\Delta a_{h' h,i,k' k}\|^2 \to 0
\end{equation}
By putting all the results in Equations \eqref{appendix3:term1}, \eqref{appendix3:term2}, \eqref{appendix3:term3}, \eqref{appendix3:term4} together, we achieve that $ 1 = \dfrac{\gL_2}{\gL_2} \to 0$, which is a contradiction. As a result, at least one of the coefficients of the linearly independent functions $\gD_n(x)/\gL_2$ does not vanish as $n \to \infty$. 

\textbf{Step 3 (Application of the Fatou's lemma). } Denote $\overline{m}_n$ as the maximum of the absolute values of the coefficients of the linear independent functions in  $\gD_n(x)/\gL_{2n}$. Given that at least one of these coefficients does not vanish, we have $1/\overline{m}_n \not \to \infty$ as $n \to \infty$. By invoking Fatou's lemma, we have that 
$$
0 = \lim_{n\to\infty}\dfrac{\|g_{G_n} - g_{G^*}\|_{L^2(\mu)}}{\overline{m}_n\gL_{2n}} \ge \int\liminf_{n \to \infty}\dfrac{|g_{G_n} - g_{G^*}|}{\overline{m}_n\gL_{2n}} d\mu(x) \ge 0.
$$
As a consequence, we achieve that 
$$
\liminf_{n \to \infty}\dfrac{|g_{G_n} - g_{G^*}|}{\overline{m}_n\gL_{2n}} = 0.
$$
When $n \to \infty$,  we denote
\begin{align*}
    \dfrac{\sum_{(h',k') \in V_{h,k}}\omega_{h',k'}^n  - \omega_{h,k}^*}{\overline{m}_n\gL_{2n}} &\to \overline{\lambda}^{h,k}_{\omega}; \\
\dfrac{\overline{Y}_{t}^{h' h, k' k}(i)}{\overline{m}_n\gL_{2n}} &\to \overline{\lambda}^{h' h, k' k}_{y,t,i}; \\
\dfrac{\overline{Z}_r^{h' h}(i)}{\overline{m}_n\gL_{2n}} &\to \overline{\lambda}^{h' h}_{z,r,i};\\
\dfrac{\overline{P}_{t_1,t_2}^{h' h, k' k}(i_1,i_2)}{\overline{m}_n\gL_{2n}} &\to \overline{\lambda}_{p,t_1,t_2,i_1,i_2}^{h' h, k' k}; \\
\dfrac{\overline{Q}_{r_1,r_2}^{h' h}(i_1,i_2)}{\overline{m}_n\gL_{2n}} &\to \overline{\lambda}_{q,r_1,r_2,i_1,i_2}^{h' h},
\end{align*}
for any $h \in [H^*], k \in [K^*],(h',k') \in V_{h,k},0 \le |t| \le 2 $, $0 \le |t_1|,|t_2| \le 1,1 \le  |r|\le 2,|r_1|=1,$ and $|r_2| = 1$.
Now, since $\sum_{(h',k') \in V_{h,k}}\omega_{h',k'}^n \to \omega_{h,k}^*$ and $f_{G_n}^{h' ,k'}(x) \to f_{G^*}^{h,k}(x)$ as $n\to \infty$ for all $(h',k') \in V_{h,k}$, we have
\begin{align*}
    v_n^{h',k'}(x;M) &\to \exp(x^\top Mx) \cdot f_{G^*}^{h,k}(x):= v^{h,k}_*(x;M); \\
    \varphi_{1n}^{h,k}(x) &\to \dfrac{1}{E^h_*(x)}\varphi'\left(f_{G^*}^{h,k}(x)\right); \\ 
    \varphi_{2n}^{h,k}(x) &\to \dfrac{1}{E_*^h(x)^2}\varphi''\left(f_{G^*}^{h,k}(x)\right); \\
    \varphi_{3n}^{h',k'}(x) &\to \dfrac{1}{E^h_*(x)}\varphi'\left(f_{G^*}^{h,k}(x)\right); \\ 
    \varphi_{4n}^{h',k'}(x) &\to \dfrac{1}{E_*^h(x)^2}\varphi''\left(f_{G^*}^{h,k}(x)\right).
\end{align*}
 Given the above notation, the limit $\displaystyle\liminf_{n \to \infty}\dfrac{|g_{G_n} - g_{G^*}|}{\overline{m}_n\gL_2}$ can be expressed as 
\begin{align}
&\sum_{h=1}^{H^*}\sum_{k=1}^{K^*}\overline{\lambda}_{\omega}^{ h, k}\varphi\left(f_{G^*}^{h,k}(x)\right) + \sum_{h=1}^{H^*}\sum_{k=1}^{K^*}\sum_{(h',k') \in V_{h,k}} \omega_{h',k'}^{n}\times \nonumber \\
&\Bigg[ \sum_{i=1}^{N^*}\sum_{|t|=0}^{1 + \mathbf{1}_{\{|V_{h,k}| > 1\}}}\overline{\lambda}_{y,t,i}^{h' h, k' k}\dfrac{\partial^{|t|} u}{\partial M^{t}}(x;M_{h,i}^*, a^*_{h,i,k})\dfrac{1}{E_*^h(x)}\varphi'\left(f^{h,k}_{G^*}(x)\right) \nonumber\\
    & \quad - \sum_{i=1}^{N^*}\sum_{|r|=1}^{1 + \mathbf{1}_{\{|V_{h,k}| > 1\}}}\overline{\lambda}_{z,r,i}^{h' h}\dfrac{\partial^{|r|} v_*^{h,k}}{\partial M^{r}}(x;M_{h,i}^*)\dfrac{1}{E_*^h(x)}\varphi'\left(f^{h,k}_{G^*}(x)\right) \nonumber \\ 
    & \quad +\sum_{\substack{|V_{h,k}| > 1, \\ i_1 \in [N^*], i_2 \in [N^*]}}\sum_{0 \le |t_1|,|t_2| \le 1}\overline{\lambda}_{p,t_1,t_2,i_1,i_2}^{h' h, k' k}\nonumber \\ 
    &\quad \qquad \times \dfrac{\partial^{|t_{1}|} u}{\partial M^{t_{1}}}(x;M^*_{h,i_1}, a^*_{h,i_1,k}) \dfrac{\partial^{|t_{2}|} u}{\partial M^{t_{2}}}(x;M^*_{h,i_2}, a^*_{h,i_2,k}) \dfrac{1}{E_*^h(x)^2}\varphi''\left(f^{h,k}_{G^*}(x)\right) \nonumber \\ 
    &\quad - \sum_{\substack{|V_{h,k}| > 1, \\ i_1 \in [N^*], i_2 \in [N^*]}}\;\sum_{|r_1|=1, |r_2| = 1 }\overline{\lambda}_{q,r_1,r_2,i_1,i_2}^{h' h}\nonumber \\ 
    &\quad \qquad \times\dfrac{\partial ^{|r_1|}v_*^{h,k}}{\partial M^{r_1}}(x;M_{h,j_1}^*)\dfrac{\partial^{|r_2|} v_*^{h,k}}{\partial M^{r_2}}(x;M_{h,j_2}^*) 
    \dfrac{1}{E_*^h(x)^2}\varphi''\left(f^{h,k}_{G^*}(x)\right)\Bigg] = 0, \label{appendix:step3SDPA}
\end{align}
for almost every $x$. Since the function $\varphi(\cdot)$ is type-2 strong identifiable, the set 
\begin{align*}
    &\Bigg\{\varphi\left(f_{\Gs}^{h,k}(x)\right),\dfrac{\partial^{|t|} u}{\partial M^{t}}(x;M_{h,i}^*, a^*_{h,i,k})\dfrac{1}{E_*^h(x)}\varphi'\left(f_{\Gs}^{h,k}(x)\right),\\
    &\dfrac{\partial^{|r|} v_*^{h,k}}{\partial M^{r}}(x;M_{h,i}^*)\dfrac{1}{E_*^h(x)}\varphi'\left(f_{\Gs}^{h,k}(x)\right), \\ 
    & \dfrac{\partial^{|t_{1}|} u}{\partial M^{t_{1}}}(x;M^*_{h,i_1}, a^*_{h,i_1,k})\cdot \dfrac{\partial^{|t_{2}|} u}{\partial M^{t_{2}}}(x;M^*_{h,i_2}, a^*_{h,i_2,k})\dfrac{1}{E_*^h(x)^2}\varphi''\left(f_{\Gs}^{h,k}(x)\right), \\ 
    & \dfrac{\partial^{|r_1|} v_*^{h,k}}{\partial M^{r_1}}(x;M_{h,i_1}^*)\dfrac{\partial^{|r_2|} v_*^{h,k}}{\partial M^{r_2}}(x;M_{h,i_2}^*) \dfrac{1}{E_*^h(x)^2}
    \varphi''\left(f_{\Gs}^{h,k}(x)\right)\Bigg\},
\end{align*}
is linearly independent for any $h \in [H^*], k \in [K^*], 0\le |t| \le 2$, $0 \le |t_1|,|t_2| \le 2, |r|\in \{1,2\},|r_1|=1,$ and $|r_2| = 1$. 
Therefore, equation \eqref{appendix:step3SDPA} indicates that all the coefficients $$\Big\{\overline{\lambda}^{h,k}_{\omega}, \overline{\lambda}^{h' h, k' k}_{y,t,i}, \overline{\lambda}^{h' h}_{z,r,i}, \overline{\lambda}_{p,t_1,t_2,i_1,i_2}^{hk'}, \overline{\lambda}_{q,r_1,r_2,i_1,i_2}^{h' h}\Big\}$$ are $0$'s, which is a contradiction. Hence, we obtain that 
$$
\lim_{\varepsilon \to 0} \inf_{G \in \gG_{H^*,K,N^*}(\Theta): \gL_2(G,G^*) \le \varepsilon} \|g_G - g_{G^*}\|_{L^2(\mu)}/ \gL_2(G,G^*) > 0. 
$$
\textbf{Proof of the global part (Equation~\eqref{theorem3:global}).}

Following the same approach as in the proof of the global part in Theorem~\ref{theorem:nonlinear_at_V} at Appendix~\ref{appendix:nonlinear_at_V}, we will demonstrate the identifiability of mixing measure in $\gG_{H^*, K, N^*}$. In particular, we prove that the equality $g_G(x) = g_{G^*}(x)$ for almost every $x$ implies that $\gL_2(G,G^*) = 0$.  Now, we express the equation $g_{G}(x) = g_{G^*}(x)$ as 
$$
\sum_{h'=1}^{H^*}\sum_{k'=1}^{K}\omega_{h',k'}\varphi\left(f_G^{h',k'}(x)\right) = \sum_{h=1}^{H^*}\sum_{k=1}^{K^*}\omega_{h,k}^*\varphi\left(f_{G^*}^{h,k}(x)\right)
$$
Since the set of functions $x \to \varphi\left(q(x)\right) $ is linearly independent with different functions  $q(x)$, we deduce that for each pair $(h,k) \in [H^*] \times [K^*]$, there exists a set $V_{h,k} = \left\{(h',k') \in [H^*] \times [K]: \varphi(f_{G}^{h',k'}(x)) = \varphi\left(f_{G^*}^{h,k}(x)\right)\right\}$. Since the function $\varphi(\cdot)$ is injective, we deduce that for each pair $(h,k) \in [H^*] \times [K^*]$, we have 
\begin{equation}\label{theorem3:identify1}
    f_{G}^{h',k'}(x) = f_{\Gs}^{h,k}(x), \text{for almost every $x$}, \text{and} \sum_{(h',k') \in [H^*]\times [K^*]} \omega_{h',k'} = \omega_{h,k}, \forall \, (h',k') \in V_{h,k}.
\end{equation}
We now prove that if $f_G^{h',k'}(x) = f_{G^*}^{h,k}(x)$ for almost every $x$, then we deduce \begin{equation}\label{theorem3:identify2}
    \{M_{h',i}, a_{h',i,k'}\}_{i=1}^{N^*} \equiv \{M_{h,i}, a_{h,i,k}\}_{i=1}^{N^*}. 
\end{equation}
For notational simplicity, we will prove that if 
\begin{equation}\label{eq:theorem3_f_G}
\sum_{i=1}^{N}\dfrac{\exp(x^\top M_{i}'x)}{\sum_{j=1}^N\exp(x^\top M_{j}'x)}\cdot (a_{i}')^\top x = \sum_{i=1}^{N}\dfrac{\exp(x^\top M_{i}x)}{\sum_{j=1}^{N}\exp(x^\top M_{j}x )}\cdot (a_{i})^\top x, \text{for almost every $x$},
\end{equation}
where every element in each set $\{M_i'\}_{i=1}^N, \{M_i\}_{i=1}^N, \{a'_{i}\}_{i=1}^N$, and $ \{a_i\}_{i=1}^N$ is pairwise distinct, then $\{(M_{i},a_i)\}_{i=1}^N \equiv \{(M_i',a_i')\}_{i=1}^N$.  The  equation \eqref{eq:theorem3_f_G} is equivalent to 
\begin{equation}\label{eq:theorem3_induction}
   \sum_{1 \le i, j \le N}\exp(x^\top (M'_{i}+M_{j})x)\cdot a_{i}' \\
=\sum_{1 \le i,j \le N}\exp(x^\top (M'_j+M_{i})x)\cdot a_{i}, 
\end{equation}
for almost every $x$. Since every element in each set $\{M_i'\}_{i=1}^N$ and $\{M_i\}_{i=1}^N$ is pairwise distinct, 
without loss of generality, there exists an open set $\gU \in \mathbb{R}^{\overline{d}} $ such that $\exp(u^\top M_1u) > \exp(u^\top M_2 u) > \ldots > \exp(u^\top M_N u)$ and $\exp(u^\top M_1'u) > \exp(u^\top M_2' u) > \ldots > \exp(u^\top M_N' u)$ for any $u \in \gU$. Therefore, if there exist two matrices $A \in \{M_i\}_{i=1}^N$ and $B \in \{M_i'\}_{i=1}^N$ such that $\exp(u^\top A u) = \exp(u^\top B u),\forall \, u \in \gU$, we can deduce $A = B$. It is also worth noting that the set of functions $x \to e^{\eta x^2}$ is linearly independent with different parameters $\eta$.  Now, we choose $x = pu$, where $p \in \mathbb{R}$, and we denote $m_i = u^\top M_iu, m_i' = u^\top M_i'u$ for all $i \in [N]$, the equation \eqref{eq:theorem3_induction} can be rewritten as 
\begin{equation}\label{eq:theorem3_induction2}
    \sum_{1 \le i,j\le N} e^{(m_i' + m_j)p^2} a_i' = \sum_{1 \le i,j\le N} e^{(m_i' + m_j)p^2} a_j,  
\end{equation}
for almost every $p$. Now, we will prove by induction that $a_t = a_t'$ and $m_t = m_t'+C$ for all $t = 1,2,\ldots, N$, where $C$ is a fixed scalar.  

For $t=1$, since $m_1 + m_1' > m_i + m_j', \forall \, (i,j') \ne (1,1)$, the coefficient of $e^{(m_1+m_1')p^2}$ is equal to $0$, which leads to $a_1 = a_1'$. 

For $t =2$, we first remove the term $e^{(m_1+m_1')p^2}a_1$ out of equation \eqref{eq:theorem3_induction2}.  Then, the maximal value of the remaining $m_i + m_{j}'$ is achieved by either $m_1' + m_2$ or $m_2' + m_1$. If $m_1' + m_2 > m_2' + m_1$, the coefficient of $e^{(m_1' + m_2)p^2}$ deduces that $a_1' = a_2$, which is a contradiction, since $a_2 \ne a_1$. Then, $m_1' + m_2 \le m_2' + m_1$. Similarly, we can also deduce that $m_2' + m_1 \le m_1' + m_2$. Consequently, we have $m_1' + m_2 = m_2' + m_1$. Consider the coefficient of $e^{(m_1'+m_2)p^2}$, we have $a_1' + a_2' = a_1 + a_2$, which leads to $a_2' = a_2$. 

Assume the claim holds for $t \ge 1$. We prove that it also holds for $t+1$. We remove the terms $e^{(m_i+m_j')p^2}a_i$ for any $1 \le i,j \le t$ out of equation \eqref{eq:theorem3_induction2}. Then, the maximal value of remaining $m_i + m_j'$ is achieved by either $m_1 + m_t'$ or $m_1' + m_t$. By following the same line of reasoning as in the case $t=2$, we deduce that $a_{t+1} = a'_{t+1}$ and $m_{t+1}-m_{t+1}' = m_1 - m_1'$. 

Since $M_{N} = M_{N}' = \mathbf{0}_{\overline{d}\times \overline{d}}$, we deduce that $m_t = m_t'$ for any $t \in [N]$, which leads to $M_t = M_{t'}$ for any $t \in [N]$. The claim in \eqref{theorem3:identify2} is completely proved. From \eqref{theorem3:identify1} and \eqref{theorem3:identify2}, we deduce that $\gL_2(G,G^*) = 0$, which completes our proof.  
\section{Proofs of Auxiliary Results}
\label{appendix:auxiliary_results}

\subsection{Proof of Proposition~\ref{prop:regression_rate_mha}}
\label{appendix:regression_rate_mha}
We begin the proof by defining notation. To begin with, we define
$\mathcal{F}_{H^*, K, N^*}(\Theta)$ as the set of regression functions of all mixing measures in
$\mathcal{G}_{H^\ast, K, N^\ast}(\Theta)$, or
\[
\mathcal{F}_{H^*, K, N^*}(\Theta) := \{f_G(x): G \in \mathcal{G}_{H^\ast, K, N^\ast}(\Theta)\}.
\]
Given $\delta > 0$, the local $L^2(\mu)$ ball centered around the regression function
$f_{G^\ast}(x)$ and intersected with the set $\mathcal{F}_{H^*, K, N^*}(\Theta)$ is
\[
\mathcal{F}_{H^*, K, N^*}(\Theta, \delta)
:= \left\{ f \in \mathcal{F}_{H^*, K, N^*}(\Theta): \|f - f_{G^\ast}\|_{L^2(\mu)} \le \delta \right\}.
\]


To quantify the complexity of this set, we employ the bracketing integral introduced by Geer et al.~\cite{vandeGeer-00}:
\begin{align}
J_B(\delta, \mathcal{F}_{H^*, K, N^*}(\Theta, \delta))
:= \int_{\delta^2/2}^{\delta}
H_B^{1/2}\big(t, \mathcal{F}_{H^*, K, N^*}(\Theta, t), \|\cdot\|_{L^2(\mu)}\big) \, dt \vee \delta,
\label{eq:JB-def}
\end{align}
where $H_B(t, \mathcal{F}_{H^*, K, N^*}(\Theta, t), \|\cdot\|_{L^2(\mu)})$ denotes the bracketing entropy~\cite{vandeGeer-00}
of $\mathcal{F}_{H^*, K, N^*}(\Theta, t)$ under the $L^2$-norm, and $t \vee \delta := \max\{t, \delta\}$.
Adapting the arguments of Theorem 7.4 and Theorem 9.2 in~\cite{vandeGeer-00} to our setting yields the following lemma:

\begin{lemma} \label{lem:geer-like}
Let $\Psi(\delta)$ be such that $\Psi(\delta) \geq J_B(\delta, \mathcal{F}_{H^*, K, N^*}(\Theta, \delta))$ and $\Psi(\delta)/\delta^2$ is non-increasing in $\delta$. Then there exists a universal constant $c$ and a sequence $(\delta_n)$ satisfying $\sqrt{n}\delta_n^2 \geq c\Psi(\delta_n)$ for which
\[
\mathbb{P}\Big( \|f_{\widehat{G}_n} - f_{G^\ast}\|_{L^2(\mu)} > \delta \Big)
\leq c \exp\Big( -\frac{n\delta^2}{c^2} \Big),
\]
for all $\delta \geq \delta_n$.
\end{lemma}

We now demonstrate that when the expert functions are Lipschitz continuous, the following bound holds:
\begin{align}
H_B\big(\varepsilon, \mathcal{F}_{H^*, K, N^*}(\Theta), \|\cdot\|_{L^2(\mu)}\big) \lesssim \log(1/\varepsilon),
\label{eq:bracketing-bound}
\end{align}
for any $0 < \varepsilon \leq 1/2$. Indeed, for any function $f_G \in \mathcal{F}_{H^*, K, N^*}(\Theta)$,
since the expert functions are bounded, we obtain that $f_G(x) \leq M$ for almost everywhere $x$,
where $M > 0$ is some bounded constant of the expert functions.
Choose $\tau \leq \varepsilon$ and take ${\xi_1, \dots, \xi_k}$ as a $\tau$-cover of $\mathcal{F}_{H^*, K, N^*}(\Theta)$ under the $L^\infty$ norm
of the set $\mathcal{F}_{H^*, K, N^*}(\Theta)$ where $k := N(\tau, \mathcal{F}_{H^*, K, N^*}(\Theta), \|\cdot\|_{L^\infty})$
is the $\tau$-covering number of the metric space $(\mathcal{F}_{H^*, K, N^*}(\Theta), \|\cdot\|_{L^\infty})$.
For each $i \in [k]$, define brackets $[L_i, U_i]$ by
\begin{align*}
L_i(x) &:= \max\{\xi_i(x) - \tau, 0\}, \\
U_i(x) &:= \min\{\xi_i(x) + \tau, M\}.
\end{align*}
This construction ensures that
$\mathcal{F}_{H^*, K, N^*}(\Theta) \subset \bigcup_{i=1}^k [L_i(x), U_i(x)]$ and $U_i(x) - L_i(x) \leq \min\{2\tau, M\}$.
Consequently,
\[
\|U_i - L_i\|_{L^2(\mu)}^2 = \int (U_i - L_i)^2 \, d\mu(x) \leq \int 4\tau^2 \, d\mu(x) = 4\tau^2,
\]
which implies that $\|U_i - L_i\|_{L^2(\mu)} \leq 2\tau$. From the definition of bracketing entropy, we therefore obtain
\begin{align}
H_B\big(2\tau, \mathcal{F}_{H^*, K, N^*}(\Theta), \|\cdot\|_{L^2(\mu)}\big)
\leq \log k
= \log N(\tau, \mathcal{F}_{H^*, K, N^*}(\Theta), \|\cdot\|_{L^\infty}).
\label{eq:bracket-cover-rel}
\end{align}

Consequently, we have to construct an upper bound on the covering number
$N(\tau, \mathcal{F}_{H^*, K, N^*}(\Theta), \|\cdot\|_{L^\infty})$.
To this end, we decompose the parameter space $\Theta$ into the following
marginal parameter sets:
\[
\Delta := \{M\in\mathbb{R}^{\overline{d}\times \overline{d}}:\ (M,a,\omega)\in\Theta\},
\]
\[
\Psi := \{a\in\mathbb{R}^{\overline{d}}:\ (M,a,\omega)\in\Theta\},
\qquad
\Omega := \{\omega\in\mathbb{R}:\ (M,a,\omega)\in\Theta\}.
\]
The compactness of $\Theta$ implies $\Delta$, $\Psi$, and $\Omega$ are compact. Therefore, for any $\tau>0$, there exist finite $\tau$-covers
$\Delta_\tau$, $\Psi_\tau$, and $\Omega_\tau$ of $\Delta$, $\Psi$, and $\Omega$,
respectively.
We have
\[
|\Delta_\tau|
\le O\!\left(\tau^{-\overline{d}^2H^\ast N^\ast}\right),\qquad
|\Psi_\tau|
\le O\!\left(\tau^{-\overline{d}H^\ast K N^\ast}\right),\qquad
|\Omega_\tau|
\le O\!\left(\tau^{-H^\ast K}\right).
\]

For each mixing measure
\[
G = \sum_{h=1}^{H^*}\sum_{k=1}^{K}\omega_{h,k}\sum_{i=1}^{N^*}\delta_{(M_{h,i},a_{h,i,k})},
\]
we define the softmax weights by
\[
\sigma_{h,i}(x;M_h)
:=
\frac{\exp(x^\top M_{h,i}x)}{\sum_{j=1}^N\exp(x^\top M_{h,j}x)},
\qquad h\in[H^\ast],\ i\in[N].
\]
We then consider the following three mixing measures:
\begin{align*}
\check{G} &:= \sum_{h=1}^{H^\ast}\sum_{k=1}^{K}\omega_{h,k}
\sum_{i=1}^{N^\ast}\delta_{(M_{h,i},\bar{a}_{h,i,k})}, \\
\tilde{G} &:= \sum_{h=1}^{H^\ast}\sum_{k=1}^{K}\bar{\omega}_{h,k}
\sum_{i=1}^{N^\ast}\delta_{(M_{h,i},\bar{a}_{h,i,k})}, \\
\bar{G} &:= \sum_{h=1}^{H^\ast}\sum_{k=1}^{K}\bar{\omega}_{h,k}
\sum_{i=1}^{N^\ast}\delta_{(\bar{M}_{h,i},\bar{a}_{h,i,k})}.
\end{align*}
Here $\bar{a}_{h,i,k} \in \Psi\tau$ is the nearest point to $a_{h,i,k}$, $\bar{\omega}{h,k} \in \Omega\tau$ is the nearest to $\omega_{h,k}$, and $(\bar{M}_{h,i}, \bar{c}_{h,i}) \in \Delta_\tau$ is the nearest to $(M_{h,i}, c_{h,i})$.

From these definitions, we obtain
\begin{align*}
\|f_G - f_{\check{G}}\|_{L^\infty}
&= \sup_{x \in \mathcal{X}} \Bigg|
\sum_{h=1}^{H^\ast}\sum_{k=1}^{K}\omega_{h,k}
\sum_{i=1}^{N^\ast} \sigma_{h,i}(x;M_h)\cdot
\Big[\,(a_{h,i,k})^\top x - (\bar a_{h,i,k})^\top x\,\Big]
\Bigg| \\
&\leq \sum_{h=1}^{H^\ast}\sum_{k=1}^{K}\sum_{i=1}^{N^\ast}
\sup_{x\in\mathcal X}
\Big|\omega_{h,k}\Big|\,
\sigma_{h,i}(x;M_h)\,
\Big|\big(a_{h,i,k}-\bar a_{h,i,k}\big)^\top x\Big| \\
&\leq \sum_{h=1}^{H^\ast}\sum_{k=1}^{K}\sum_{i=1}^{N^\ast}
\sup_{x\in\mathcal X}
\Big|\omega_{h,k}\Big|\,
\Big\|a_{h,i,k}-\bar a_{h,i,k}\Big\|_2\,\|x\|_2 \\
&\lesssim \sum_{h=1}^{H^\ast}\sum_{k=1}^{K}\sum_{i=1}^{N^\ast}
\Big\|a_{h,i,k}-\bar a_{h,i,k}\Big\|_2
\lesssim \tau.
\end{align*}
The first inequality uses the triangle inequality, the second uses that the softmax weight is at most $1$, and the final bound follows because $|x| \le B$ for some $B>0$ and $\omega$ is bounded over $\Theta$.

Next, we have
\begin{align*}
\|f_{\check{G}} - f_{\tilde{G}}\|_{L^\infty}
&= \sup_{x \in \mathcal{X}} \Bigg|
\sum_{h=1}^{H^\ast}\sum_{k=1}^{K}\sum_{i=1}^{N^\ast}
\Big(\omega_{h,k}-\bar\omega_{h,k}\Big)\,
\sigma_{h,i}(x;M_h)\,
(\bar a_{h,i,k})^\top x
\Bigg| \\
&\le
\sum_{h=1}^{H^\ast}\sum_{k=1}^{K}\sum_{i=1}^{N^\ast}
\sup_{x\in\mathcal X}
\Big|\omega_{h,k}-\bar\omega_{h,k}\Big|\,
\Big|(\bar a_{h,i,k})^\top x\Big| \\
&\lesssim
\sum_{h=1}^{H^\ast}\sum_{k=1}^{K}
\Big|\omega_{h,k}-\bar\omega_{h,k}\Big|
\lesssim \tau.
\end{align*}
Again, the triangle inequality gives the first bound, and the second uses that each expert $(\bar a_{h,i,k})^\top x$ is uniformly bounded.

Finally, we have
\begin{align*}
\|f_{\tilde{G}} - f_{\bar{G}}\|_{L^\infty}
&= \sup_{x \in \mathcal{X}} \Bigg|
\sum_{h=1}^{H^\ast}\sum_{k=1}^{K}\sum_{i=1}^{N^\ast}
\bar\omega_{h,k}\,
\Big[
\sigma_{h,i}(x;M_h)
-
\sigma_{h,i}(x;\bar M_h)
\Big]\,
(\bar a_{h,i,k})^\top x
\Bigg| \\
&\leq
\sum_{h=1}^{H^\ast}\sum_{k=1}^{K}\sum_{i=1}^{N^\ast}
\sup_{x\in\mathcal X}
|\bar\omega_{h,k}|\,
\Big|\sigma_{h,i}(x;M_h)-\sigma_{h,i}(x;\bar M_h)\Big|\,
\Big|(\bar a_{h,i,k})^\top x\Big|.
\end{align*}
For this term, we use the Lipschitz continuity of the softmax function.
In particular, for each fixed $h$ and $i$, the map
\[
(M_h)\mapsto \sigma_{h,i}(x;M_h)
=
\frac{\exp(x^\top M_{h,i}x)}{\sum_{j=1}^N\exp(x^\top M_{h,j}x)}
\]
is Lipschitz uniformly over $x\in\mathcal X$ since $\mathcal X$ is bounded and $(M)$ ranges over a compact set. Therefore,
\begin{align*}
\|f_{\tilde{G}} - f_{\bar{G}}\|_{L^\infty}
&\lesssim
\sum_{h=1}^{H^\ast}\sum_{i=1}^{N^\ast}
\Big(
\|M_{h,i}-\bar M_{h,i}\|_F\cdot \|x\|^2
\Big)
\lesssim \tau .
\end{align*}
Above, the last inequality occurs as the input space is bounded, that is, $\|x\|\le B$, and
$(\bar M_{h,i},\bar c_{h,i})$ is chosen from a $\tau$-cover.

According to the triangle inequality, we have
\[
\|f_G - f_{\bar{G}}\|_{L^\infty}
\le
\|f_G - f_{\check{G}}\|_{L^\infty}
+
\|f_{\check{G}} - f_{\tilde{G}}\|_{L^\infty}
+
\|f_{\tilde{G}} - f_{\bar{G}}\|_{L^\infty}
\lesssim \tau.
\]

By the definition of the covering number, we deduce that
\begin{align}
N(\tau, \mathcal{F}_{H^*, K, N^*}(\Theta), \|\cdot\|_{L^\infty})
&\leq |\Delta_\tau| \times |\Psi_\tau| \times |\Omega_\tau| \nonumber\\
&\leq
O\!\left(\tau^{-\overline{d}^2H^\ast N^\ast}\right)
\times
O\!\left(\tau^{-\overline{d} H^\ast K N^\ast}\right)
\times
O\!\left(\tau^{- H^\ast K}\right)
\nonumber\\
&\leq
O\!\left(
\tau^{-\big[\overline{d}^2H^\ast N + \overline{d} H^\ast K N + H^\ast K\big]}
\right).
\label{eq:cover-size}
\end{align}

Substituting the covering number bound~\eqref{eq:cover-size} into
\eqref{eq:bracket-cover-rel}, we obtain
\begin{align*}
H_B\big(2\tau, \mathcal{F}_{H^*, K, N^*}(\Theta), \|\cdot\|_{L^2(\mu)}\big)
&\le
\log N(\tau, \mathcal{F}_{H^*, K, N^*}(\Theta), \|\cdot\|_{L^\infty}) \\
&\le
\log\!\Big(
C\,\tau^{-\big[d^2H^\ast N^\ast + d H^\ast K N^\ast + H^\ast K\big]}
\Big) \\
&\le
C_1 \log(1/\tau),
\end{align*}
where $C_1>0$ is a constant depending on $(H^\ast,K,N^\ast,\overline{d})$.
Replacing $2\tau$ by $\varepsilon$ yields
\[
H_B\big(\varepsilon, \mathcal{F}_{H^*, K, N^*}(\Theta), \|\cdot\|_{L^2(\mu)}\big)
\lesssim \log(1/\varepsilon).
\]
As a result, it follows that
\begin{align}
J_B(\delta, \mathcal{F}_{H^*, K, N^*}(\Theta, \delta))
&= \int_{\delta^2/2}^{\delta}
H_B^{1/2}\big(t, \mathcal{F}_{H^*, K, N^*}(\Theta, t), \|\cdot\|_{L^2(\mu)}\big) \, dt \vee \delta\\
&\lesssim \int_{\delta^2/2}^{\delta} \sqrt{\log(1/t)} \, dt \vee \delta.
\label{eq:JB-bound}
\end{align}

Define $\Psi(\delta) = \delta \sqrt{\log(1/\delta)}$. Then $\Psi(\delta)/\delta^2$ is non-increasing in $\delta$.
Equation \eqref{eq:JB-bound} shows that $\Psi(\delta) \geq J_B(\delta, \mathcal{F}_{H^*, K, N^*}(\Theta, \delta))$.
Take $\delta_n = \sqrt{\log(n)/n}$. Then for some universal constant $c$, we have $\sqrt{n}\delta_n^2 \geq c\Psi(\delta_n)$.
Finally, applying Lemma \ref{lem:geer-like} yields the statement of the theorem.

\subsection{Sigmoid, Tanh, and Softplus Are Type-1/Type-2 Strongly Identifiable Functions}

\textbf{Part I. Sigmoid, Tanh, and Softplus Are Type-1 Strongly Identifiable}

To prove that \(x\mapsto\varphi(x)\in\{\operatorname{sigmoid}(x+b),\tanh(x+b),\operatorname{softplus}(x+b)\}\) with \(b\ne 0\) satisfies the Type-1 strong identifiability condition, we note that these functions are injective, twice differentiable, and Lipschitz continuous. We then show that all the coefficients of the following equation equal to 0:

\[
\sum_{i=1}^{N}\sum_{k=0}^{2}
\sum_{\alpha\in\mathcal I_k}
\gamma_{\alpha,k,i}\,
x^\alpha
\exp(x^\top M_i x)
\varphi^{(k)}(a_i^\top x)=0, \text{for almost every \(x\).}
\]
Since the left-hand side is real analytic in \(x\), its vanishing almost everywhere implies that it vanishes identically. Collecting terms according to the derivative order of \(\varphi\), the identity can be rewritten as

\[
\sum_{i=1}^{N}
\exp(x^\top M_i x)
\left[
A_i(x)\varphi(a_i^\top x)
+B_i(x)\varphi'(a_i^\top x)
+C_i(x)\varphi''(a_i^\top x)
\right]=0,
\]

where \(A_i(x)=Q_{4,i}+Q_{2,a,i}+Q_{0,i}\), \(B_i(x)=Q_{3,i}+Q_{1,i}\), and \(C_i(x)=Q_{2,b,i}\), where \(Q_{4,i},Q_{3,i},Q_{2,a,i},Q_{2,b,i},Q_{1,i},Q_{0,i}\) denote the homogeneous polynomial contributions corresponding to the index pairs \((|\alpha|,k)=(4,0),(3,1),(2,0),(2,2),(1,1),(0,0)\), respectively.

For each \(i\), the functions \(\varphi(a_i^\top x)\), \(\varphi'(a_i^\top x)\), and \(\varphi''(a_i^\top x)\) become singular on the complex hyperplanes
\[
a_i^\top x+b=z_m,\qquad m\in\mathbb Z,
\]

where \(z_m=(m+1/2)\pi\mathrm i\) for tanh and \(z_m=(2m+1)\pi\mathrm i\) for sigmoid and softplus.

For sigmoid and tanh, these functions have poles of orders \(1,2,3\), respectively. For softplus, they have a logarithmic singularity, a simple pole, and a double pole, respectively. Hence, their singularities have increasing strengths.

Since \(b>0\) and the vectors \(a_i\) are pairwise distinct and nonzero, the singular hyperplanes associated with different indices \(i\) do not coincide. At a generic point of a hyperplane associated with \(a_i\), all terms indexed by \(j\neq i\) remain holomorphic. Hence, the coefficient \(C_i(x)\) of the strongest singular term must vanish on that hyperplane. Since this holds for infinitely many parallel hyperplanes, we deduce that \(C_i(x)\equiv0\). Applying the same argument to the remaining terms successively gives \(B_i(x)\equiv0\) and \(A_i(x)\equiv0\).

Finally, since the homogeneous polynomials appearing in \(A_i(x),B_i(x),\) and \(C_i(x)\) have different degrees, we deduce that

\[
Q_{4,i}=Q_{3,i}=Q_{2,a,i}
=Q_{2,b,i}=Q_{1,i}=Q_{0,i}=0,
\quad \forall\,1\le i\le N.
\]

By the definitions of these polynomial components and the linear independence of distinct monomials \(x^\alpha\), it follows that
\[
\gamma_{\alpha,k,i}=0,
\quad
\forall\,i,\;
k\in\{0,1,2\},\;
\alpha\in\mathcal I_k.
\]
Therefore, the collection of functions in Definition 1 is linearly independent, completing the proof.

\textbf{Part II. Type-2 Strong Identifiability of Sigmoid, Tanh, and Softplus}

To prove that \(z\mapsto\varphi(z)\in\{\operatorname{sigmoid}(z+b),\tanh(z+b),\operatorname{softplus}(z+b)\}\), with \(b\neq0\), satisfies the Type-2 strong identifiability condition, we show that all the coefficients of the following equation equal to 0:
\[
\begin{aligned}
\sum_{h=1}^{H^*}\sum_{k=1}^{K^*}\Bigg[
A_{h,k}\varphi(f_{G^*}^{h,k}(x))
+\sum_{i=1}^{N^*}\sum_{\alpha\in\mathcal J}
\gamma_{h,k,i,\alpha}
x^\alpha u_{h,k,i}(x)\varphi'(f_{G^*}^{h,k}(x))\\
+\sum_{1 \le i \le j \le N^*}\sum_{\alpha\in\mathcal J}
\eta_{h,k,i,j,\alpha}
x^\alpha u_{h,k,i}(x)u_{h,k,j}(x)
\varphi''(f_{G^*}^{h,k}(x))
\Bigg]=0, \text{ for almost every \(x\),}
\end{aligned}
\]
where we denote $u_{h,k,i}(x) := \exp(x^\top M_{h,i}^*x) \cdot ((a_{h,i,k}^*)^\top x)$. Since the left-hand side is real analytic in \(x\), its vanishing almost everywhere implies that it vanishes identically. Collecting terms according to the derivative order of \(\varphi\), the identity can be rewritten as
\[
\sum_{h=1}^{H^*}\sum_{k=1}^{K^*}
\Big[
A_{h,k}\varphi(f_{G^*}^{h,k}(x))
+B_{h,k}(x)\varphi'(f_{G^*}^{h,k}(x))
+C_{h,k}(x)\varphi''(f_{G^*}^{h,k}(x))
\Big]=0,
\]
where
\begin{align*}
  B_{h,k}(x)&=
\sum_{i=1}^{N^*}\sum_{\alpha\in\mathcal J}
\gamma_{h,k,i,\alpha}x^\alpha u_{h,k,i}(x) \\
C_{h,k}(x)&=
\sum_{1 \le i \le j \le N^*}\sum_{\alpha\in\mathcal J}
\eta_{h,k,i,j,\alpha}x^\alpha
u_{h,k,i}(x)u_{h,k,j}(x).
\end{align*}
For sigmoid and tanh, the functions \(\varphi,\varphi'\), and \(\varphi''\) have poles of orders \(1,2,3\), respectively. For softplus, they have a logarithmic singularity, a simple pole, and a double pole, respectively.

For each \((h,k)\), the singularities of these functions occur when
\[
f_{G^*}^{h,k}(x)+b=z_m,
\]
where \(z_m=(m+1/2)\pi\mathrm i\) for tanh and \(z_m=(2m+1)\pi\mathrm i\) for sigmoid and softplus.
We denote 
$
f_{G^*}^{h,k}(x)
=
\frac{P_{h,k}(x)}{Z_h(x)},
$
where
\[
P_{h,k}(x):=\sum_{i=1}^{N^*}u_{h,k,i}(x),
\qquad
Z_h(x):=\sum_{i=1}^{N^*}
\exp(x^\top M_{h,i}^*x).
\]
For each \((h,k)\), consider a smooth genuine pole hypersurface
\[
\mathcal H_{h,k}:=
\{x:Z_h(x)=0,\ P_{h,k}(x)\neq0,\ \nabla Z_h(x)\neq0\}
\]
of \(f_{G^*}^{h,k}\). Near \(\mathcal H_{h,k}\), the equations \(f_{G^*}^{h,k}(x)+b=z_m\) define singular hypersurfaces that accumulate at \(\mathcal H_{h,k}\), where every smooth point \(x_0\in\mathcal H_{h,k}\) is the limit of a sequence \((x_m)\) satisfying \(f_{G^*}^{h,k}(x_m)+b=z_m\) as \(|m|\to\infty\). At generic points of these hypersurfaces that are not shared by other components, the strongest singularity therefore forces \(C_{h,k}\) to vanish.

Since these singular hypersurfaces accumulate at \(\mathcal H_{h,k}\) and \(C_{h,k}(x)\) is holomorphic in its neighborhood, the identity theorem applied along the transverse complex direction implies \(C_{h,k}(x)\equiv0\). Repeating this argument for the remaining terms yields
$
B_{h,k}(x)\equiv0$ and $A_{h,k}=0.$ Consequently, we obtain that all the coefficients 
$\gamma_{h,k,i,\alpha}$ and $
\eta_{h,k,i,j,\alpha}$ are equal to zero, completing the proof. 

\section{Experiment Details and Results}
\label{appendix:experiment}

This appendix is organized into two parts. Section~\ref{appendix:details} provides the ground-truth parameters used in the three experiments described in Section~\ref{sec:experiments}. Section~\ref{appendix:results} reports the detailed results of Experiments~2 and~3, which were deferred from the main text.

\subsection{Experiment Details}
\label{appendix:details}

\textbf{Ground-truth parameters for Experiments 1 and 2}

For Experiments~1 and~2 ($d = 2$, $H^* = N^* = K^* = 2$), the gating matrices $M^*$ are specified as follows:
\begin{align*}
M_{0,0}^* &= \begin{bmatrix} 2.0 & 0.0 \\ 0.0 & 0.5 \end{bmatrix}, \quad
M_{0,1}^* = \begin{bmatrix} 0 & 0 \\ 0 & 0 \end{bmatrix}, \\
M_{1,0}^* &= \begin{bmatrix} 1.5 & 0.2 \\ 0.2 & 0.3 \end{bmatrix}, \quad
M_{1,1}^* = \begin{bmatrix} 0 & 0 \\ 0 & 0 \end{bmatrix}.
\end{align*}
The corresponding gating weights $\omega_{h,k}^*$ and expert coefficients $a_{h,i,k}^* \in \mathbb{R}^2$ are listed in Table~\ref{tab:true_params}.

\begin{table}[!h]
\caption{Ground-truth parameters used in Experiments~1 and~2 for gating weights $\omega_{h,k}^*$ and expert coefficients $a_{h,i,k}^* \in \mathbb{R}^2$ with $H^* = 2$ heads, $N^* = 2$ experts per head, and $K^* = 2$ channels.}
\label{tab:true_params}
\centering
{
\begin{tabular}{ccccc}
\hline
Head $h$ & Expert $i$ & Channel $k$ & $\omega_{h,k}^*$ & $a_{h,i,k}^*$ \\
\hline
0 & 0 & 0 & 1.0 & $(1.0, -0.5)$ \\
0 & 0 & 1 & 0.5 & $(0.5, 0.8)$ \\
0 & 1 & 0 & 1.0 & $(-1.0, 0.8)$ \\
0 & 1 & 1 & 0.5 & $(0.2, -0.3)$ \\
\hline
1 & 0 & 0 & 0.8 & $(0.6, 0.4)$ \\
1 & 0 & 1 & 0.3 & $(-0.2, 0.5)$ \\
1 & 1 & 0 & 0.8 & $(-0.7, -0.2)$ \\
1 & 1 & 1 & 0.3 & $(0.3, -0.4)$ \\
\hline
\end{tabular}
}
\end{table}

\textbf{Ground-truth parameters for Experiment 3}

For Experiment~3 ($d = 3$, $H^* = N^* = K^* = 3$), the gating matrices $M^* \in \mathbb{R}^{3 \times 3}$ are specified as follows:
\begin{align*}
M_{0,0}^* &= \begin{bmatrix} 2.0 & 0.0 & 0.3 \\ 0.0 & 0.5 & 0.1 \\ 0.3 & 0.1 & 1.0 \end{bmatrix}, \quad
M_{0,1}^* = \begin{bmatrix} 0.5 & 0.0 & 0.2 \\ 0.0 & 2.0 & 0.1 \\ 0.2 & 0.1 & 0.8 \end{bmatrix}, \quad
M_{0,2}^* = \begin{bmatrix} 0 & 0 & 0 \\ 0 & 0 & 0 \\ 0 & 0 & 0 \end{bmatrix}, \\
M_{1,0}^* &= \begin{bmatrix} 1.5 & 0.2 & 0.1 \\ 0.2 & 0.3 & 0.0 \\ 0.1 & 0.0 & 1.0 \end{bmatrix}, \quad
M_{1,1}^* = \begin{bmatrix} 0.3 & -0.1 & 0.0 \\ -0.1 & 1.2 & 0.2 \\ 0.0 & 0.2 & 0.7 \end{bmatrix}, \quad
M_{1,2}^* = \begin{bmatrix} 0 & 0 & 0 \\ 0 & 0 & 0 \\ 0 & 0 & 0 \end{bmatrix}, \\
M_{2,0}^* &= \begin{bmatrix} 1.0 & 0.3 & 0.1 \\ 0.3 & 0.8 & 0.0 \\ 0.1 & 0.0 & 0.6 \end{bmatrix}, \quad
M_{2,1}^* = \begin{bmatrix} 0.7 & 0.0 & 0.2 \\ 0.0 & 1.3 & 0.1 \\ 0.2 & 0.1 & 0.5 \end{bmatrix}, \quad
M_{2,2}^* = \begin{bmatrix} 0 & 0 & 0 \\ 0 & 0 & 0 \\ 0 & 0 & 0 \end{bmatrix}, \\
\end{align*}
The corresponding gating weights $\omega_{h,k}^*$ and expert coefficients $a_{h,i,k}^* \in \mathbb{R}^3$ are listed in Table~\ref{tab:true_params_larger}.

\begin{table}[!h]
\caption{Ground-truth parameters used in Experiment~3 for gating weights $\omega_{h,k}^*$ and expert coefficients $a_{h,i,k}^* \in \mathbb{R}^3$ with $H^* = 3$ heads, $N^* = 3$ experts per head, and $K^* = 3$ channels.}
\label{tab:true_params_larger}
\centering
{
\begin{tabular}{ccccc}
\hline
Head $h$ & Expert $i$ & Channel $k$ & $\omega_{h,k}^*$ & $a_{h,i,k}^*$ \\
\hline
0 & 0 & 0 & 1.0 & $(1.0, -0.5, 0.3)$ \\
0 & 0 & 1 & 0.5 & $(0.5, 0.8, -0.2)$ \\
0 & 0 & 2 & 0.6 & $(0.4, -0.3, 0.6)$ \\
0 & 1 & 0 & 1.0 & $(-1.0, 0.8, 0.1)$ \\
0 & 1 & 1 & 0.5 & $(0.2, -0.3, 0.5)$ \\
0 & 1 & 2 & 0.6 & $(-0.5, 0.4, -0.3)$ \\
0 & 2 & 0 & 1.0 & $(0.6, 0.3, -0.4)$ \\
0 & 2 & 1 & 0.5 & $(-0.2, 0.7, 0.2)$ \\
0 & 2 & 2 & 0.6 & $(0.3, -0.5, 0.8)$ \\
\hline
1 & 0 & 0 & 0.8 & $(0.6, 0.4, -0.2)$ \\
1 & 0 & 1 & 0.3 & $(-0.2, 0.5, 0.3)$ \\
1 & 0 & 2 & 0.4 & $(0.5, -0.3, 0.7)$ \\
1 & 1 & 0 & 0.8 & $(-0.7, -0.2, 0.4)$ \\
1 & 1 & 1 & 0.3 & $(0.3, -0.4, 0.1)$ \\
1 & 1 & 2 & 0.4 & $(-0.4, 0.6, -0.5)$ \\
1 & 2 & 0 & 0.8 & $(0.2, 0.8, -0.3)$ \\
1 & 2 & 1 & 0.3 & $(-0.6, 0.1, 0.4)$ \\
1 & 2 & 2 & 0.4 & $(0.7, -0.2, 0.3)$ \\
\hline
2 & 0 & 0 & 0.7 & $(0.5, 0.6, -0.4)$ \\
2 & 0 & 1 & 0.4 & $(-0.3, 0.4, 0.7)$ \\
2 & 0 & 2 & 0.2 & $(0.2, -0.5, 0.3)$ \\
2 & 1 & 0 & 0.7 & $(-0.4, 0.3, 0.5)$ \\
2 & 1 & 1 & 0.4 & $(0.6, -0.7, 0.1)$ \\
2 & 1 & 2 & 0.2 & $(-0.2, 0.5, -0.6)$ \\
2 & 2 & 0 & 0.7 & $(0.8, -0.1, 0.3)$ \\
2 & 2 & 1 & 0.4 & $(-0.5, 0.3, -0.2)$ \\
2 & 2 & 2 & 0.2 & $(0.1, 0.7, -0.4)$ \\
\hline
\end{tabular}
}
\end{table}

\newpage

\subsection{Experimental Results}
\label{appendix:results}

\textbf{Results of Experiment 2}

\begin{figure*}[h]
    \centering
    \begin{subfigure}{0.32\textwidth}
        \centering
        \includegraphics[width=\textwidth]{figures/d1_k3_e4.pdf}
        \caption{ MHA ($K=3$)}  
        \label{fig:sub1GELU}
    \end{subfigure}
    \hfill
    \begin{subfigure}{0.32\textwidth}
        \centering
        \includegraphics[width=\textwidth]{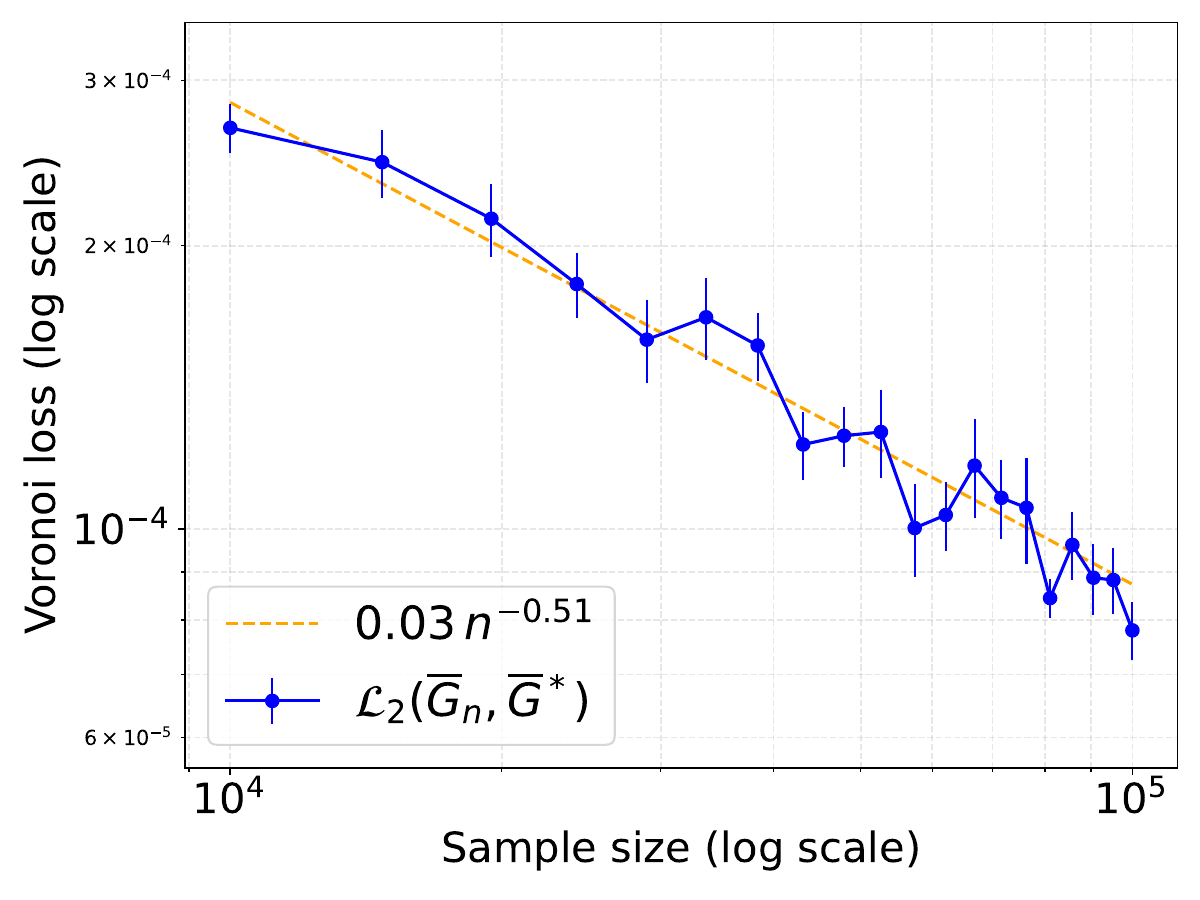}
        \caption{Value Gated Attention ($K=3$)}
        \label{fig:sub2GELU}
    \end{subfigure}
    \hfill
    \begin{subfigure}{0.32\textwidth}
        \centering
        \includegraphics[width=\textwidth]{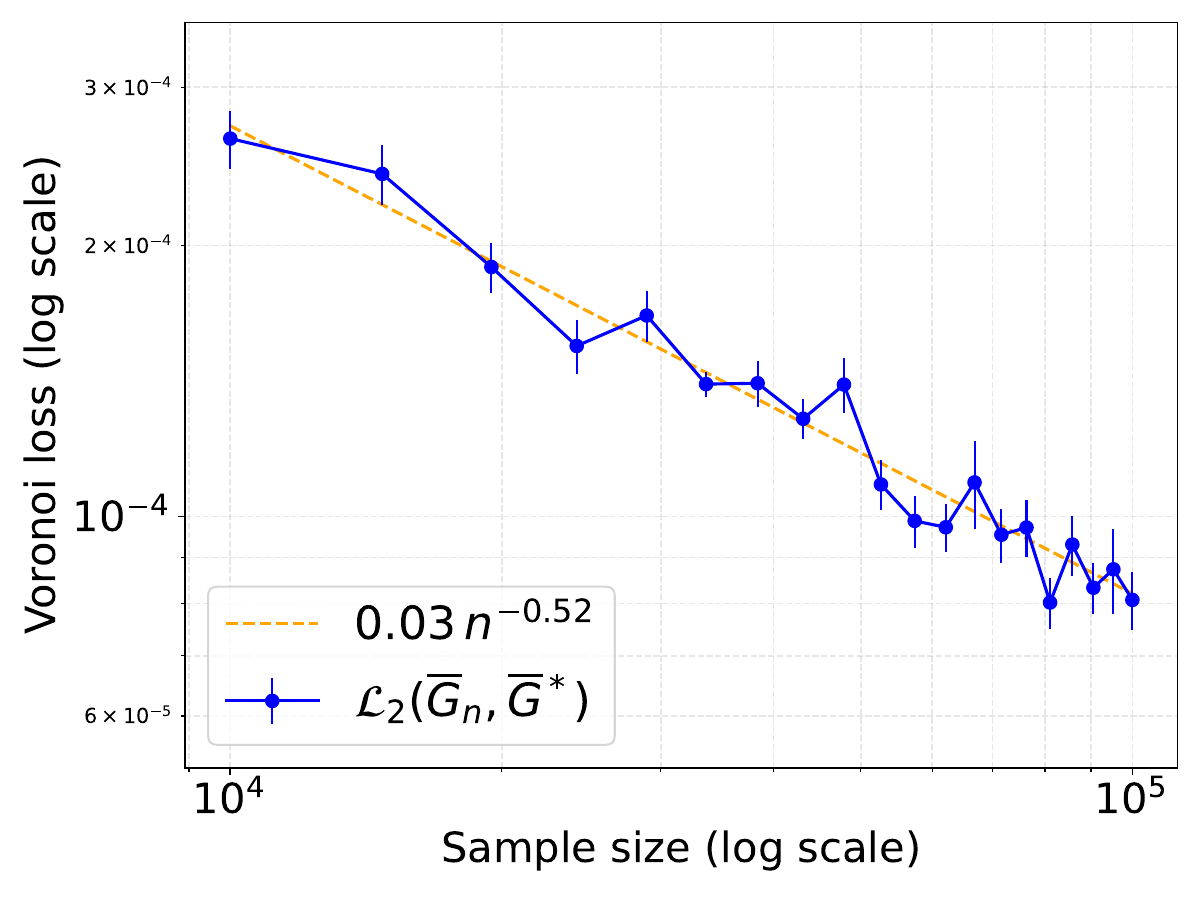}
        \caption{SDPA Gated Attention ($K=3$)}
        \label{fig:sub3GELU}
    \end{subfigure}
    \begin{subfigure}{0.32\textwidth}
        \centering
        \includegraphics[width=\textwidth]{figures/d1_k4_e4.pdf}
        \caption{MHA ($K=4$)}
        \label{fig:sub4GELU}
    \end{subfigure}
    \hfill
    \begin{subfigure}{0.32\textwidth}
        \centering
        \includegraphics[width=\textwidth]{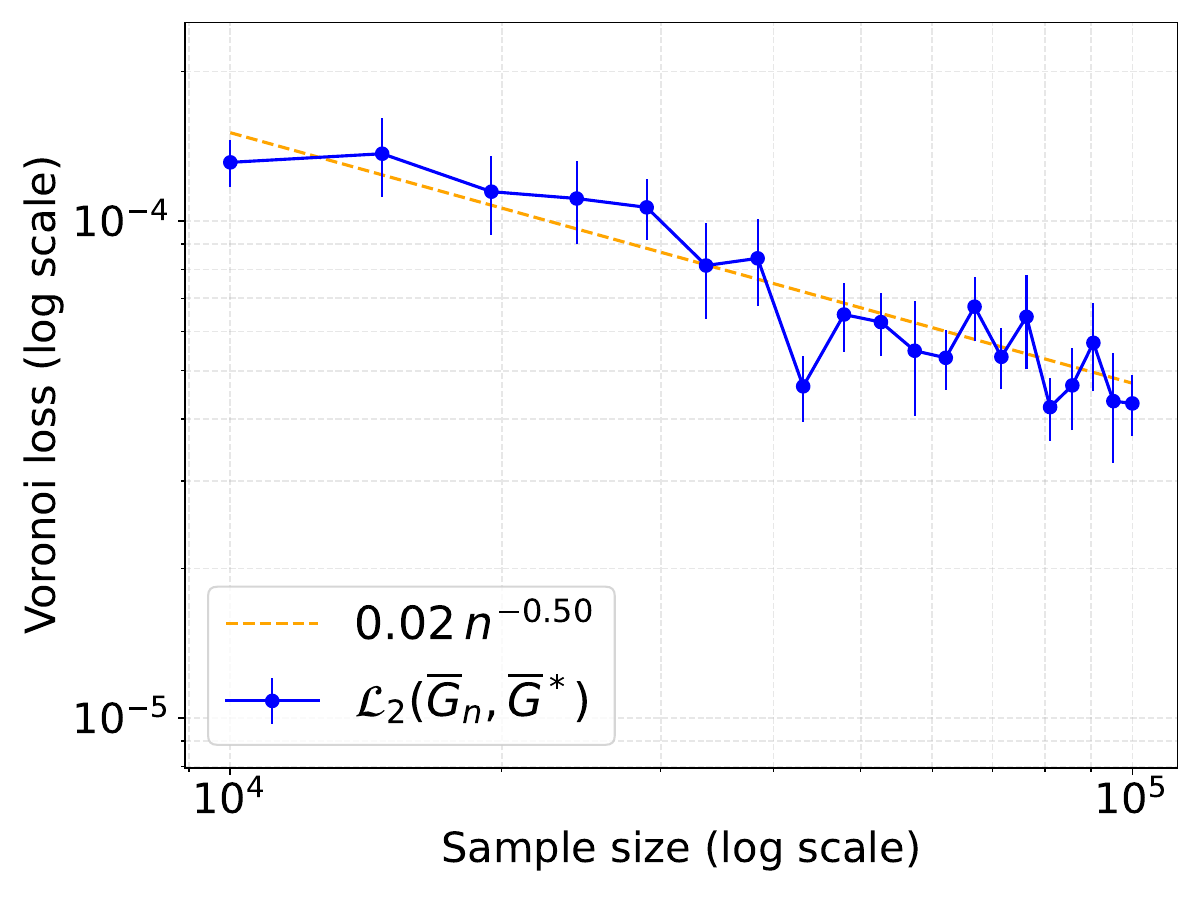}
        \caption{Value Gated Attention ($K=4$)}
        \label{fig:sub5GELU}
    \end{subfigure}
    \hfill
    \begin{subfigure}{0.32\textwidth}
        \centering
        \includegraphics[width=\textwidth]{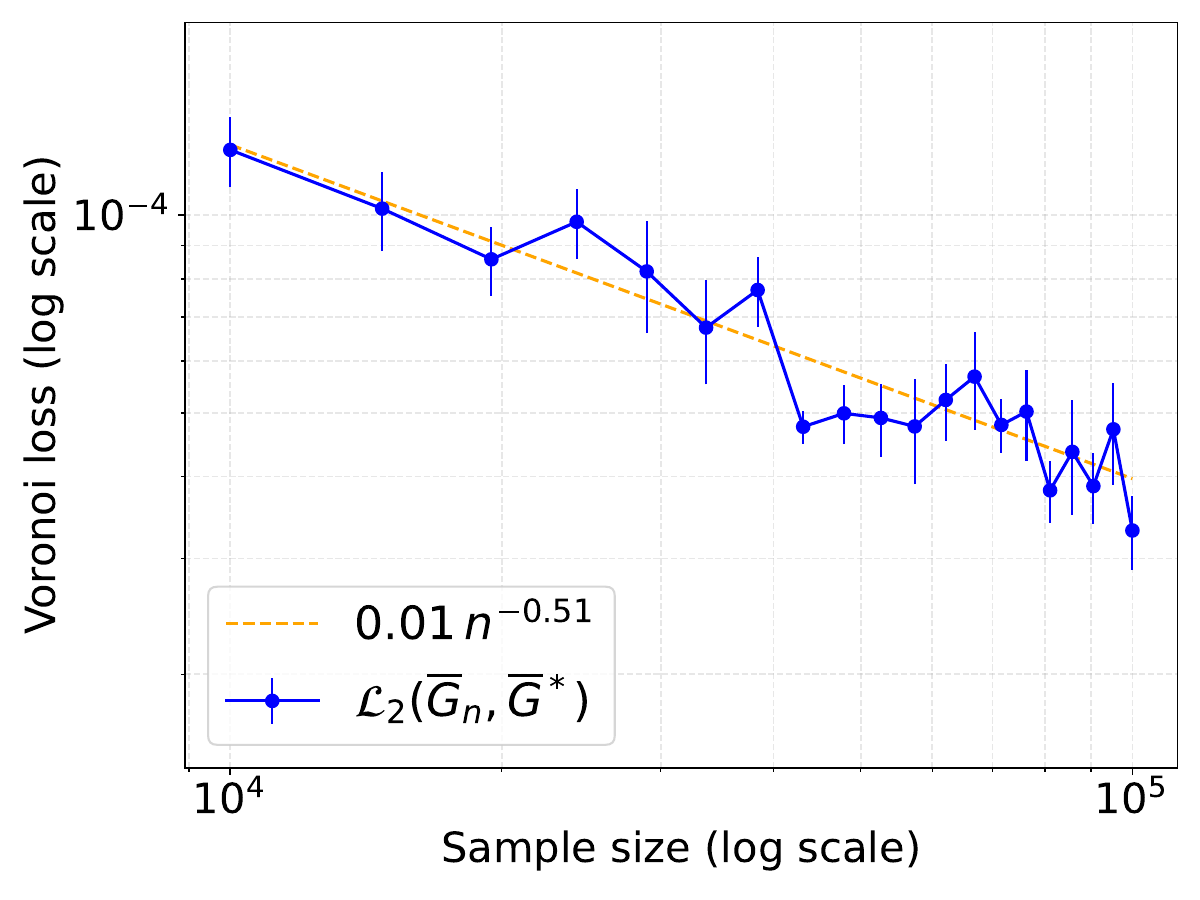}
        \caption{SDPA Gated Attention ($K=4$)}
        \label{fig:sub6GELU}
    \end{subfigure}
\captionsetup{justification=justified,singlelinecheck=false}
    \caption{Experiment~2: log-log plots of empirical Voronoi losses versus sample size $n$ for multi-head self-attention and gated attention mechanisms with tanh activation function.}
    \label{fig:combinedGELU}
\end{figure*}

To assess whether the sample-efficiency advantage observed in Experiment~1 is specific to the Sigmoid activation, we repeat the experiment with tanh function. As shown in Figure~\ref{fig:combinedGELU}, gated attention with tanh function likewise attains fast convergence in both settings. The Voronoi loss decays at empirical rates of approximately $O(n^{-0.51})$ and $O(n^{-0.50})$ in Setting~I (Figures~\ref{fig:sub2GELU} and~\ref{fig:sub5GELU}), and $O(n^{-0.52})$ and $O(n^{-0.51})$ in Setting~II (Figures~\ref{fig:sub3GELU} and~\ref{fig:sub6GELU}). Vanilla multi-head self-attention again converges much more slowly, at rates of $O(n^{-0.12})$ and $O(n^{-0.13})$ (Figures~\ref{fig:sub1GELU} and~\ref{fig:sub4GELU}). These results confirm that the sample-efficiency advantage of gated attention is robust to the choice of nonlinearity.

\textbf{Results of Experiment 3}

\begin{figure*}[h]
    \centering
    \begin{subfigure}{0.32\textwidth}
        \centering
        \includegraphics[width=\textwidth]{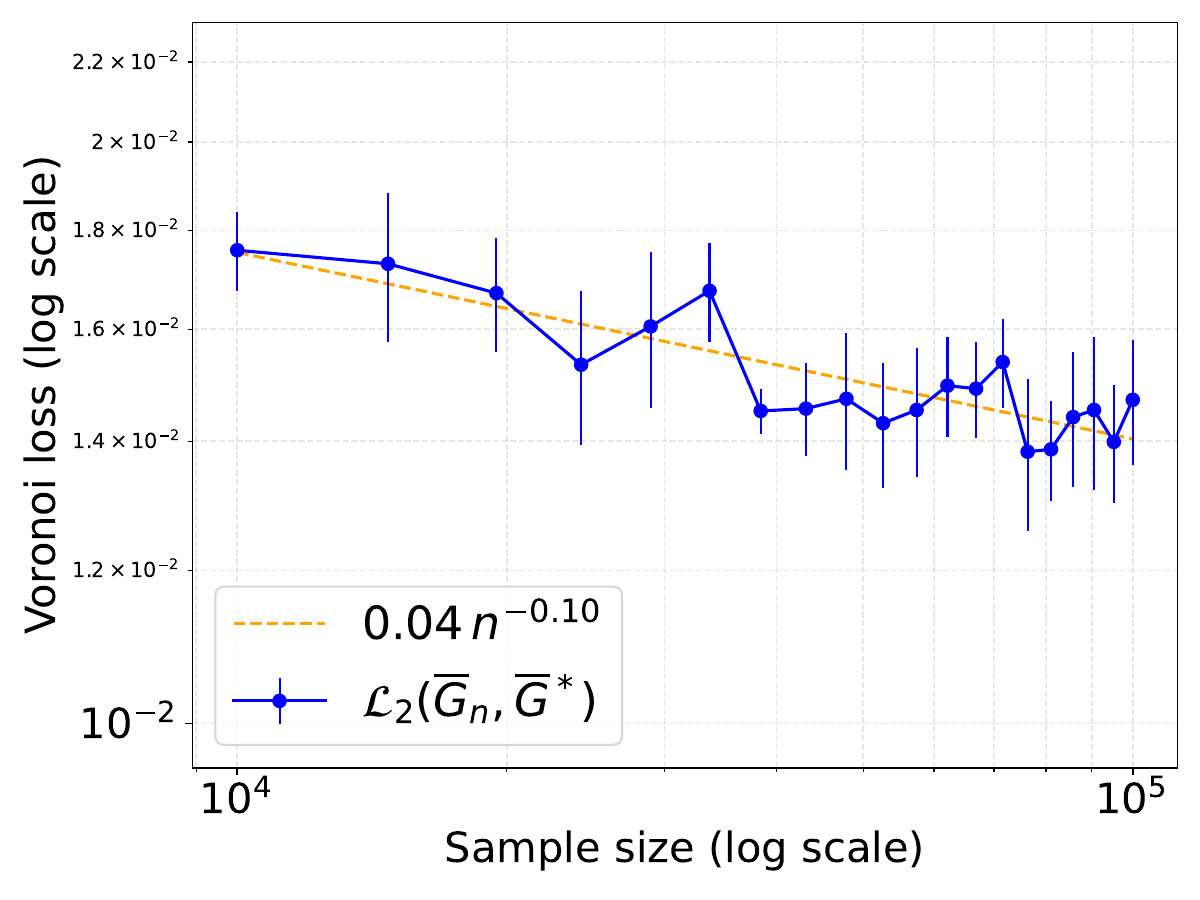}
        \caption{ MHA ($K=4$)}  
        \label{fig:sub1Larger}
    \end{subfigure}
    \hfill
    \begin{subfigure}{0.32\textwidth}
        \centering
        \includegraphics[width=\textwidth]{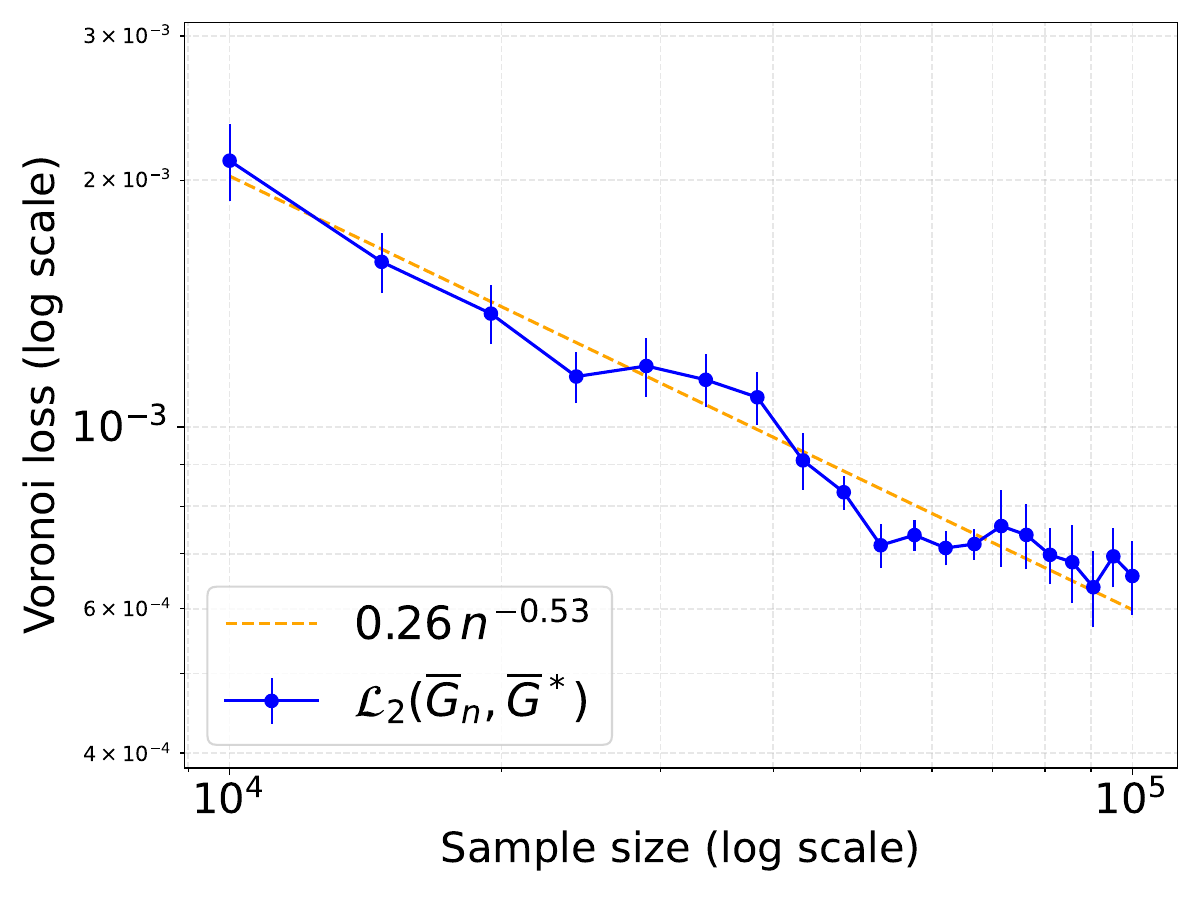}
        \caption{Value Gated Attention ($K=4$)}
        \label{fig:sub2Larger}
    \end{subfigure}
    \hfill
    \begin{subfigure}{0.32\textwidth}
        \centering
        \includegraphics[width=\textwidth]{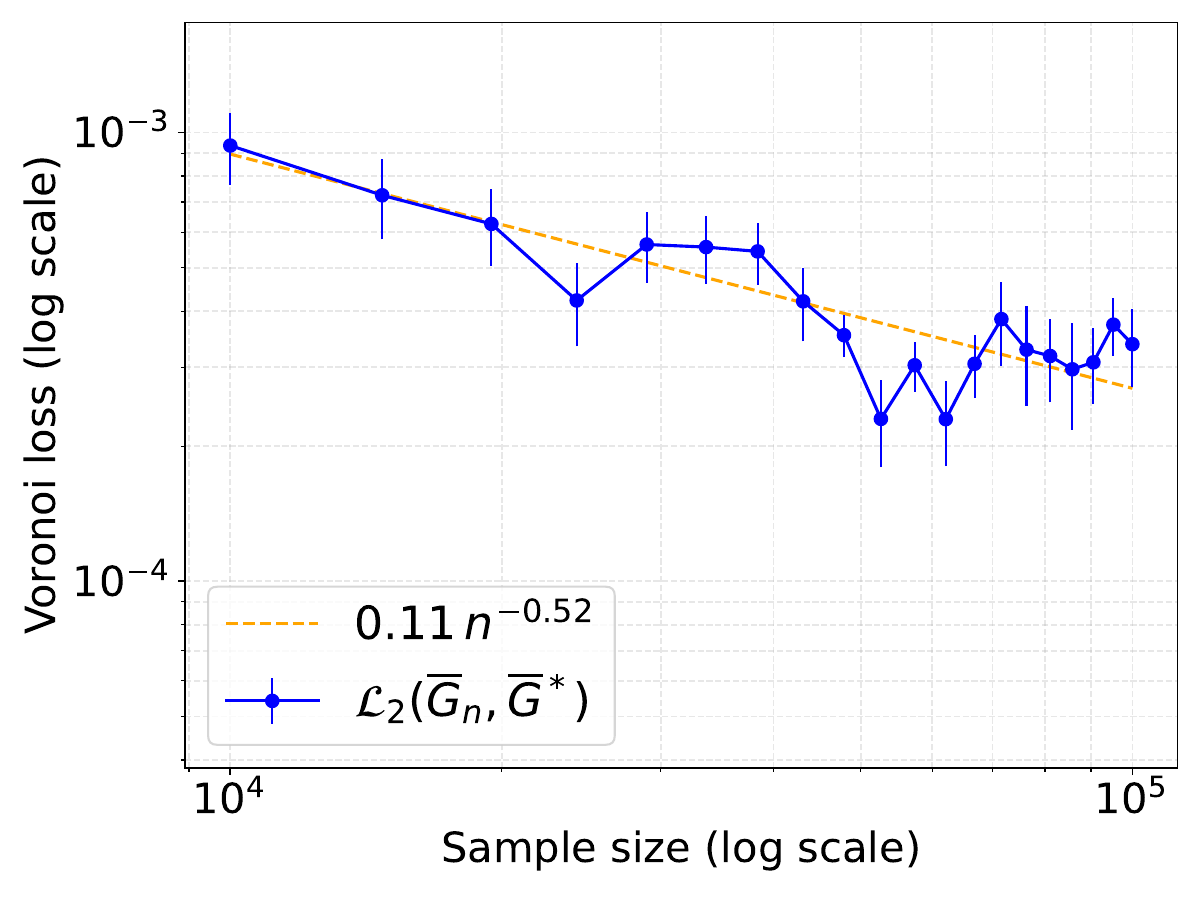}
        \caption{SDPA Gated Attention ($K=4$)}
        \label{fig:sub3Larger}
    \end{subfigure}
    \begin{subfigure}{0.32\textwidth}
        \centering
        \includegraphics[width=\textwidth]{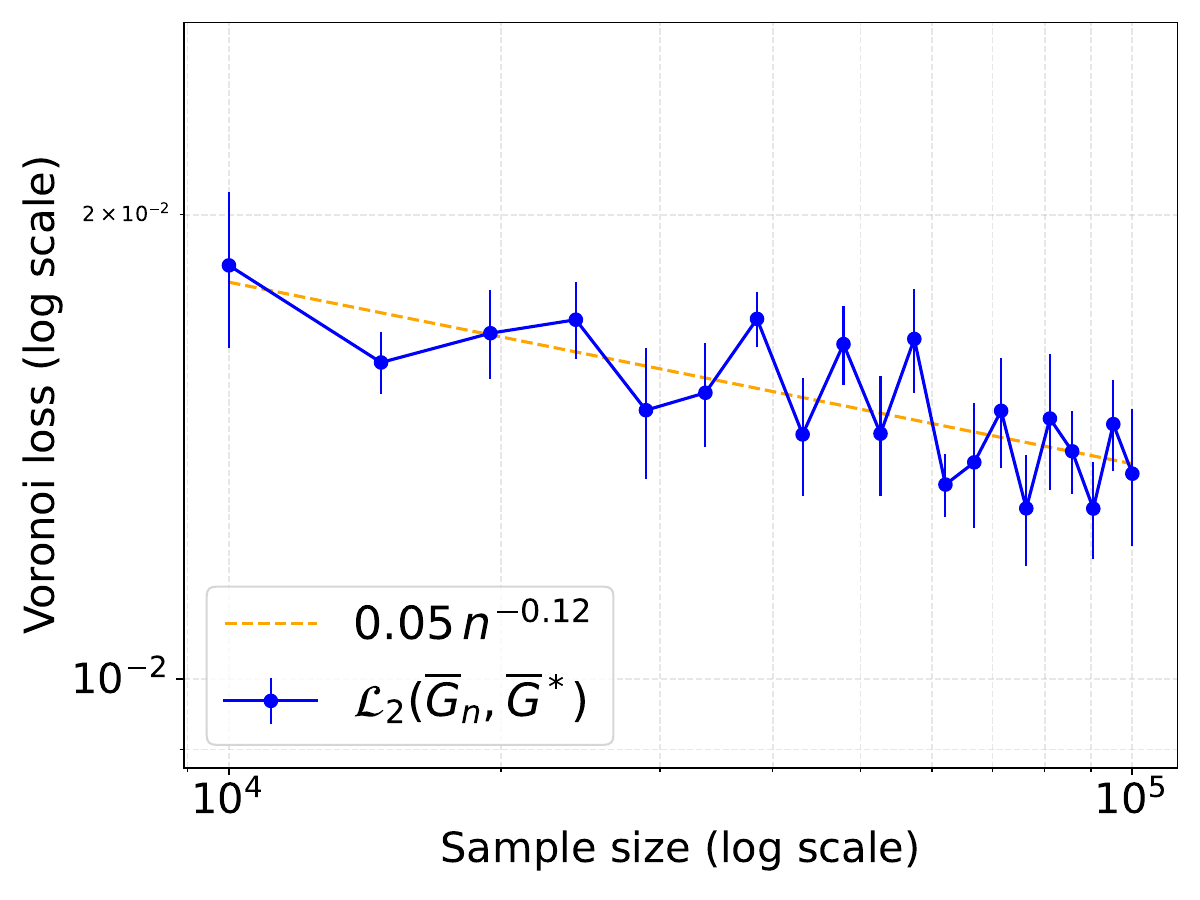}
        \caption{MHA ($K=5$)}
        \label{fig:sub4Larger}
    \end{subfigure}
    \hfill
    \begin{subfigure}{0.32\textwidth}
        \centering
        \includegraphics[width=\textwidth]{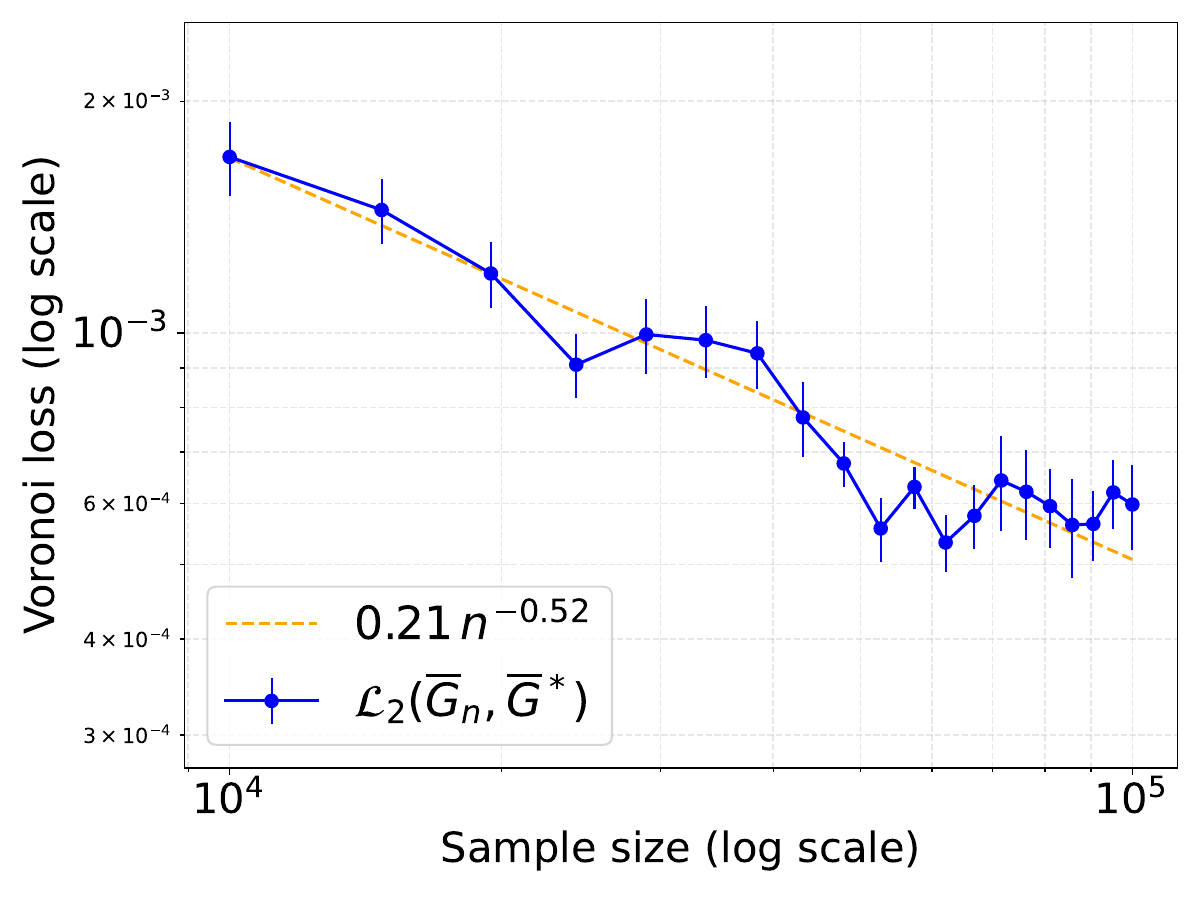}
        \caption{Value Gated Attention ($K=5$)}
        \label{fig:sub5Larger}
    \end{subfigure}
    \hfill
    \begin{subfigure}{0.32\textwidth}
        \centering
        \includegraphics[width=\textwidth]{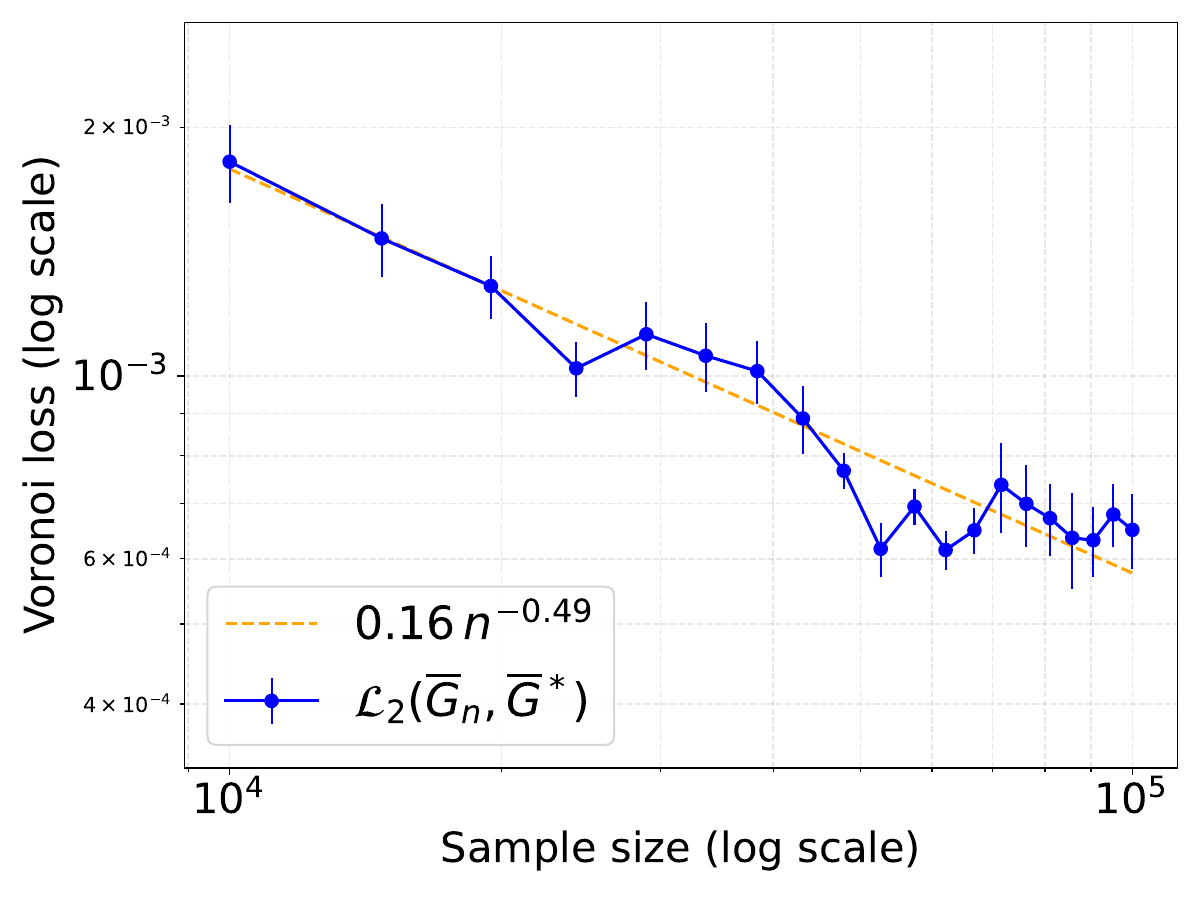}
        \caption{SDPA Gated Attention ($K=5$)}
        \label{fig:sub6Larger}
    \end{subfigure}
\captionsetup{justification=justified,singlelinecheck=false}
    \caption{Experiment~3: log-log plots of empirical Voronoi losses versus sample size $n$ for multi-head self-attention and gated attention mechanisms in higher-dimensional inputs ($d=3$) and larger head/expert/channel configurations ($H^* = N^* = K^* = 3$).}
    \label{fig:combinedLarger}
\end{figure*}

We further evaluate our theory under higher-dimensional inputs ($d = 3$) and larger head/expert/channel configurations ($H^* = N^* = K^* = 3$, with $K \in \{4, 5\}$ fitted channels). Figure~\ref{fig:combinedLarger} shows that gated attention continues to converge rapidly in this larger configuration. The Voronoi loss decays at empirical rates of approximately $O(n^{-0.53})$ and $O(n^{-0.52})$ in Setting~I (Figures~\ref{fig:sub2Larger} and~\ref{fig:sub5Larger}), and $O(n^{-0.52})$ and $O(n^{-0.49})$ in Setting~II (Figures~\ref{fig:sub3Larger} and~\ref{fig:sub6Larger}). Vanilla multi-head self-attention, by contrast, exhibits notably slower rates of $O(n^{-0.10})$ and $O(n^{-0.12})$ (Figures~\ref{fig:sub1Larger} and~\ref{fig:sub4Larger}). These results confirm that the sample-efficiency advantage of gated attention persists in higher-dimensional inputs and larger model configurations.

\section{Additional Experiments on Real-World Data}

\subsection{Objective}

The controlled HMoE experiments in Section~\ref{sec:experiments} isolate the parameter-estimation behavior predicted by our theory. We complement them with an end-to-end language-modeling experiment designed to answer a practical question: when the remainder of a Transformer and its training recipe are fixed, does inserting the nonlinearities studied in Settings~I and~II improve performance as the number of training tokens increases? Following the language-pretraining pipeline of CompeteSMoE~\cite{nguyen2025competesmoe}, we train standard MHA, Value Gated Attention, and SDPA Gated Attention from scratch and evaluate the same checkpoints throughout pretraining.

\subsection{Dataset and Prediction Task}

We pretrain on a subset of SlimPajama, a cleaned and deduplicated language corpus assembled from seven public data sources. Documents are tokenized with an 8,000-token SentencePiece vocabulary and packed into sequences of 1,024 tokens. The training objective is causal next-token prediction: for a token sequence $(x_1,\ldots,x_T)$, the model minimizes
\begin{equation*}
    \mathcal{L}_{\mathrm{LM}}
    = -\frac{1}{T-1}\sum_{t=1}^{T-1}
    \log p_\theta(x_{t+1}\mid x_{1:t}).
\end{equation*}
Training shards are traversed in a fixed randomized order shared by all three variants. We use a disjoint SlimPajama holdout subset for validation and report perplexity, $\exp(\mathcal{L}_{\mathrm{LM}})$, so lower values indicate better language modeling.

\subsubsection{Model and Attention Variants}

The backbone is a 153,723,712-parameter decoder-only sparse Transformer. It contains 16 pre-layer-normalized Transformer layers, hidden dimension 512, four causal self-attention heads per layer, head-projection dimension 82, rotary positional embeddings, and ReLU feed-forward blocks. Each sparse feed-forward block has 64 experts of hidden dimension 128 and routes every token to its top two experts. All three variants use exactly the same embeddings, sparse feed-forward layers, normalization, residual connections, causal masks, and output head.

Only the attention operation is changed. For one head, let
\begin{equation*}
    Q=XW_Q,\qquad K=XW_K,\qquad V=XW_V,
    \qquad
    A=\operatorname{softmax}\!\left(\frac{QK^\top}{\sqrt{d_h}}\right).
\end{equation*}
The three compared attention modules are
\begin{align*}
    \text{MHA (Control):}\quad & Y=AV, \\
    \text{Value Gated (VG):}\quad & Y=A\varphi(V), \\
    \text{SDPA Gated (SG):}\quad & Y=\varphi(AV),
\end{align*}
followed by the same output projection. We apply the activation elementwise with
$\varphi(z)=\tanh(z+0.001)$. This implementation directly matches Settings~I and~II: VG places the nonlinearity after the value projection, whereas SG places it after the complete scaled dot-product attention output. The activation introduces no learned parameters, so all variants have the same parameter count.

\begin{table}[h]
\caption{Architecture and pretraining configuration for the real-world validation.}
\label{tab:real_transformer_config}
\centering
\small
\begin{tabular}{ll}
\hline
Configuration & Value \\
\hline
Transformer layers / hidden size & $16$ / $512$ \\
Attention heads / head dimension & $4$ / $82$ \\
Feed-forward experts / active experts & $64$ / $2$ \\
Expert hidden dimension & $128$ \\
Parameters & $153{,}723{,}712$ \\
Context length / vocabulary size & $1{,}024$ / $8{,}000$ \\
Batch size & $64$ sequences \\
Optimizer / learning rate & AdamW / $2.5\times10^{-4}$ \\
Weight decay / gradient clipping & $0.01$ / $0.1$ \\
Learning-rate schedule & Cosine \\
Dropout & $0$ \\
\hline
\end{tabular}
\end{table}

\subsection{Training and Evaluation Protocol}

We initialize every model with seed 42, 52 and optimize it with AdamW using a learning rate of $2.5\times10^{-4}$, weight decay $0.01$, gradient clipping at $0.1$, mixed-precision training, and a cosine learning-rate schedule with minimum multiplier $0.1$. We use no model or attention dropout. As in the CompeteSMoE pipeline, the sparse feed-forward layers use balancing and router regularization; these terms and their coefficients are identical for all attention variants. Training is conducted on an NVIDIA A100 80GB GPU.

Each optimization step processes $64\times1{,}024=65{,}536$ tokens. We save a checkpoint every 10,000 steps and continue training from that checkpoint on the next nonoverlapping block of SlimPajama tokens. Consequently, the checkpoints at 10,000 through 70,000 steps correspond to approximately 655M, 1,311M, 1,966M, 2,621M, 3,277M, 3,932M, and 4,588M seen tokens.

We evaluate each checkpoint on the fixed SlimPajama validation split and four zero-shot benchmarks. HellaSwag measures commonsense sentence completion, the Children's Book Test (CBT) evaluates contextual cloze prediction, BLiMP tests grammatical knowledge through minimal sentence pairs, and RACE evaluates reading comprehension. For each benchmark item, the model scores the candidate completions autoregressively and selects the candidate with the highest likelihood. Table~\ref{tab:real_transformer_results} reports validation perplexity and benchmark accuracy at every token budget.

\subsubsection{Results}

\paragraph{Language modeling.}
Both gated-attention variants have lower mean validation perplexity than standard MHA at all eight reported budgets, with VG obtaining the lowest mean throughout. At 327.68M processed tokens, the mean perplexities are $7.2538$, $7.2377$, and $7.2472$ for MHA, VG, and SG, respectively. At 4,587.52M tokens, they are $6.7904$, $6.7590$, and $6.7710$. The improvements over MHA are consistent across the reported budgets for this seed pair.

\paragraph{Zero-shot evaluation.}
A gated variant has the highest mean accuracy in 31 of the 32 budget--benchmark comparisons; the exception is BLiMP at 10,000 steps. Individually, VG exceeds MHA in 29 comparisons and SG in 26, while both exceed MHA in 24. At the largest budget, VG improves over MHA by $0.26$, $0.51$, $0.58$, and $0.61$ percentage points on HellaSwag, CBT, BLiMP, and RACE, respectively; SG improves by $0.35$, $0.34$, $0.50$, and $0.73$ points. The downstream results favor gated attention overall for this seed pair, but the ranking depends on the task and budget, and the reported standard errors reflect only two runs.

\begin{table*}[p]
\caption{Validation perplexity and zero-shot accuracy (\%) throughout SlimPajama pretraining. Entries are mean $\pm$ one standard error across seeds 42 and 52 ($n=2$), with $\mathrm{SE}=s/\sqrt{2}$ using the sample standard deviation $s$. Control denotes standard MHA; VG and SG apply $\tanh(z+0.001)$ after the value projection and SDPA output, respectively. Bold entries mark the best unrounded means within each budget (lower perplexity, higher accuracy). Tokens count processed tokens, not necessarily unique tokens.}
\label{tab:real_transformer_results}
\centering
\small
\setlength{\tabcolsep}{2pt}
\renewcommand{\arraystretch}{1.12}
\resizebox{\textwidth}{!}{%
\begin{tabular}{@{}r l c c c c c@{}}
\toprule
Tokens (M) & Attention & Perplexity $\downarrow$ & HellaSwag $\uparrow$ & CBT $\uparrow$ & BLiMP $\uparrow$ & RACE $\uparrow$ \\
\midrule
327.68 & Control & $7.2538 \pm 0.0030$ & $27.99 \pm 0.17$ & $57.08 \pm 0.30$ & $66.32 \pm 0.04$ & $28.78 \pm 0.19$ \\
 & VG & $\mathbf{7.2377 \pm 0.0023}$ & $\mathbf{28.42 \pm 0.04}$ & $\mathbf{57.24 \pm 0.07}$ & $66.37 \pm 0.48$ & $\mathbf{29.37 \pm 0.24}$ \\
 & SG & $7.2472 \pm 0.0104$ & $27.85 \pm 0.06$ & $56.71 \pm 0.03$ & $\mathbf{66.51 \pm 0.23}$ & $29.17 \pm 0.06$ \\
\midrule
655.36 & Control & $6.3808 \pm 0.0037$ & $27.68 \pm 0.16$ & $60.93 \pm 0.13$ & $\mathbf{69.58 \pm 0.24}$ & $29.44 \pm 0.16$ \\
 & VG & $\mathbf{6.3435 \pm 0.0010}$ & $27.93 \pm 0.09$ & $60.83 \pm 0.07$ & $69.25 \pm 0.11$ & $\mathbf{29.68 \pm 0.20}$ \\
 & SG & $6.3507 \pm 0.0119$ & $\mathbf{27.97 \pm 0.02}$ & $\mathbf{60.97 \pm 0.00}$ & $69.06 \pm 0.71$ & $29.42 \pm 0.58$ \\
\midrule
1,310.72 & Control & $6.7146 \pm 0.0192$ & $27.81 \pm 0.00$ & $65.66 \pm 0.19$ & $72.02 \pm 0.03$ & $29.27 \pm 0.18$ \\
 & VG & $\mathbf{6.6759 \pm 0.0188}$ & $28.10 \pm 0.05$ & $66.25 \pm 0.15$ & $72.09 \pm 0.14$ & $29.35 \pm 0.09$ \\
 & SG & $6.6828 \pm 0.0286$ & $\mathbf{28.11 \pm 0.09}$ & $\mathbf{66.53 \pm 0.14}$ & $\mathbf{72.09 \pm 0.61}$ & $\mathbf{29.68 \pm 0.07}$ \\
\midrule
1,966.08 & Control & $6.8119 \pm 0.0645$ & $28.13 \pm 0.15$ & $66.87 \pm 0.10$ & $72.75 \pm 0.16$ & $29.60 \pm 0.22$ \\
 & VG & $\mathbf{6.7771 \pm 0.0680}$ & $\mathbf{28.29 \pm 0.06}$ & $\mathbf{67.59 \pm 0.15}$ & $72.77 \pm 0.27$ & $\mathbf{29.84 \pm 0.26}$ \\
 & SG & $6.7827 \pm 0.0550$ & $28.02 \pm 0.02$ & $67.54 \pm 0.25$ & $\mathbf{73.01 \pm 0.73}$ & $29.82 \pm 0.10$ \\
\midrule
2,621.44 & Control & $6.7486 \pm 0.0028$ & $28.06 \pm 0.11$ & $68.17 \pm 0.02$ & $73.44 \pm 0.23$ & $29.53 \pm 0.13$ \\
 & VG & $\mathbf{6.7090 \pm 0.0064}$ & $\mathbf{28.33 \pm 0.08}$ & $\mathbf{68.54 \pm 0.12}$ & $73.21 \pm 0.20$ & $\mathbf{29.87 \pm 0.08}$ \\
 & SG & $6.7269 \pm 0.0003$ & $28.27 \pm 0.05$ & $68.47 \pm 0.24$ & $\mathbf{73.51 \pm 0.52}$ & $29.60 \pm 0.10$ \\
\midrule
3,276.80 & Control & $6.8050 \pm 0.0414$ & $28.36 \pm 0.00$ & $68.86 \pm 0.22$ & $73.71 \pm 0.25$ & $29.52 \pm 0.13$ \\
 & VG & $\mathbf{6.7791 \pm 0.0570}$ & $\mathbf{28.55 \pm 0.16}$ & $69.23 \pm 0.28$ & $74.10 \pm 0.36$ & $29.84 \pm 0.06$ \\
 & SG & $6.7908 \pm 0.0382$ & $28.50 \pm 0.08$ & $\mathbf{69.46 \pm 0.14}$ & $\mathbf{74.12 \pm 0.19}$ & $\mathbf{29.90 \pm 0.35}$ \\
\midrule
3,932.16 & Control & $6.7869 \pm 0.0161$ & $28.59 \pm 0.02$ & $69.38 \pm 0.13$ & $73.07 \pm 0.08$ & $29.42 \pm 0.10$ \\
 & VG & $\mathbf{6.7518 \pm 0.0002}$ & $\mathbf{28.66 \pm 0.00}$ & $69.83 \pm 0.03$ & $\mathbf{73.78 \pm 0.66}$ & $29.75 \pm 0.05$ \\
 & SG & $6.7702 \pm 0.0128$ & $28.56 \pm 0.04$ & $\mathbf{69.91 \pm 0.14}$ & $73.66 \pm 0.21$ & $\mathbf{30.01 \pm 0.20}$ \\
\midrule
4,587.52 & Control & $6.7904 \pm 0.0328$ & $28.65 \pm 0.23$ & $69.67 \pm 0.19$ & $73.13 \pm 0.06$ & $29.38 \pm 0.08$ \\
 & VG & $\mathbf{6.7590 \pm 0.0100}$ & $28.91 \pm 0.06$ & $\mathbf{70.17 \pm 0.41}$ & $\mathbf{73.71 \pm 0.74}$ & $29.99 \pm 0.20$ \\
 & SG & $6.7710 \pm 0.0276$ & $\mathbf{29.00 \pm 0.11}$ & $70.00 \pm 0.21$ & $73.62 \pm 0.20$ & $\mathbf{30.11 \pm 0.34}$ \\
\bottomrule
\end{tabular}%
}
\end{table*}




\end{document}

%% file: math_commands.tex
\usepackage{eqnarray,amsmath}
\usepackage{booktabs}       
\usepackage{nicefrac}       
\renewcommand{\arraystretch}{1.35}

\usepackage{epsf}
\usepackage{epsfig}
\usepackage{fancyhdr}
\usepackage{graphics}
\usepackage{graphicx}
\usepackage{epstopdf}
\usepackage{psfrag}
\usepackage{fullpage}
\usepackage{pdfpages}

\usepackage{enumerate}   
\usepackage{multirow}
\usepackage{bm}

\usepackage{mathtools}

\usepackage{url}

\usepackage{color}

\usepackage{amsthm}
\usepackage{amsmath}
\usepackage{amssymb,bbm}
\usepackage{caption}
\usepackage{algorithmic}
\usepackage{algorithm}
\usepackage{textcomp}
\usepackage{siunitx}
\usepackage{wrapfig}
\usepackage{algorithmic}
\usepackage{algorithm}
\usepackage{multirow}
\usepackage{multicol}
\usepackage{mathtools}

\def\1{\bm{1}}

\DeclareMathAlphabet{\mathsfit}{\encodingdefault}{\sfdefault}{m}{sl}
\SetMathAlphabet{\mathsfit}{bold}{\encodingdefault}{\sfdefault}{bx}{n}

\def\gD{{\mathcal{D}}}

\def\gG{{\mathcal{G}}}

\def\gJ{{\mathcal{J}}}

\def\gL{{\mathcal{L}}}

\def\gO{{\mathcal{O}}}

\def\gR{{\mathcal{R}}}

\def\gU{{\mathcal{U}}}

\newcommand{\Gs}{G_{*}}

\DeclareMathOperator*{\argmin}{arg\,min}

\usepackage{eqnarray,amsmath}

\usepackage{amsthm}
\usepackage{amsmath}
\usepackage{amssymb,bbm}

\allowdisplaybreaks

\newcommand{\bbE}{\mathbb{E}}

\newcommand{\var}{\mathrm{Var}}

%% file: references.bib
@inproceedings{
qiu2025gated,
title={Gated Attention for Large Language Models: Non-linearity, Sparsity, and Attention-Sink-Free},
author={Zihan Qiu and Zekun Wang and Bo Zheng and Zeyu Huang and Kaiyue Wen and Songlin Yang and Rui Men and Le Yu and Fei Huang and Suozhi Huang and Dayiheng Liu and Jingren Zhou and Junyang Lin},
booktitle={The Thirty-ninth Annual Conference on Neural Information Processing Systems},
year={2025}
}

@inproceedings{vaswani2017attention,
 author = {Vaswani, Ashish and Shazeer, Noam and Parmar, Niki and Uszkoreit, Jakob and Jones, Llion and Gomez, Aidan N and Kaiser, \L ukasz and Polosukhin, Illia},
 booktitle = {Advances in Neural Information Processing Systems},
 publisher = {Curran Associates, Inc.},
 title = {Attention is All you Need},
 volume = {30},
 year = {2017}
}

@inproceedings{nguyen2024sigmoid,
  title={Sigmoid Gating is More Sample Efficient than Softmax Gating in Mixture of Experts},
  author={Nguyen, Huy and Ho, Nhat and Rinaldo, Alessandro},
  booktitle = "Advances in Neural Information Processing Systems",
  year={2024}
}

@article{yan2025sigmoid,
      title={Sigmoid Self-Attention has Lower Sample Complexity than Softmax Self-Attention: A Mixture-of-Experts Perspective},
      author={Yan, Fanqi and Nguyen, Huy and Akbarian, Pedram and Ho, Nhat and Rinaldo, Alessandro},
      journal={arXiv preprint arXiv:2502.00281},
      year={2025}
}

@article{nguyen2026convergence,
  author    = {Nguyen, Huy and Ho, Nhat and Rinaldo, Alessandro},
  title     = {Convergence Rates for Softmax Gating Mixture of Experts},
  journal   = {IEEE Transactions on Information Theory},
  year      = {2026},
  volume    = {72},
  number    = {2},
  pages     = {1276--1304},
  doi       = {10.1109/TIT.2025.3644021}
}

@article{akbarian2024quadratic,
      title={Quadratic Gating Mixture of Experts: Statistical Insights into Self-Attention}, 
      author={Akbarian, Pedram and Nguyen, Huy and Han, Xing and Ho, Nhat},
      journal={arXiv preprint arXiv:2410.11222}, 
      year={2024}
}

@article{yu97lecam,
author = "B. Yu",
title = "Assouad, {F}ano, and {L}e {C}am",
journal = "Festschrift for Lucien Le Cam",
editor = "D. Pollard and E. Torgersen and G. Yang",
pages = "423--435",
year = "1997"
}

@inproceedings{
gu2024attentionsink,
title={When Attention Sink Emerges in Language Models: An Empirical View},
author={Xiangming Gu and Tianyu Pang and Chao Du and Qian Liu and Fengzhuo Zhang and Cunxiao Du and Ye Wang and Min Lin},
booktitle={The Thirteenth International Conference on Learning Representations},
year={2025}
}

@inproceedings{
xiao2024efficient,
title={Efficient Streaming Language Models with Attention Sinks},
author={Guangxuan Xiao and Yuandong Tian and Beidi Chen and Song Han and Mike Lewis},
booktitle={The Twelfth International Conference on Learning Representations},
year={2024}
}

@inproceedings{nguyen2024statistical,
      title={Statistical Perspective of Top-K Sparse Softmax Gating Mixture of Experts}, 
      author={H. Nguyen and P. Akbarian and F. Yan and N. Ho},
      booktitle={International Conference on Learning Representations},
      year={2024}
}

@inproceedings{ramapuram2024sigmoidattention,
  title={Theory, Analysis, and Best Practices for Sigmoid Self-Attention},
  author={Ramapuram, Jason and Danieli, Federico and Dhekane, Eeshan Gunesh and Weers, Floris and Busbridge, Dan and Ablin, Pierre and Likhomanenko, Tatiana and Digani, Jagrit and Gu, Zijin and Shidani, Amitis and others},
  booktitle={The Thirteenth International Conference on Learning Representations},
  year = 2025
}

@article{jiang_approximation_1999,
	title = {On the {Approximation} {Rate} of {Hierarchical} {Mixtures}-of-{Experts} for {Generalized} {Linear} {Models}},
	volume = {11},
	issn = {0899-7667},
	url = {https://doi.org/10.1162/089976699300016403},
	doi = {10.1162/089976699300016403},
	number = {5},
	journal = {Neural Computation},
	author = {Jiang, Wenxin and Tanner, Martin A},
	month = jul,
	year = {1999},
	pages = {1183--1198},
}

@BOOK{vandeGeer-00,
AUTHOR = "S. van de Geer",
TITLE = "Empirical processes in M-estimation",
PUBLISHER = "Cambridge University Press",
YEAR = "2000"
}

@article{Jacobs1991adaptive,
  title={Adaptive Mixtures of Local Experts},
  author={Robert A. Jacobs and Michael I. Jordan and Steven J. Nowlan and Geoffrey E. Hinton},
  journal={Neural Computation},
  year={1991},
  volume={3},
  pages={79-87}
}

@article{ho2022gaussian,
  author  = {Nhat Ho and Chiao-Yu Yang and Michael I. Jordan},
  title   = {Convergence Rates for {G}aussian Mixtures of Experts},
  journal = {Journal of Machine Learning Research},
  year    = {2022},
  volume  = {23},
  number  = {323},
  pages   = {1--81},
}

@article{Jordan1993hmoe,
  title={Hierarchical mixtures of experts and the EM algorithm},
  author={Jordan, Michael I and Jacobs, Robert A},
  journal={Neural computation},
  volume={6},
  number={2},
  pages={181--214},
  year={1994},
  publisher={MIT Press}
}

@InProceedings{radford2021clip,
  title = 	 {Learning Transferable Visual Models From Natural Language Supervision},
  author =       {Radford, Alec and Kim, Jong Wook and Hallacy, Chris and Ramesh, Aditya and Goh, Gabriel and Agarwal, Sandhini and Sastry, Girish and Askell, Amanda and Mishkin, Pamela and Clark, Jack and Krueger, Gretchen and Sutskever, Ilya},
  booktitle = 	 {Proceedings of the 38th International Conference on Machine Learning},
  pages = 	 {8748--8763},
  year = 	 {2021},
  editor = 	 {Meila, Marina and Zhang, Tong},
  volume = 	 {139},
  series = 	 {Proceedings of Machine Learning Research},
  month = 	 {18--24 Jul},
  publisher =    {PMLR}
  
}

@INPROCEEDINGS{liu2021swin,
  author={Liu, Ze and Lin, Yutong and Cao, Yue and Hu, Han and Wei, Yixuan and Zhang, Zheng and Lin, Stephen and Guo, Baining},
  booktitle={2021 IEEE/CVF International Conference on Computer Vision (ICCV)}, 
  title={Swin Transformer: Hierarchical Vision Transformer using Shifted Windows}, 
  year={2021},
  pages={9992-10002},
}

@InProceedings{manole22refined,
  title = 	 {Refined Convergence Rates for Maximum Likelihood Estimation under Finite Mixture Models},
  author =       {T. Manole and N. Ho},
  booktitle = 	 {Proceedings of the 39th International Conference on Machine Learning},
  pages = 	 {14979--15006},
  year = 	 {2022},
  volume = 	 {162},
  series = 	 {Proceedings of Machine Learning Research},
  month = 	 {17--23 Jul},
  publisher =    {PMLR}
}

@article{nguyen2024hmoe,
      title={On Expert Estimation in Hierarchical Mixture of Experts: Beyond Softmax Gating Functions}, 
      author={Huy Nguyen and Xing Han and Carl William Harris and Suchi Saria and Nhat Ho},
      year={2024},
      Journal = {arXiv preprint arXiv:2410.02935}
}

@article{jiang1999hmoe,
author = {Wenxin Jiang and Martin A. Tanner},
title = {{Hierarchical mixtures-of-experts for exponential family regression models: approximation and maximum likelihood estimation}},
volume = {27},
journal = {The Annals of Statistics},
number = {3},
publisher = {Institute of Mathematical Statistics},
pages = {987 -- 1011},
year = {1999}
}

@article{deepseekv3,
  title={Deepseek-v3 technical report},
  author={DeepSeek-AI and others},
  journal={arXiv preprint arXiv:2412.19437},
  year={2024}
}

@article{qwen2025,
  title={Qwen2.5 Technical Report},
  author={Qwen and others},
  journal={arXiv preprint arXiv:2412.15115},
  year={2025}
}

@inproceedings{li2025hmoe,
  title     = {Hierarchical Mixture of Experts: Generalizable Learning for High-Level Synthesis},
  author    = {Li, Wenqi and Wang, Ding and Ding, Zijian and Sohrabizadeh, Atefeh and Qin, Zongyue and Cong, Jason and Sun, Yizhou},
  booktitle = {Proceedings of the AAAI Conference on Artificial Intelligence},
  year      = {2025}
}

@inproceedings{liao2025hmora,
  title     = {{HMoRA: Making LLMs More Effective with Hierarchical Mixture of LoRA Experts}},
  author    = {Liao, Mengqi and Chen, Wei and Shen, Junfeng and Guo, Shengnan and Wan, Huaiyu},
  booktitle = {International Conference on Learning Representations},
  year      = {2025}
}

@article{openai2024gpt4technicalreport,
      title={GPT-4 Technical Report}, 
      author={OpenAI and others},
      journal={arXiv preprint arXiv:2303.08774},
      year={2024}
}

@article{comanici2025gemini25pushingfrontier,
      title={Gemini 2.5: Pushing the Frontier with Advanced Reasoning, Multimodality, Long Context, and Next Generation Agentic Capabilities}, 
      author={GeminiTeam},
      journal={arXiv preprint arXiv:2507.06261},
      year={2025}
}

@article{grattafiori2024llama3,
  title={The Llama 3 Herd of Models},
  author={Aaron Grattafiori and Abhimanyu Dubey and Abhinav Jauhri and Abhinav Pandey and Abhishek Kadian and Ahmad Al-Dahle and Aiesha Letman and Akhil Mathur and others},
  journal={arXiv preprint arXiv:2407.21783},
  year={2024}
}

@inproceedings{chen2022theory,
 author = {Chen, Zixiang and Deng, Yihe and Wu, Yue and Gu, Quanquan and Li, Yuanzhi},
 booktitle = {Advances in Neural Information Processing Systems},
 editor = {S. Koyejo and S. Mohamed and A. Agarwal and D. Belgrave and K. Cho and A. Oh},
 pages = {23049--23062},
 publisher = {Curran Associates, Inc.},
 title = {Towards Understanding the Mixture-of-Experts Layer in Deep Learning},
 volume = {35},
 year = {2022}
}

@inproceedings{nguyen2023demystifying,
      title={Demystifying Softmax Gating Function in {G}aussian Mixture of Experts}, 
      author={Huy Nguyen and TrungTin Nguyen and Nhat Ho},
      booktitle = "Advances in Neural Information Processing Systems",
      year={2023}
}

@inproceedings{oldfield2024specialize,
      title={Multilinear Mixture of Experts: Scalable Expert Specialization through Factorization}, 
      author={James Oldfield and Markos Georgopoulos and Grigorios G. Chrysos and Christos Tzelepis and Yannis Panagakis and Mihalis A. Nicolaou and Jiankang Deng and Ioannis Patras},
      booktitle = "Advances in Neural Information Processing Systems",
      year={2024}
}

@inproceedings{dosovitskiy_image_2021,
	title = {An Image is Worth 16x16 Words: Transformers for Image Recognition at Scale},
	url = {https://openreview.net/forum?id=YicbFdNTTy},
	booktitle = {International {Conference} on {Learning} {Representations}},
	author = {A. Dosovitskiy and L. Beyer and A. Kolesnikov and D. Weissenborn and X. Zhai and T. Unterthiner and M. Dehghani and M. Minderer and G. Heigold and S. Gelly and J. Uszkoreit and N. Houlsby},
	year = {2021},
}

@inproceedings{liu2023llava,
author      = {Liu, Haotian and Li, Chunyuan and Wu, Qingyang and Lee, Yong Jae},
title       = {Visual Instruction Tuning},
booktitle   = {NeurIPS},
year        = {2023}
}

@inproceedings{bhojanapalli2021lowrank,
  title     = {Low-Rank Bottleneck in Multi-Head Attention Models},
  author    = {Bhojanapalli, Srinadh and Yun, Chulhee and Rawat, Ankit Singh and Reddi, Sashank and Kumar, Sanjiv},
  booktitle = {International Conference on Machine Learning},
  year      = {2020},
}

@inproceedings{katharopoulos2020transformers,
  title     = {Transformers are {RNNs}: Fast Autoregressive Transformers with Linear Attention},
  author    = {Katharopoulos, Angelos and Vyas, Apoorv and Pappas, Nikolaos and Fleuret, Fran{\c{c}}ois},
  booktitle = {Proceedings of the 37th International Conference on Machine Learning},
  year      = {2020},
  pages     = {5156--5165},
}

@inproceedings{choromanski2021rethinking,
  title     = {Rethinking Attention with {Performers}},
  author    = {Choromanski, Krzysztof and Likhosherstov, Valerii and Dohan, David and Song, Xingyou and Gane, Andreea and Sarlos, Tamas and Hawkins, Peter and Davis, Jared and Mohiuddin, Afroz and Kaiser, Lukasz and Belanger, David},
  booktitle = {International Conference on Learning Representations},
  year      = {2021},
}

@article{nguyen2025competesmoe,
  title         = {{CompeteSMoE}: Statistically Guaranteed Mixture of Experts
                   Training via Competition},
  author        = {Nguyen, Nam V. and Nguyen, Huy and Pham, Quang and
                   Nguyen, Van and Ramasamy, Savitha and Ho, Nhat},
  journal       = {arXiv preprint arXiv:2505.13380},
  year          = {2025},
  eprint        = {2505.13380},
  archiveprefix = {arXiv},
  primaryclass  = {cs.LG},
  url           = {https://arxiv.org/abs/2505.13380}
}
